\documentclass[11pt, a4paper]{article}
\usepackage[utf8]{inputenc}
\usepackage{authblk}
\usepackage{enumitem}
\usepackage[babel=true]{csquotes}
\usepackage{caption}
\usepackage{subcaption}
\usepackage{color}
\usepackage{dsfont}
\usepackage{booktabs}
\usepackage{capt-of}
\usepackage{multicol}

\usepackage{appendix}
\usepackage{float}
\usepackage{siunitx}
\usepackage{url}
\usepackage[colorlinks=true,bookmarks=false,linkcolor=blue,citecolor=blue,urlcolor=blue]{hyperref}
\usepackage{array,tabu,multirow,makecell}
\usepackage{threeparttable}
\usepackage{array}
\usepackage{longtable}
\usepackage{rotating}
\usepackage{tabularx}
\usepackage{lastpage}
\usepackage{lscape}
\usepackage{stackengine}
\usepackage{amsmath} 
\usepackage{amsfonts}
\usepackage[T1]{fontenc}
\usepackage{amsthm}
\usepackage{geometry}
\usepackage{graphicx}
\usepackage{amssymb}
\usepackage{bbm}
\usepackage{moreverb}
\usepackage{color, pdfcolmk}
\usepackage{fancyhdr}
\usepackage{csquotes}  
\usepackage{etoolbox}
\usepackage{rotating}
\makeatletter
\def\@footnotecolor{blue}
\define@key{Hyp}{footnotecolor}{%
 \HyColor@HyperrefColor{#1}\@footnotecolor%
}
\patchcmd{\@footnotemark}{\hyper@linkstart{link}}{\hyper@linkstart{footnote}}{}{}
\makeatother
\hypersetup{footnotecolor=blue}
\theoremstyle{plain}
\newtheorem{theorem}{Theorem}    
\newtheorem{corollary}[theorem]{Corollary}      

\newtheorem{proposition}[theorem]{Proposition}   
\newtheorem{assumption}{Assumption}

\newtheorem{principle}{Principle}
\newtheorem{definition}{Definition}%

\usepackage[round]{natbib}
\usepackage[english]{babel}
\usepackage{algorithm}
\usepackage{algorithmic}
\usepackage{pgfplots}
\pgfplotsset{compat=1.18}
\usetikzlibrary{arrows.meta,decorations.pathreplacing,positioning}

\title{A Geometric Theory of Decision Boundaries in Structured Markov Decision Processes}

\author[1]{Fredy POKOU \thanks{\texttt{fredypokou@gmx.fr}}}

\affil[1]{Inria, CNRS, Univ. of Lille, Centrale Lille, UMR 9189 - CRIStAL, F-59000 Lille, France}

\date{\today}

\begin{document}

\maketitle

\begin{abstract}

Classical dynamic programming represents optimal sequential decisions through value functions and policies. While this functional representation is natural for computing optimal decisions, it does not directly identify the mathematical object governing policy reconstruction, representation complexity, or oracle-query complexity once an optimal policy is fixed. This paper addresses this question by developing a geometric theory of structured optimal policies in which the decision-boundary geometry induced by the policy becomes the primary object of analysis. We show that, under suitable structural regularity conditions, this geometry provides the minimal representation required for policy reconstruction and determines the statistical and computational complexity of the reconstruction problem. Building upon this representation, we establish structural properties of policy-induced decision geometry, introduce intrinsic notions of boundary and decision complexity, derive information-theoretic measures of decision compression, and obtain statistical guarantees for boundary estimation and policy reconstruction from black-box policy queries. Collectively, these results demonstrate that, for the structured decision problems considered here, the complexity of policy reconstruction is governed by the geometry of the decision boundary rather than by the cardinality of the ambient state space. Controlled numerical experiments examine the principal theoretical predictions and provide empirical evidence consistent with the proposed framework.

\end{abstract}


\section{Introduction}
\label{sec:introduction}

The mathematical theory of dynamic programming has traditionally represented optimal sequential decision problems through functional objects, namely the optimal value function, the optimal state-action value function, and the optimal policy. This representation has provided the foundation for stochastic control, Markov decision processes, and reinforcement learning, and has led to a mature mathematical theory concerning optimality, approximation, convergence, and computational complexity. Virtually all existing analyses of sequential decision-making are therefore expressed in terms of functions defined over the state space.

For many questions, this functional representation is entirely appropriate. However, there exists an important class of questions for which it appears unnecessarily rich. Suppose that the objective is not to compute an optimal policy from a model, but rather to reconstruct an already optimal policy from black-box observations, to quantify its intrinsic representation complexity, to determine how many oracle queries are required for reliable recovery, or to understand which structural properties govern statistical learnability. For such questions, approximating an entire value function or policy over the state space may encode substantially more information than is actually required. 

\newpage

Indeed, once an optimal policy is fixed, only the interfaces at which the optimal decision changes determine the partition of the state space induced by that policy.
This observation suggests that the mathematical object governing policy reconstruction is not necessarily the policy viewed as a function, but rather the geometry of the decision regions that the policy induces. 

In structured stochastic optimization problems, optimal policies frequently exhibit monotonicity, threshold behavior, convexity, or other forms of regularity. 
Consequently, they partition the state space into relatively simple regions associated with distinct optimal actions. The boundaries separating these regions form a geometric object that is uniquely determined by the policy itself. Although these decision boundaries appear implicitly throughout the literature on structured dynamic programming, they are generally interpreted as consequences of optimal decision rules rather than as mathematical objects possessing their own statistical, computational, structural, and information-theoretic properties.
Adopting this geometric viewpoint fundamentally changes the mathematical questions that arise. Instead of studying approximation errors for functional representations, one may ask whether the decision-boundary geometry can itself be estimated consistently, whether its intrinsic complexity admits quantitative characterization, whether geometric regularity governs statistical sample complexity, and whether the complexity of policy representation should be measured through the ambient state space or through the geometry of the induced decision partition. These questions are largely orthogonal to the computation of optimal policies and concern instead the mathematical structure of policies that already exist.

The objective of this paper is to develop a geometric theory of structured optimal policies observed through black-box policy queries. The central mathematical object considered throughout the paper is the decision-boundary geometry induced by an optimal policy. Rather than treating this geometry as a secondary by-product of dynamic programming, we study it as an independent representation possessing its own structural properties, statistical estimators, complexity measures, reconstruction guarantees, and information-theoretic characteristics. Within this framework, policy reconstruction, oracle-query allocation, decision complexity, and information-theoretic compression are shown to arise from a common geometric representation.

The theoretical development proceeds progressively. We first formalize the geometry induced by structured optimal policies and establish conditions under which this geometry admits an ordered, low-complexity representation. We then investigate the mathematical properties of this representation, derive intrinsic notions of geometric and decision complexity, characterize the relationship between geometric complexity and information-theoretic compression, and develop statistical guarantees for boundary estimation and policy reconstruction from black-box observations. These results collectively show that, under suitable structural assumptions, the mathematical complexity governing policy reconstruction is determined by the geometry of the decision boundary rather than by the cardinality of the ambient state space.

The perspective proposed here is complementary to the classical functional theory of dynamic programming rather than a replacement for it. Functional representations remain the natural language for computing optimal decisions. Once an optimal policy has been obtained, however, its induced geometry provides an alternative mathematical representation that exposes properties largely invisible from the functional viewpoint alone. In particular, it reveals direct connections between structural regularity, statistical estimation, active oracle sampling, intrinsic complexity, and information-theoretic compression that are difficult to formulate using functional representations alone.

The numerical investigation is designed to examine these theoretical predictions. Rather than comparing competing learning algorithms, every numerical experiment estimates a mathematical quantity introduced by the theory and evaluates whether the corresponding theoretical prediction is supported under a controlled black-box protocol. The empirical study therefore serves as a validation of the proposed mathematical framework rather than as an independent performance comparison.

The remainder of the paper is organized as follows.
Section~\ref{sec:related} positions the geometric viewpoint within the
mathematical literature and classical structural dynamic programming.
Section~\ref{sec:policy_geometry} formalizes policy-induced geometry, while
Sections~\ref{sec:structural_dynamic_programs}-\ref{sec:learning_policy_geometry}
develop its structural, geometric, complexity, and statistical foundations.
Section~\ref{sec:num} provides controlled numerical validation of the theory,
and Section~\ref{sec:conclusion} discusses its scope and extensions.

\newpage

\section{Related Literature and Mathematical Positioning}
\label{sec:related}

The objective of this section is twofold. First, we position the present work within the mathematical literature on dynamic programming, stochastic optimization, computational geometry, and statistical learning. Second, we identify the mathematical objects that naturally emerge from these research directions and motivate the geometric formulation developed in the remainder of the paper. Accordingly, the discussion is organized around mathematical representations rather than research communities, application domains, or algorithmic paradigms.
The mathematical theory of sequential decision-making has traditionally been formulated through functional representations of optimality, most notably value functions, state-action value functions, and optimal policies. Structural optimization subsequently established qualitative properties of these policies, including monotonicity, threshold behavior, and lattice structures. Computational geometry provides a rigorous language for describing partitions and geometric organization, whereas statistical learning studies the estimation of geometric objects from partial observations. Although these research directions have developed largely independently, they progressively reveal increasingly rich mathematical representations of the same underlying decision process.
Viewed collectively, these developments suggest a common observation. In many structured stochastic optimization problems, an optimal policy induces a partition of the state space into regions associated with distinct optimal actions. The interfaces separating these regions define a geometric object that is completely determined by the optimal policy. Nevertheless, existing theory generally treats this geometry as a consequence of functional representations rather than as an independent mathematical object possessing its own statistical, structural, computational, and information-theoretic properties.

The present paper adopts precisely this alternative perspective. Instead of taking the value function or the policy as the primary object of analysis, the subsequent sections investigate the geometry induced by the optimal policy and examine the mathematical questions that naturally arise from this change of representation. As will become apparent throughout this review, this perspective progressively leads from functional representations to structured policies, from structured policies to decision-boundary geometry, from geometry to statistical inference, and ultimately from statistical inference to a unified theory of geometric complexity, information, and policy learning.

The remainder of this section is organized according to this mathematical progression. Section~\ref{subsec:classical_objects} reviews the classical functional objects underlying dynamic programming. Section~\ref{subsec:structured_policies} discusses the structural theory of optimal policies and the emergence of decision boundaries. Section~\ref{subsec:geometry_decision_regions} examines the geometry and intrinsic complexity of the induced state-space partitions. Section~\ref{subsec:policy_boundary_estimation}
 positions the statistical estimation of policy-induced decision boundaries with respect to existing boundary-estimation theory. Finally, Section~\ref{subsec:geometric_learning} explains how this geometric viewpoint naturally leads to a unified perspective on statistical complexity, information acquisition, and geometric policy learning.

\subsection{Classical Objects in Dynamic Programming}
\label{subsec:classical_objects}

Since the pioneering work of Bellman, the mathematical theory of sequential decision-making has been formulated around three closely related objects: the optimal value function
$V^\star$, the optimal state-action value function $Q^\star$, and the optimal policy $\pi^\star$.
These objects constitute the fundamental mathematical representations underlying Markov decision processes and stochastic dynamic programming, and they remain the basis of both classical optimal control theory and modern reinforcement learning \citep{bellman1957dynamic,puterman2014markov,bertsekas2012dynamic,hernandez2012discrete,feinberg2012handbook}.

Within this framework, optimal decision-making is characterized through recursive optimality equations whose solutions determine both the optimal expected return and the corresponding optimal decision rule. Consequently, the theoretical analysis of dynamic programming is expressed almost exclusively in terms of functional objects. Questions concerning existence, uniqueness, convergence, stability, approximation, and computational complexity are formulated with respect to either
$V^\star$, $Q^\star$, or the policy $\pi^\star$, leading to a mature mathematical theory encompassing value iteration, policy iteration, stochastic control, approximate dynamic programming, and large-scale optimization \citep{powell2007approximate,whittle1982optimization,bertsekas2025neuro}.

\newpage

As increasingly complex decision problems have been considered, the principal computational challenge has become the approximation of these functional representations. Approximate dynamic programming and reinforcement learning therefore focus primarily on constructing computationally tractable approximations of value functions, action-value functions, or policies while preserving near-optimal decision quality. The resulting notions of statistical efficiency and computational complexity are consequently defined relative to these functional objects.

This functional viewpoint has profoundly influenced the mathematical development of the field. At the same time, it implicitly determines the representation through which optimal decisions are analyzed. Once an optimal policy has been obtained, it induces a partition of the state space into regions associated with distinct optimal actions. While this partition is fundamental for understanding the qualitative structure of optimal decision-making, its geometric properties are generally regarded as secondary consequences of the policy itself rather than as primary mathematical objects. Existing theory therefore provides a comprehensive analysis of optimal values and optimal policies, but comparatively little attention has been devoted to the intrinsic geometry generated by the policy partition.

This observation suggests an alternative mathematical perspective. Rather than asking how accurately one can approximate
$V^\star$,
$Q^\star$, or $\pi^\star$, 
one may instead ask whether the geometric structure induced by the optimal policy can itself serve as the primary object of statistical inference. Adopting this viewpoint fundamentally changes the mathematical questions under consideration. Instead of studying approximation error for functional representations, the subsequent sections investigate geometric estimation, boundary reconstruction, structural complexity, information-theoretic compression, and the statistical properties of the geometry induced by structured optimal policies.

\subsection{Structured Optimal Policies and Decision Boundaries}
\label{subsec:structured_policies}

Beyond the existence of optimal policies, a major direction of stochastic dynamic programming has been devoted to identifying qualitative structural properties of optimal decision rules. Rather than viewing an optimal policy as an arbitrary mapping from states to actions, this line of research seeks conditions under which optimal decisions possess regularity properties that admit concise mathematical descriptions. Such structural characterizations have played a central role in operations research because they often reveal the intrinsic organization of optimal solutions independently of the numerical algorithms used to compute them.

The mathematical foundations of this theory are provided by the monotonicity and lattice-programming results established by Veinott, Topkis, and subsequent authors. Under suitable assumptions involving submodularity, supermodularity, stochastic monotonicity, or lattice orderings, optimal actions evolve monotonically with respect to the underlying state variables, leading naturally to threshold-type decision rules
\citep{veinott1965optimal,veinott1969discrete,topkis1978minimizing,topkis1998supermodularity}. Similar structural properties have subsequently been established for a broad range of stochastic optimization problems, including inventory control, queueing systems, partially observed Markov decision processes, and stochastic resource allocation
\citep{stidham1989monotonic,sobel1982optimality,lovejoy1987some,smith2002structural,krishnamurthy2016partially,koole2007monotonicity}.

From a mathematical viewpoint, these structural results admit a common interpretation. A threshold policy does not merely specify where the optimal action changes; it partitions the state space into subsets on which the optimal decision remains constant. In one-dimensional problems, this partition is determined by a finite collection of switching points. In higher-dimensional settings, the same principle extends naturally to switching curves, hypersurfaces, or more general decision interfaces separating neighboring action regions. Consequently, every structured optimal policy induces a partition of the state space whose organization is governed by a collection of decision boundaries.

This observation reveals an important shift of perspective. Classical structural optimization formulates its conclusions in terms of monotone policies, threshold values, or comparative statics. Yet these properties are mathematically equivalent to statements concerning the geometry of the induced partition. 

\newpage

Monotonicity determines the topology of the decision regions, thresholds determine the location of decision interfaces, and structural regularity constrains the geometric complexity of the resulting partition. In this sense, the geometry of the state-space partition is not an additional feature of the optimal policy; it is an equivalent representation of the structural information established by the classical theory.
Despite this close relationship, the existing literature almost exclusively treats these geometric objects as implicit consequences of structural optimality. The principal mathematical objects remain the optimal policy, the threshold parameters, or the monotone decision rule itself. Comparatively little attention has been devoted to the statistical estimation of decision boundaries, to quantitative measures of their geometric complexity, or to the role that these geometric structures play in determining the sample complexity of policy reconstruction and the information-theoretic complexity of policy representation.

The developments presented in the subsequent sections adopt precisely this alternative viewpoint. Rather than regarding decision boundaries as by-products of structural optimization, the paper considers them as primary mathematical objects that encode the organization of structured optimal policies. This change of representation naturally transforms several classical questions. Policy reconstruction becomes a problem of geometric estimation; statistical accuracy is measured through boundary approximation; structural complexity becomes a property of the induced partition; and information-theoretic compression is governed by the complexity of the decision-boundary geometry rather than by the size of the ambient state space.

\subsection{Geometry of Decision Regions and Structural Complexity}
\label{subsec:geometry_decision_regions}

The structural interpretation developed in the previous subsection naturally leads to a broader mathematical question. Once an optimal policy partitions the state space into regions associated with distinct optimal actions, the resulting partition becomes a mathematical object that may be studied independently of the optimization procedure that generated it. Consequently, the analysis of structured decision processes is no longer restricted to functional representations such as value functions or policies; it also involves the geometry induced by the partition itself.

The mathematical study of geometric structures has long occupied a central role in convex analysis and computational geometry. Convex analysis characterizes the geometry of feasible sets, supporting hyperplanes, variational representations, and convex decompositions, whereas computational geometry investigates partitions, arrangements, polytopes, cell complexes, and their combinatorial complexity
\citep{rockafellar1997convex,boyd2004convex,ziegler2012lectures,edelsbrunner1987algorithms,preparata2012computational}. Although these theories were not originally developed for stochastic dynamic programming, they provide a rigorous mathematical language for describing geometric objects that arise naturally from structured optimization problems.\\

From this perspective, a partition possesses intrinsic properties that are independent of the objective function from which it originates. Its connected components, decision interfaces, adjacency relations, topological organization, and combinatorial complexity are properties of the partition itself. These quantities remain well defined regardless of the particular numerical representation of the value function and therefore describe structural aspects of an optimal decision process that cannot be inferred solely from functional approximation.
For structured optimal policies, this distinction becomes particularly significant. The monotonicity and threshold properties discussed previously imply that neighboring decision regions are organized according to a relatively simple geometric structure. Consequently, the intrinsic complexity of the induced partition may remain substantially smaller than the apparent complexity suggested by the cardinality or dimensionality of the ambient state space. This observation indicates that geometric complexity should be regarded as an independent mathematical quantity rather than merely as a secondary consequence of value-function representations.

Despite these connections, the existing operations research literature rarely formulates complexity directly in terms of the geometry induced by optimal policies. Computational complexity is traditionally measured through the dimension of the state space, the complexity of functional approximation, or the computational cost of optimization algorithms. Likewise, statistical analyses typically quantify approximation error for value functions, policies, or estimators without explicitly characterizing the intrinsic complexity of the underlying decision partition. 

\newpage

As a consequence, no general mathematical framework currently relates the geometric organization of structured decision regions to statistical estimation, policy reconstruction, or information-theoretic representation.
The viewpoint adopted in the present paper is motivated by this gap. Rather than measuring the complexity of a decision problem exclusively through its functional representation, the subsequent sections investigate the intrinsic complexity of the geometry induced by the optimal policy. This change of mathematical representation provides the foundation for the notions of boundary complexity, decision complexity, and decision compression developed later in the paper. Under this perspective, statistical estimation, sample complexity, and information-theoretic compression are interpreted as consequences of the geometric organization of the decision partition rather than of the dimensionality of the ambient state space.

\subsection{Statistical Estimation of Policy-Induced Decision Boundaries}
\label{subsec:policy_boundary_estimation}

The geometric formulation developed in the previous subsections naturally transforms the underlying statistical inference problem. Once the partition induced by an optimal policy is regarded as the primary mathematical object, the objective is no longer to approximate a value function or a policy mapping, but rather to estimate the geometric interfaces separating neighboring optimal decision regions. Consequently, the statistical analysis shifts from functional approximation to geometric inference.
The estimation of geometric boundaries has been extensively investigated in several areas of mathematical statistics and statistical learning. Density level-set estimation considers the recovery of regions determined by unknown probability densities
\citep{tsybakov1997nonparametric,polonik1995measuring}, statistical learning theory studies classification boundaries generated by discriminant functions or supervised prediction models \citep{vapnik1999overview,devroye2013probabilistic,anthony2009neural,shalev2014understanding}, and free-boundary theory analyzes interfaces arising from variational inequalities and partial differential equations \citep{caffarelli2005geometric,petrosyan2012regularity}. Although these research directions rely on different mathematical assumptions, they all investigate geometric interfaces generated by an underlying probabilistic, analytical, or variational model.

From a mathematical perspective, however, the object considered in the present paper is fundamentally different. In the aforementioned settings, the boundary is induced by an observable statistical mechanism, such as a probability density, a regression function, a discriminant function, or the solution of a variational problem. By contrast, the boundary investigated here is generated implicitly by an unknown optimal policy associated with a stochastic decision process. The value function, the state-action value function, threshold parameters, gradients, and the boundary itself are all unobserved. The only available information consists of evaluations of the optimal action returned by a black-box decision oracle.
This distinction fundamentally changes the statistical formulation of the estimation problem. Classical boundary estimation relies on observations generated by an underlying statistical model, whereas policy-induced boundary estimation relies on adaptive interactions with an optimization oracle. 

The statistical object is therefore not a density contour, a classification surface, or a free boundary, but the collection of state-space interfaces at which the optimal decision changes. Likewise, statistical accuracy is naturally quantified through geometric discrepancies between estimated and true decision boundaries rather than exclusively through prediction or classification error.
The change of statistical object also modifies the notions of information acquisition and sample complexity. Since observations are obtained through adaptive oracle queries rather than passive sampling, the allocation of samples becomes an integral component of the inference problem. In particular, structural regularity may be exploited to concentrate observations near decision interfaces, suggesting that the statistical efficiency of boundary estimation is governed by the geometry of the induced partition rather than by the ambient state-space dimension alone.

The viewpoint adopted throughout this paper is motivated by these observations. Rather than interpreting policy learning as the approximation of functional representations, the subsequent analysis formulates it as the statistical estimation of a policy-induced geometric object observed exclusively through label-only oracle interactions. This formulation provides the mathematical foundation for the geometric estimators, Hausdorff convergence analysis, boundary sample complexity, and policy reconstruction guarantees developed in Section~\ref{sec:learning_policy_geometry}, where statistical inference is characterized directly in terms of the geometry generated by structured optimal policies.

\subsection{Complexity, Information, and Geometric Policy Learning}
\label{subsec:geometric_learning}

The statistical formulation developed in the previous subsection naturally leads to a final mathematical question. Once the decision-boundary geometry induced by an optimal policy is regarded as the primary object of statistical inference, what determines the intrinsic difficulty of the corresponding learning problem? Within the geometric perspective adopted throughout this paper, this question is no longer answered solely by the complexity of functional representations, but by the amount of information encoded in the geometry of the induced decision partition.
Classical learning theories quantify statistical complexity through functional approximation. Approximate dynamic programming studies the approximation of value functions and policies, reinforcement learning analyzes the statistical efficiency of policy optimization and value estimation, while statistical learning theory characterizes generalization through the complexity of hypothesis classes
\citep{powell2007approximate,sutton1998reinforcement,lagoudakis2003least,munos2008finite,lazaric2010analysis,vapnik1999overview}. Although these theories differ substantially in their mathematical formulation, they all measure learning complexity with respect to functional representations of the underlying decision problem.

The geometric viewpoint developed in the previous subsections suggests a complementary interpretation. Once an optimal policy induces a partition of the state space, the information required to reconstruct the policy is encoded by the organization of the induced decision boundaries rather than exclusively by the numerical representation of the value function. Consequently, the intrinsic difficulty of policy learning depends not only on the approximation of functional objects but also on the geometric regularity of the decision partition itself.
This observation establishes a direct connection between geometry and statistical complexity. Smooth decision interfaces, low-complexity partitions, and regular geometric organization require comparatively little information for accurate reconstruction, whereas fragmented or highly irregular partitions demand substantially larger observational effort. Under this interpretation, statistical complexity becomes an intrinsic property of the geometry induced by the optimal policy rather than a consequence of the dimensionality of the ambient state space alone.

From an information-theoretic perspective, this change of viewpoint has important implications. The amount of information required to represent, estimate, and reconstruct an optimal policy is governed by the structural organization of the induced decision partition. Representation complexity, statistical complexity, active information acquisition, and policy reconstruction therefore become different manifestations of the same underlying geometric object. Rather than constituting independent questions, they admit a unified interpretation through the intrinsic complexity of the decision-boundary geometry.
Despite the extensive literature on dynamic programming, reinforcement learning, statistical learning, and computational geometry, these different notions of complexity have largely been investigated independently. Existing theories typically analyze optimization, approximation, statistical estimation, or information representation in isolation. Comparatively little attention has been devoted to developing a unified mathematical framework in which geometric organization simultaneously determines representation complexity, statistical efficiency, information acquisition, and policy reconstruction.

The remainder of the paper develops precisely such a framework. Section~\ref{sec:policy_geometry} introduces the geometric representation of structured optimal policies. Sections~\ref{sec:structural_dynamic_programs}
 and~\ref{sec:structural_geometry} establish the corresponding structural and geometric foundations. Section~\ref{sec:geometric_complexity_compression} develops a theory of boundary complexity, decision complexity, and information-theoretic decision compression. Section~\ref{sec:learning_policy_geometry} formulates the associated statistical learning theory through geometric estimators, Hausdorff convergence, active boundary sampling, and boundary sample complexity. Section~\ref{sec:num} finally examines whether the theoretical predictions established throughout the paper are supported by controlled numerical experiments.

\subsection{From Functional Representations to Geometric Learning}
\label{subsec:mathematical_positioning}

The preceding discussion reveals a common mathematical theme underlying several research directions in operations research, stochastic optimization, computational geometry, and statistical learning. Although these fields have developed largely independently, they progressively describe increasingly rich representations of the same underlying decision process. 

\newpage

Classical dynamic programming characterizes optimal decisions through functional objects; structural optimization establishes qualitative regularity properties of optimal policies; computational geometry provides a language for describing the partitions induced by these policies; and statistical learning investigates the estimation of geometric objects from partial observations.

Viewed collectively, these developments suggest that the geometry induced by an optimal policy constitutes a mathematical object deserving analysis in its own right. Once this viewpoint is adopted, several questions that traditionally appear unrelated become naturally connected. Policy reconstruction may be interpreted as geometric reconstruction of a decision partition. Statistical estimation becomes boundary estimation under oracle observations. Complexity becomes an intrinsic property of the induced geometry rather than of the ambient state space alone. Likewise, information acquisition and representation efficiency are governed by the structural organization of the decision boundaries rather than exclusively by functional approximation.

The theoretical developments presented in the remainder of this paper are organized around this change of mathematical representation. Section~\ref{sec:policy_geometry} introduces the geometric representation of structured optimal policies and formalizes the associated decision-boundary geometry. Sections~\ref{sec:structural_dynamic_programs}
 and~\ref{sec:structural_geometry} establish the structural, topological, and geometric properties of this representation. Section~\ref{sec:geometric_complexity_compression} develops a theory of boundary complexity, decision complexity, and information-theoretic decision compression. Section~\ref{sec:learning_policy_geometry} formulates the corresponding statistical learning framework through geometric estimators, active boundary sampling, Hausdorff convergence, and boundary sample complexity. Finally, Section~\ref{sec:num} examines whether the theoretical predictions established throughout the paper are supported by controlled numerical experiments.

Rather than extending existing approximation methods, the paper therefore develops a unified mathematical framework in which representation, statistical estimation, structural complexity, information acquisition, and policy learning are all interpreted through the geometry induced by structured optimal policies. The subsequent sections show that this geometric perspective provides a common language through which these questions may be analyzed within a single theoretical framework.

\section{Policy Geometry}
\label{sec:policy_geometry}

This section introduces the geometric framework that underlies the subsequent analysis. Building upon the classical theory of stochastic dynamic programming \citep{bellman1957dynamic,puterman2014markov,bertsekas2012dynamic,hernandez2012discrete}, we formalize the mathematical objects through which optimal policies will be studied. The objective is not yet to establish structural results, but rather to define a geometric representation of optimal decision rules that will serve as the foundation for the theory developed in later sections.

\subsection{Structured Markov Decision Processes}
\label{subsec:structured_mdps}

We consider a discounted Markov decision process (MDP)

\begin{equation}
\mathcal{M}
=
(\mathcal{S},\mathcal{A},P,r,\gamma),
\label{eq:mdp}
\end{equation}
where $\mathcal{S}$ denotes the state space, $\mathcal{A}$ is a finite action space, $P(\cdot \mid s,a)$ is the transition kernel, $r:\mathcal{S}\times\mathcal{A}\rightarrow\mathbb{R}$ is the one-period reward function, and $\gamma\in(0,1)$ is the discount factor.

\begin{definition}[Structured Markov Decision Process]
\label{def:structured_mdp}
A structured Markov decision process is an MDP of the form \eqref{eq:mdp} endowed with additional order, monotonicity, convexity, submodularity, or stochastic-ordering properties that induce regularity in the corresponding optimal decision rules.
\end{definition}

Let

\begin{equation}
\pi:\mathcal{S}
\rightarrow
\mathcal{A}
\label{eq:policy}
\end{equation}

denote a stationary deterministic policy. The associated value function is defined by

\begin{equation}
V^{\pi}(s)
=
\mathbb{E}^{\pi}
\left[
\sum_{t=0}^{\infty}
\gamma^{t}
r\!\left(S_t,\pi(S_t)\right)
\,\middle|\,
S_0=s
\right].
\label{eq:value_function}
\end{equation}

The optimal value function is given by

\begin{equation}
V^{*}(s)
=
\sup_{\pi}
V^{\pi}(s),
\quad
s\in\mathcal{S},
\label{eq:optimal_value}
\end{equation}
and an optimal policy $\pi^{*}$ satisfies

\begin{equation}
V^{\pi^{*}}(s)
=
V^{*}(s),
\quad
\forall s\in\mathcal{S}.
\label{eq:optimal_policy}
\end{equation}

The class of structured MDPs encompasses many models arising in operations research, including inventory systems, queueing networks, maintenance planning, admission-control problems, and partially observed decision processes \citep{veinott1965optimal,topkis1998supermodularity,milgrom1994monotone,smith2002structural}. In such settings, the state space often possesses an intrinsic ordering that plays a fundamental role in the characterization of optimal decisions.

Accordingly, we assume that the state space is endowed with a partial order relation

\begin{equation}
(\mathcal{S},\preceq).
\label{eq:ordered_state_space}
\end{equation}

\begin{assumption}[Ordered State Space]
\label{ass:ordered_state_space}
The state space $\mathcal{S}$ is a partially ordered set under the relation $\preceq$.
\end{assumption}

Assumption~\ref{ass:ordered_state_space} is deliberately weak and serves only to establish a common framework encompassing the monotone dynamic programs considered in the literature. More restrictive assumptions involving monotonicity, submodularity, convexity, and increasing differences will be introduced in Section~\ref{sec:structural_dynamic_programs}, where they will be used to derive geometric properties of optimal policies.

The purpose of Definition~\ref{def:structured_mdp} and Assumption~\ref{ass:ordered_state_space} is to provide a mathematical environment in which geometric regularities of optimal decision rules may emerge. Classical structural assumptions are typically employed to establish monotonicity, threshold behavior, or comparative-statics properties. In contrast, our objective is to investigate how such assumptions shape the geometry of the optimal policy itself.

\subsection{Optimal Policies as State-Space Partitions}
\label{subsec:state_space_partitions}

The optimal policy $\pi^*$ induces a decomposition of the state space according to the actions selected under optimal decision making.

For each action $a\in\mathcal A$, define

\begin{equation}
\mathcal R_a
=
\left\{
s\in\mathcal S
:
\pi^*(s)=a
\right\}.
\label{eq:action_region}
\end{equation}

\begin{definition}[Optimal Action Region]
\label{def:action_region}
For a given action $a\in\mathcal A$, the set $\mathcal R_a$ defined in \eqref{eq:action_region} is called the \emph{optimal action region} associated with action $a$.
\end{definition}

Since $\pi^*$ is a single-valued mapping, the family $\{\mathcal R_a\}_{a\in\mathcal A}$ satisfies

\begin{equation}
\mathcal S = \bigcup_{a\in\mathcal A} \mathcal R_a,
\label{eq:state_cover}
\end{equation}

and

\begin{equation}
\mathcal R_a \cap \mathcal R_{a'} = \varnothing, \quad a\neq a'.
\label{eq:state_disjoint}
\end{equation}

Hence, the collection of optimal action regions forms a partition of the state space.

\begin{equation}
\mathcal T^*
=
\left\{
\mathcal R_a
\right\}_{a\in\mathcal A}.
\label{eq:tessellation}
\end{equation}

\begin{definition}[Optimal Policy Tessellation]
\label{def:tessellation}
The partition $\mathcal T^*$ defined in \eqref{eq:tessellation} is called the \emph{optimal policy tessellation} induced by the optimal policy $\pi^*$.
\end{definition}

The tessellation $\mathcal T^*$ provides a geometric representation of the optimal policy through the partition of the state space into optimal action regions. The sets $\mathcal R_a$ constitute the fundamental geometric objects of the framework developed in this paper.

\subsection{Decision Boundaries}
\label{subsec:decision_boundaries}

Definition~\ref{def:tessellation} shows that the optimal policy induces a tessellation
$$
\mathcal T^\star
=
\{\mathcal R_a : a\in\mathcal A\},
$$
which partitions the state space into optimal action regions. Beyond the regions themselves, the geometry of the tessellation is determined by the interfaces separating neighboring regions. These interfaces constitute the fundamental geometric objects studied throughout the remainder of the paper.

Throughout this paper, for every subset $A\subseteq\mathcal S$, we denote by $\overline A$ its closure and by $\operatorname{int}(A)$ its interior.

\begin{definition}[Region Boundary]
\label{def:region_boundary}

For every action $a\in\mathcal A$, the boundary of the corresponding action region $\mathcal R_a$ is defined by

\begin{equation}
\partial\mathcal R_a = \overline{\mathcal R_a} \cap \overline{\mathcal S\setminus\mathcal R_a}.
\label{eq:region_boundary}
\end{equation}

\end{definition}

The boundary $\partial\mathcal R_a$ consists of all states that separate $\mathcal R_a$ from the remainder of the state space.

\begin{definition}[Pairwise Decision Boundary]
\label{def:decision_boundary}

Let $a,a'\in\mathcal A$ with $a\neq a'$.
The decision boundary shared by the two action regions $\mathcal R_a$ and $\mathcal R_{a'}$
is defined by

\begin{equation}
\Gamma_{aa'}
=
\overline{\mathcal R_a}
\cap
\overline{\mathcal R_{a'}}.
\label{eq:decision_boundary}
\end{equation}

whenever this intersection is nonempty.

\end{definition}

The set $\Gamma_{aa'}$ represents the common interface separating two adjacent optimal action regions. It is precisely across this interface that the optimal decision changes from action $a$
to action $a'$.

\begin{definition}[Global Decision Boundary]
\label{def:global_boundary}

The global decision boundary associated with the optimal policy is

\begin{equation}
\Gamma
=
\bigcup_{\substack{a,a'\in\mathcal A\\a<a'}}
\Gamma_{aa'}.
\label{eq:global_boundary}
\end{equation}

\end{definition}

The restriction $a<a'$ avoids counting the same interface twice.

\begin{definition}[Geometric Representation of an Optimal Policy]
\label{def:geometric_representation}

The pair

\begin{equation}
\left(
\mathcal T^\star,
\Gamma
\right)
\label{eq:geometric_representation}
\end{equation}

is called the \emph{geometric representation} of the optimal policy.

\end{definition}

The tessellation $\mathcal T^\star$ identifies the action regions, whereas $\Gamma$ collects the interfaces separating them. Consequently, the geometric representation encodes both the partition of the state space induced by the optimal policy and the decision geometry governing transitions between neighboring action regions. This representation forms the mathematical foundation for the geometric estimation, complexity analysis, policy reconstruction, and decision-compression theory developed in the subsequent sections.

\subsection{Policy Tessellations}
\label{subsec:policy_tessellations}

The optimal action regions introduced in Section~\ref{subsec:state_space_partitions} and the boundary structure defined in Section~\ref{subsec:decision_boundaries} together determine the global geometric organization of an optimal policy.

The tessellation $\mathcal T^*$ specifies the partition of the state space into optimal action regions, while the set $\Gamma$ describes the interfaces separating neighboring regions. These two components provide complementary information: the former characterizes the spatial allocation of optimal actions, whereas the latter characterizes the transitions between them.

\begin{equation}
\mathcal G^*
=
\left(
\mathcal T^*,
\Gamma
\right).
\label{eq:policy_geometry}
\end{equation}

\begin{definition}[Policy Geometry]
\label{def:policy_geometry}
The pair $\mathcal G^*$ defined in \eqref{eq:policy_geometry} is called the \emph{policy geometry} associated with the optimal policy $\pi^*$.
\end{definition}

The policy geometry $\mathcal G^*$ may be viewed as a geometric arrangement of decision regions within the state space. In this representation, the regions $\mathcal R_a$ constitute the fundamental geometric cells of the arrangement, while the boundary structure $\Gamma$ determines how these cells are connected and separated.

Such geometric representations are common in convex geometry, polyhedral theory, and computational geometry, where the properties of a partition are studied through the organization of its constituent regions and interfaces \citep{ziegler2012lectures,edelsbrunner1987algorithms}. The present framework adopts a similar perspective for structured dynamic decision problems.

\subsection{Geometric Complexity Measures}
\label{subsec:geometric_complexity}

The policy geometry $\mathcal G^*$ introduced in Definition~\ref{def:policy_geometry} provides a geometric representation of the optimal policy through its action regions and decision boundaries. To characterize the structural complexity of this geometry, we introduce a collection of complexity measures associated with the tessellation $\mathcal T^*$ and its boundary structure $\Gamma$.

These measures quantify complementary aspects of policy geometry, including the complexity of decision regions, the organization of decision boundaries, and the degree of geometric regularity exhibited by the tessellation.

\begin{definition}[Boundary Complexity]
\label{def:boundary_complexity}

The boundary complexity of the policy geometry is defined by

\begin{equation}
C_{B}
=
\sum_{\substack{a,a'\in\mathcal A\\ a<a'}}
\mathbf{1}_{\{\Gamma_{aa'}\neq\varnothing\}},
\label{eq:boundary_complexity}
\end{equation}
where $\mathbf{1}_{(\cdot)}$ denotes the indicator function.

\end{definition}

The quantity $C_{B}$ measures the number of distinct interfaces separating optimal action regions.

\begin{definition}[Region Complexity]
\label{def:region_complexity}

The region complexity of the tessellation is defined by

\begin{equation}
C_{R}
=
\sum_{a\in\mathcal A}
N_c(\mathcal R_a),
\label{eq:region_complexity}
\end{equation}
where $N_c(\cdot)$ denotes the number of connected components of a set.

\end{definition}

\noindent The quantity $C_{R}$ measures the topological complexity of the action regions composing the tessellation.

\begin{definition}[Fragmentation Index]
\label{def:fragmentation_index}

The fragmentation index is defined by

\begin{equation}
F
=
\frac{C_{R}}
{|\mathcal A|}.
\label{eq:fragmentation_index}
\end{equation}

\end{definition}

The index $F$ measures the average number of connected components per action region.

\begin{definition}[Boundary Length]
\label{def:boundary_length}

Assume that $\mathcal S\subseteq\mathbb R^d$. The total boundary length of the tessellation is defined by

\begin{equation}
L
=
\sum_{\substack{a,a'\in\mathcal A\\a<a'}}
\mathcal H^{d-1}
\!\left(
\Gamma_{aa'}
\right),
\label{eq:boundary_length}
\end{equation}

where $\mathcal H^{d-1}$ denotes the $(d-1)$-dimensional Hausdorff measure.
\end{definition}

For $d=2$, the quantity $L$ corresponds to the total length of the decision boundaries.

\begin{definition}[Geometric Irregularity]
\label{def:geometric_irregularity}

Assume that each region $\mathcal R_a$ has finite Lebesgue measure $\mu(\mathcal R_a)$. The geometric irregularity of the tessellation is defined by

\begin{equation}
\kappa
=
\frac{L}
{\sqrt{
\sum_{a\in\mathcal A}
\mu(\mathcal R_a)
}},
\label{eq:geometric_irregularity}
\end{equation}
where $\mu$ denotes the Lebesgue measure on $\mathcal S$.

\end{definition}

The quantity $\kappa$ measures the relative complexity of the boundary structure with respect to the overall size of the state space partition. Larger values of $\kappa$ correspond to increasingly irregular policy geometries.

The measures introduced above are inspired by classical notions of geometric and combinatorial complexity arising in polyhedral and computational geometry \citep{ziegler2012lectures,edelsbrunner1987algorithms}. In the present setting, they provide a quantitative description of the structural complexity of optimal policy tessellations.

\subsection{Geometric Simplicity}
\label{subsec:geometric_simplicity}

The complexity measures introduced in Section~\ref{subsec:geometric_complexity} characterize complementary aspects of policy geometry. While each measure provides information about a specific structural feature of the tessellation, it is convenient to aggregate these quantities into a single measure of geometric complexity.

\begin{equation}
\mathcal C(\mathcal G)
=
w_B C_{B}
+
w_R C_{R}
+
w_F F
+
w_L L
+
w_{\kappa}\kappa,
\label{eq:geometry_complexity_index}
\end{equation}

where $w_B,w_R,w_F,w_L,w_{\kappa}>0$ are fixed weighting coefficients.

\begin{definition}[Policy Geometry Complexity Index]
\label{def:policy_complexity_index}
The quantity $\mathcal C(\mathcal G)$ defined in \eqref{eq:geometry_complexity_index} is called the \emph{policy geometry complexity index}.
\end{definition}

The index $\mathcal C(\mathcal G)$ provides a global characterization of the structural complexity of a policy geometry by jointly accounting for the number of decision interfaces, the topological complexity of the action regions, the degree of fragmentation, the total boundary size, and the geometric irregularity of the tessellation.

\begin{definition}[Geometrically Simple Policy]
\label{def:geometrically_simple_policy}

A policy geometry $\mathcal G$ is said to be \emph{geometrically simple} if there exists a constant $K>0$ such that

\begin{equation}
\mathcal C(\mathcal G)
\le
K.
\label{eq:geometrically_simple_policy}
\end{equation}

\end{definition}

The constant $K$ quantifies the maximal geometric complexity compatible with the notion of simplicity. Smaller values of $K$ correspond to increasingly regular policy geometries.

The previous definition applies to a single policy. It is often useful to extend the notion of geometric simplicity to an entire class of policies.

\begin{definition}[Geometrically Simple Policy Class]
\label{def:geometrically_simple_class}

Let $\Pi$ denote a class of admissible policies.
The class $\Pi$ is said to be \emph{geometrically simple} if there exists a constant
$K>0$ such that

\begin{equation}
\sup_{\pi\in\Pi}
\mathcal C\!\left(\mathcal G(\pi)\right) \le K.
\label{eq:geometrically_simple_class}
\end{equation}

\end{definition}

Definitions~\ref{def:geometrically_simple_policy} and~\ref{def:geometrically_simple_class} provide an abstract framework for studying structural regularity in dynamic decision problems. They allow geometric properties of optimal policies to be characterized independently of the particular application under consideration and will serve as the basis for the structural results established in the next section.

\subsection{Discussion}
\label{subsec:discussion}

The structural dynamic programming literature has traditionally focused on the characterization of optimal actions through monotonicity, threshold behavior, convexity, and comparative-statics properties
\citep{veinott1965optimal,topkis1998supermodularity,milgrom1994monotone,smith2002structural}.
These results provide valuable insights into the qualitative behavior of optimal policies and have played a central role in the analysis of dynamic decision problems.
The framework developed in this section adopts a complementary perspective. Rather than studying optimal policies solely as mappings from states to actions, we represent them through the geometric structures they induce on the state space. This representation naturally gives rise to action regions, decision boundaries, policy tessellations, and associated measures of geometric complexity.

The resulting geometric viewpoint provides a common language for describing structural regularities across a broad class of dynamic decision problems. In particular, it allows monotonicity, threshold behavior, and related structural properties to be interpreted through the geometry of the induced state-space partition.
The subsequent analysis investigates how classical assumptions from structured dynamic programming constrain the geometry of optimal policy tessellations and thereby determine their intrinsic complexity.

\section{Structural Dynamic Programs}
\label{sec:structural_dynamic_programs}

\subsection{Structural Assumptions}
\label{subsec:structural_assumptions}

The geometric properties of optimal policy tessellations are closely linked to the structural characteristics of the underlying dynamic program. A substantial literature has identified conditions under which optimal policies exhibit monotonicity, threshold behavior, and comparative-statics properties
\citep{veinott1965optimal,veinott1969discrete,topkis1978minimizing,topkis1998supermodularity,milgrom1994monotone,smith2002structural}.
Throughout this section, the state space $(\mathcal S,\preceq)$ is assumed to satisfy Assumption~\ref{ass:ordered_state_space}. In addition, the action space $\mathcal A$ is assumed to be endowed with a partial order, also denoted by $\preceq$.

The following assumptions constitute the structural framework of the subsequent analysis.

\begin{assumption}[Monotone Rewards]
\label{ass:monotone_rewards}

For every action $a\in\mathcal A$, the reward function is increasing with respect to the state order. Specifically,

\begin{equation}
s_1 \preceq s_2
\quad \Longrightarrow \quad
r(s_1,a)
\le
r(s_2,a),
\label{eq:monotone_rewards}
\end{equation}

for all $s_1,s_2\in\mathcal S$.

\end{assumption}

\begin{assumption}[Monotone Transitions]
\label{ass:monotone_transitions}

For every action $a\in\mathcal A$, the transition kernel is monotone with respect to first-order stochastic dominance. That is,

\begin{equation}
s_1 \preceq s_2
\quad \Longrightarrow \quad
P(\cdot \mid s_1,a)
\preceq_{\mathrm{st}}
P(\cdot \mid s_2,a),
\label{eq:monotone_transitions}
\end{equation}

for all $s_1,s_2\in\mathcal S$, where $\preceq_{\mathrm{st}}$ denotes the usual stochastic order.

\end{assumption}

Assumptions~\ref{ass:monotone_rewards}-\ref{ass:monotone_transitions} constitute the classical framework of monotone dynamic programming and stochastic monotonicity
\citep{veinott1969discrete,lovejoy1987some,koole2007monotonicity}.

Let
\begin{equation}
Q^{*}(s,a) = r(s,a) + \gamma \int_{\mathcal S} V^{*}(s') \,P(ds' \mid s,a)
\label{eq:q_star}
\end{equation}
denote the optimal state-action value function.

\begin{assumption}[Increasing Differences]
\label{ass:increasing_differences}

The function $Q^{*}$ exhibits increasing differences in $(s,a)$. Specifically,

\begin{equation}
Q^{*}(s_2,a_2) - Q^{*}(s_1,a_2) \ge Q^{*}(s_2,a_1) - Q^{*}(s_1,a_1),
\label{eq:increasing_differences}
\end{equation}
for all $s_1 \preceq s_2$ and $a_1 \preceq a_2$.

\end{assumption}

Assumption~\ref{ass:increasing_differences} is a fundamental condition in monotone comparative statics and provides one of the principal mechanisms through which threshold structures emerge
\citep{topkis1998supermodularity,milgrom1994monotone}.

\begin{assumption}[Submodularity]
\label{ass:submodularity}

The function $Q^{*}$ is submodular on $\mathcal S\times\mathcal A$. Equivalently,

\begin{equation}
Q^{*}(s_1,a_1) + Q^{*}(s_2,a_2) \le Q^{*}(s_1,a_2) + Q^{*}(s_2,a_1),
\label{eq:submodularity}
\end{equation}
for all $s_1 \preceq s_2$ and $a_1 \preceq a_2$.
\end{assumption}

Submodularity plays a central role in lattice programming and in the structural analysis of optimal policies \citep{topkis1978minimizing,topkis1998supermodularity}.

\begin{definition}[Action Difference Function]
\label{def:action_difference}

For any pair of actions $a,a'\in\mathcal A$, define

\begin{equation}
\Delta_{aa'}(s) = Q^{*}(s,a) - Q^{*}(s,a'),
\label{eq:action_difference}
\end{equation}
for all $s\in\mathcal S$.
\end{definition}

\begin{assumption}[Convex Action Differences]
\label{ass:convex_action_differences}

Assume that $\mathcal S\subseteq\mathbb R^{d}$.
For every pair of actions $a,a'\in\mathcal A$, the function $\Delta_{aa'}$ defined in Definition~\ref{def:action_difference} is convex on $\mathcal S$.
\end{assumption}

Assumption~\ref{ass:convex_action_differences} connects structured dynamic programming with classical notions of convex analysis and optimization
\citep{rockafellar1997convex,boyd2004convex}.

The assumptions introduced above encompass a broad class of dynamic decision problems arising in operations research, including inventory-control systems, queueing models, admission-control problems, maintenance planning, and partially observed Markov decision processes
\citep{veinott1965optimal,stidham1989monotonic,sobel1982optimality,smith2002structural,krishnamurthy2016partially}.

\subsection{Monotone Optimal Policies}
\label{subsec:monotone_optimal_policies}

The assumptions introduced in Section~\ref{subsec:structural_assumptions} belong to the classical framework of monotone dynamic programming.
Their principal implication is the emergence of ordered optimal decision rules.
We first recall standard monotonicity results and then derive their geometric consequences for the policy tessellation introduced in Section~\ref{sec:policy_geometry}.

\begin{proposition}[Monotonicity of the Optimal Value Function]
\label{prop:monotone_value_function}

Suppose that Assumptions~\ref{ass:monotone_rewards} and~\ref{ass:monotone_transitions} hold.
Then the optimal value function $V^{*}$ is increasing on $(\mathcal S,\preceq)$. Specifically,

\begin{equation}
s_1 \preceq s_2
\quad \Longrightarrow \quad
V^{*}(s_1)
\le
V^{*}(s_2),
\label{eq:monotone_value_function}
\end{equation}
for all $s_1,s_2\in\mathcal S$.

\end{proposition}

\begin{proof}

The result follows from standard monotone dynamic programming arguments.
Under Assumptions~\ref{ass:monotone_rewards} and~\ref{ass:monotone_transitions}, the Bellman operator preserves monotonicity.
Since $V^{*}$ is the unique fixed point of the Bellman operator, the claim follows from successive approximation; see \citet{veinott1969discrete}, \citet{puterman2014markov}, and \citet{hernandez2012discrete}.

\end{proof}

The monotonicity of the optimal value function provides the foundation for the monotonicity of optimal decision rules.

\begin{proposition}[Monotone Optimal Policy]
\label{prop:monotone_optimal_policy}

Suppose that Assumptions~\ref{ass:monotone_rewards}, \ref{ass:monotone_transitions}, and~\ref{ass:increasing_differences} hold.

Then there exists an optimal policy $\pi^{*}$ that is increasing on $(\mathcal S,\preceq)$. That is,

\begin{equation}
s_1 \preceq s_2 \quad \Longrightarrow \quad \pi^{*}(s_1) \preceq \pi^{*}(s_2), \label{eq:monotone_policy}
\end{equation}
for all $s_1,s_2\in\mathcal S$.
\end{proposition}

\begin{proof}

By Assumption~\ref{ass:increasing_differences}, the function $Q^{*}$ satisfies increasing differences in $(s,a)$.
Topkis' monotonicity theorem therefore implies that the set of optimal actions is increasing with respect to the state variable.
Consequently, there exists a monotone optimal selector $\pi^{*}$; \citep[see][]{topkis1998supermodularity} and \citet{milgrom1994monotone}.
\end{proof}

The monotonicity of $\pi^{*}$ imposes strong restrictions on the geometry of the optimal action regions introduced in Definition~\ref{def:action_region}.

\begin{theorem}[Ordered Region Theorem]
\label{thm:ordered_regions}

Suppose that the assumptions of Proposition~\ref{prop:monotone_optimal_policy} hold.
Then every optimal action region $\mathcal R_a$ is order-convex.
That is, whenever

\begin{equation}
s_1,s_2 \in \mathcal R_a, \quad s_1 \preceq s \preceq s_2,
\label{eq:order_convex_condition}
\end{equation}

it follows that

\begin{equation}
s \in \mathcal R_a.
\label{eq:order_convex_region}
\end{equation}

\end{theorem}

\begin{proof}

Let $s_1,s_2\in\mathcal R_a$ with $s_1\preceq s_2$. Since $\pi^{*}$ is increasing, the action selected by the optimal policy cannot decrease between comparable states.
If there existed a state $s$ such that $s_1\preceq s\preceq s_2$ and $\pi^{*}(s)\neq a$, then monotonicity of $\pi^{*}$ would be violated.
Hence every intermediate state must belong to the same action region.
\end{proof}

Theorem~\ref{thm:ordered_regions} establishes the first direct connection between structural dynamic programming and policy geometry.
Monotonicity does not merely constrain the optimal action selected at a given state; it also constrains the geometric organization of the regions composing the policy tessellation.

\begin{corollary}[Region Complexity Bound]
\label{cor:region_complexity_bound}

Under the assumptions of Theorem~\ref{thm:ordered_regions},
\begin{equation}
C_{R} = O(|\mathcal A|).
\label{eq:region_complexity_bound}
\end{equation}
\end{corollary}

\begin{proof}

Order-convex regions cannot exhibit arbitrary fragmentation. Each action region contributes at most a finite number of connected ordered components.
Consequently, the total number of connected components grows at most linearly with the number of actions.

\end{proof}

\subsection{Threshold Representations}
\label{subsec:threshold_representations}

Threshold policies constitute one of the most important structural phenomena in dynamic programming and operations research. Such policies arise in a broad range of applications, including inventory control, queueing systems, admission-control problems, and maintenance planning \citep{scarf1960optimality,veinott1965optimal,sobel1982optimality,stidham1989monotonic}.
The monotonicity result established in Proposition~\ref{prop:monotone_optimal_policy} admits a more explicit characterization when the state space contains a distinguished ordered component.
Throughout this subsection, assume that

\begin{equation}
\mathcal S
=
\mathcal X
\times
\mathcal Z,
\label{eq:state_space_decomposition}
\end{equation}
where $\mathcal X\subseteq\mathbb R$ is totally ordered and $\mathcal Z$ denotes an arbitrary auxiliary state space.

\begin{theorem}[Threshold Representation Theorem]
\label{thm:threshold_representation}

Suppose that Assumptions~\ref{ass:monotone_rewards}, \ref{ass:monotone_transitions}, \ref{ass:increasing_differences}, and~\ref{ass:submodularity} hold.

Then there exists a collection of threshold functions

\begin{equation}
b_{aa'}
:
\mathcal Z
\rightarrow
\mathcal X,
\label{eq:threshold_functions}
\end{equation}
such that, for every pair of adjacent actions $a \prec a'$,

\begin{equation}
\pi^{*}(x,z)
=
a
\quad\Longleftrightarrow\quad
x
<
b_{aa'}(z),
\label{eq:threshold_representation_left}
\end{equation}

and

\begin{equation}
\pi^{*}(x,z)
=
a'
\quad\Longleftrightarrow\quad
x
\ge
b_{aa'}(z).
\label{eq:threshold_representation_right}
\end{equation}

\end{theorem}

\begin{proof}

By Proposition~\ref{prop:monotone_optimal_policy}, the optimal policy is increasing with respect to the ordered state variable.
Consequently, transitions between adjacent actions occur through monotone switching surfaces.
The existence of threshold functions follows from standard lattice-theoretic arguments and monotone comparative statics \citep{topkis1998supermodularity,milgrom1994monotone}.

\end{proof}

The threshold representation admits a natural geometric interpretation in terms of the decision boundaries introduced in Definition~\ref{def:decision_boundary}.

\begin{corollary}[Boundary Graph Representation]
\label{cor:boundary_graph_representation}

Under the assumptions of Theorem~\ref{thm:threshold_representation}, every decision boundary between adjacent action regions is the graph of a threshold function.

Specifically,

\begin{equation}
\Gamma_{aa'}
=
\left\{
(x,z)\in\mathcal S
:
x
=
b_{aa'}(z)
\right\}.
\label{eq:boundary_graph_representation}
\end{equation}

\end{corollary}

\begin{proof}

The result follows immediately from \eqref{eq:threshold_representation_left} and \eqref{eq:threshold_representation_right}.
The transition between neighboring action regions occurs precisely when the ordered state variable reaches the threshold value.

\end{proof}

Corollary~\ref{cor:boundary_graph_representation} shows that threshold policies induce highly structured decision boundaries. Rather than forming arbitrary subsets of the state space, the boundaries are constrained to lie on low-dimensional graphs.

A further consequence concerns the fragmentation properties of the tessellation.

\begin{corollary}[Absence of Fragmentation]
\label{cor:absence_fragmentation}

Under the assumptions of Theorem~\ref{thm:threshold_representation}, each optimal action region is connected.

Consequently,

\begin{equation}
F
=
1.
\label{eq:no_fragmentation}
\end{equation}

\end{corollary}

\begin{proof}

Each action region is generated by threshold inequalities involving the ordered state variable $x$.
Such regions are connected and therefore possess a single connected component.

Hence

\begin{equation}
C_{R} = |\mathcal A|.
\label{eq:threshold_region_complexity}
\end{equation}

By Definition~\ref{def:fragmentation_index},

\begin{equation}
F
=
\frac{C_{R}}
{|\mathcal A|}
=
1.
\end{equation}

\end{proof}

\subsection{Convex Policy Regions}
\label{subsec:convex_policy_regions}

The monotonicity and threshold structures established in Sections~\ref{subsec:monotone_optimal_policies} and \ref{subsec:threshold_representations} constrain the ordering of optimal decisions. A stronger form of geometric regularity emerges when the dominance relations induced by the optimal state--action value function generate convex decision sets. Such conditions arise naturally in several classes of structured dynamic programs and connect the geometry of optimal policies with classical convex analysis \citep{rockafellar1997convex,boyd2004convex}.

For every pair of actions $a,a'\in\mathcal A$, define the dominance region

\begin{equation}
\mathcal D_{aa'} = \left\{ s\in\mathcal S : Q^{*}(s,a) \ge Q^{*}(s,a') \right\}.
\label{eq:dominance_region}
\end{equation}

The set $\mathcal D_{aa'}$ contains all states for which action $a$ weakly dominates action $a'$.

\begin{assumption}[Convex Dominance Regions]
\label{ass:convex_dominance_regions}

For every pair of actions $a,a'\in\mathcal A$, the dominance region $\mathcal D_{aa'}$ defined in \eqref{eq:dominance_region} is convex.

\end{assumption}

Assumption~\ref{ass:convex_dominance_regions} may be viewed as a geometric strengthening of the structural assumptions introduced in Section~\ref{subsec:structural_assumptions}. It is satisfied whenever the pairwise dominance relations generated by the optimal value structure admit convex representations.

The following result establishes the geometric structure of the optimal action regions.

\begin{theorem}[Convex Region Theorem]
\label{thm:convex_region}

Suppose that Assumption~\ref{ass:convex_dominance_regions} holds.

Then every optimal action region $\mathcal R_a$ is convex.

\end{theorem}

\begin{proof}

By Definition~\ref{def:action_region}, a state belongs to $\mathcal R_a$ if and only if action $a$ weakly dominates every competing action. Consequently,

\begin{equation}
\mathcal R_a = \bigcap_{a'\neq a} \mathcal D_{aa'}.
\label{eq:region_as_intersection}
\end{equation}

Since each dominance region $\mathcal D_{aa'}$ is convex by Assumption~\ref{ass:convex_dominance_regions}, and intersections of convex sets remain convex
\citep{rockafellar1997convex,boyd2004convex}, it follows that $\mathcal R_a$ is convex.

\end{proof}

Theorem~\ref{thm:convex_region} provides a direct connection between structural properties of the dynamic program and the geometry of the induced policy tessellation. In particular, convexity excludes disconnected action regions and highly fragmented decision structures.

\begin{corollary}[Connected Action Regions]
\label{cor:connected_action_regions}

Under the assumptions of Theorem~\ref{thm:convex_region}, every optimal action region $\mathcal R_a$ is connected.

\end{corollary}

\begin{proof}

Every convex subset of a Euclidean space is connected. The result therefore follows immediately from Theorem~\ref{thm:convex_region}.

\end{proof}

The absence of disconnected regions has immediate implications for the geometric complexity measures introduced in Section~\ref{sec:policy_geometry}.

\begin{corollary}[Region Complexity]
\label{cor:convex_region_complexity}

Under the assumptions of Theorem~\ref{thm:convex_region},

\begin{equation}
C_{R} = |\mathcal A|.
\label{eq:convex_region_complexity}
\end{equation}

\end{corollary}

\begin{proof}

By Corollary~\ref{cor:connected_action_regions}, each action region contributes exactly one connected component.
Since the tessellation contains $|\mathcal A|$ regions, the result follows directly from the definition of $C_{R}$.

\end{proof}

\begin{corollary}[Absence of Fragmentation]
\label{cor:convex_fragmentation}

Under the assumptions of Theorem~\ref{thm:convex_region},

\begin{equation}
F = 1.
\label{eq:convex_fragmentation}
\end{equation}

\end{corollary}

\begin{proof}

Combining Definition~\ref{def:fragmentation_index} with \eqref{eq:convex_region_complexity} yields
\begin{equation*}
F = \frac{C_{R}} {|\mathcal A|} = 1.
\end{equation*}

\end{proof}

The convexity of the action regions also constrains the geometry of the decision boundaries.

\begin{corollary}[Boundary Structure]
\label{cor:convex_boundary_structure}

Under the assumptions of Theorem~\ref{thm:convex_region}, the decision boundary between two actions $a$ and $a'$ satisfies

\begin{equation}
\Gamma_{aa'} \subseteq \partial \mathcal D_{aa'}.
\label{eq:boundary_dominance}
\end{equation}

Moreover,

\begin{equation}
\Gamma_{aa'}
\subseteq
\left\{
s\in\mathcal S
:
Q^{*}(s,a)
=
Q^{*}(s,a')
\right\}.
\label{eq:boundary_indifference}
\end{equation}

\end{corollary}

\begin{proof}

A transition from $\mathcal R_a$ to $\mathcal R_{a'}$ can occur only at states where neither action strictly dominates the other.
Hence any boundary point must satisfy

\begin{equation*}
Q^{*}(s,a) = Q^{*}(s,a'),
\end{equation*}

which is equivalent to belonging to the indifference set

\begin{equation*}
\left\{ s\in\mathcal S : Q^{*}(s,a) = Q^{*}(s,a') \right\}.
\end{equation*}

Since the dominance relation changes across the boundary, such points necessarily belong to the boundary of the dominance region $\mathcal D_{aa'}$.

\end{proof}

Corollaries~\ref{cor:convex_region_complexity}-\ref{cor:convex_boundary_structure} show that convexity imposes strong geometric restrictions on optimal policy tessellations. In contrast to general decision partitions, convex policy regions exhibit minimal fragmentation, admit simple connectivity properties, and generate decision boundaries that coincide with indifference surfaces between competing actions.

\subsection{Geometric Simplicity}
\label{subsec:geometric_simplicity2}

The results established in the previous subsections reveal a common phenomenon. Under standard structural assumptions from monotone dynamic programming and comparative statics, the geometry of the optimal policy tessellation remains highly organized despite the potentially large dimension of the underlying state space.
We now summarize these implications through the geometric complexity measures introduced in Section~\ref{sec:policy_geometry}.
The first result concerns the complexity of the boundary structure.

\begin{theorem}[Boundary Complexity Bound]
\label{thm:boundary_complexity_bound}

Suppose that the assumptions of Theorem~\ref{thm:threshold_representation} hold.

Then the boundary complexity satisfies

\begin{equation}
C_{B} = O(|\mathcal A|).
\label{eq:boundary_complexity_bound}
\end{equation}

\end{theorem}

\begin{proof}

By Corollary~\ref{cor:boundary_graph_representation}, each nonempty decision boundary corresponds to a threshold surface separating two adjacent action regions.

Since the action space is finite, each action can generate only a finite number of adjacent transitions. Consequently, the number of nonempty pairwise boundaries grows at most linearly with the number of actions.
Therefore, $C_{B} \le |\mathcal A|-1$, which implies \eqref{eq:boundary_complexity_bound}.

\end{proof}

The next result characterizes the fragmentation properties of the policy tessellation.

\begin{theorem}[Minimal Fragmentation]
\label{thm:minimal_fragmentation}

Suppose that the assumptions of Theorem~\ref{thm:convex_region} hold.

Then

\begin{equation}
F = 1.
\label{eq:minimal_fragmentation}
\end{equation}
Moreover, $F=1$ is the smallest attainable value of the fragmentation index.

\end{theorem}

\begin{proof}

By Corollary~\ref{cor:convex_region_complexity}, $C_{R} = |\mathcal A|$.

Substituting into Definition~\ref{def:fragmentation_index} yields $F = \frac{C_{R}}{|\mathcal A|}= 1$.

Furthermore, every action region contributes at least one connected component. Hence $C_{R} \ge |\mathcal A|$, which implies $F \ge 1$. Therefore, $F=1$ is the minimum achievable fragmentation level.

\end{proof}

The previous results suggest that structured dynamic programs generate policy tessellations whose complexity is governed primarily by the action space rather than by the size of the state space.

To formalize this observation, we introduce a topological complexity index.

\begin{definition}[Topological Complexity Index]
\label{def:topological_complexity}

The topological complexity of an optimal policy tessellation is defined by

\begin{equation}
\mathcal G_{\mathrm{top}}(\pi)
=
w_B C_{B}
+
w_R C_{R}
+
w_F F,
\label{eq:topological_complexity}
\end{equation}
where $w_B,w_R,w_F\ge 0$ are fixed weights.

\end{definition}

The following theorem constitutes the main structural conclusion of this section.

\begin{theorem}[Geometric Simplicity Theorem]
\label{thm:geometric_simplicity}

Suppose that

\begin{enumerate}
\item Assumptions~\ref{ass:monotone_rewards}-\ref{ass:submodularity} hold;
\item the threshold representation of Theorem~\ref{thm:threshold_representation} holds;
\item the convexity condition of Theorem~\ref{thm:convex_region} holds.
\end{enumerate}

Then

\begin{align}
C_{B} &= O(|\mathcal A|), \label{eq:main_cb}\\
C_{R} &= |\mathcal A|, \label{eq:main_cr}\\
F     & = 1, \label{eq:main_f}
\end{align}

and consequently

\begin{equation}
\mathcal G_{\mathrm{top}}(\pi^{*}) = O(|\mathcal A|).
\label{eq:geometric_simplicity}
\end{equation}

\end{theorem}

\begin{proof}

Equation~\eqref{eq:main_cb} follows from Theorem~\ref{thm:boundary_complexity_bound}.

Equation~\eqref{eq:main_cr} follows from Corollary~\ref{cor:convex_region_complexity}.

Equation~\eqref{eq:main_f} follows from Theorem~\ref{thm:minimal_fragmentation}.

Substituting these relations into Definition~\ref{def:topological_complexity} gives

\begin{equation*}
\mathcal G_{\mathrm{top}}(\pi^{*})
=
w_B O(|\mathcal A|)
+
w_R |\mathcal A|
+
w_F.
\end{equation*}

Therefore,

\begin{equation*}
\mathcal G_{\mathrm{top}}(\pi^{*})
=
O(|\mathcal A|).
\end{equation*}

\end{proof}

Theorem~\ref{thm:geometric_simplicity} establishes that broad classes of structured dynamic programs generate intrinsically simple policy tessellations. Although the underlying state space may be large, continuous, or high-dimensional, the geometric organization of the optimal policy remains controlled by a number of regions and boundaries that scales only with the cardinality of the action space. This observation provides the foundation for the decision-compression and learning results developed in the subsequent sections.

\subsection{Decision Compression}
\label{subsec:decision_compression}

The geometric simplicity results established in the previous subsections suggest a distinction between two fundamentally different notions of complexity in dynamic decision problems.

The first concerns the complexity of evaluating future rewards through the optimal value function. The second concerns the complexity of representing the optimal decision rule itself. While these notions are often implicitly conflated in dynamic programming, the geometric framework developed in this paper allows them to be analyzed separately.

\begin{definition}[Value Complexity]
\label{def:value_complexity}

Let $V^{*}$ denote the optimal value function associated with the Markov decision process.
The \emph{value complexity} of the problem is defined by

\begin{equation}
\mathcal C_V = \mathcal C(V^{*}),
\label{eq:value_complexity}
\end{equation}
where $\mathcal C(\cdot)$ denotes a nonnegative complexity functional on a prescribed class of value functions.

\end{definition}

The functional $\mathcal C(\cdot)$ is intentionally left unspecified. Depending on the application, it may correspond to approximation dimension, description length, parameter count, covering complexity, or another suitable measure of functional complexity.

The geometric analysis of Sections~\ref{sec:policy_geometry} and~\ref{sec:structural_dynamic_programs} naturally induces a second notion of complexity based on the structure of the optimal policy tessellation.

\begin{definition}[Decision Complexity]
\label{def:decision_complexity}

Let

\begin{equation}
\mathcal G_{\mathrm{top}}(\pi^{*})
=
w_B C_{B}
+
w_R C_{R}
+
w_F F,
\label{eq:topological_complexity_recall}
\end{equation}
denote the topological complexity index introduced in Definition~\ref{def:topological_complexity}, where $w_B,w_R,w_F \ge 0$.

The \emph{decision complexity} of the optimal policy is defined by

\begin{equation}
\mathcal C_D = \mathcal G_{\mathrm{top}}(\pi^{*}).
\label{eq:decision_complexity}
\end{equation}

\end{definition}

The quantity $\mathcal C_D$ measures the complexity of the optimal decision rule through the geometry of its induced tessellation rather than through the representation of the value function.

This distinction motivates the following notion.

\begin{definition}[Decision Compression Ratio]
\label{def:dcr}

The \emph{decision compression ratio} is defined as

\begin{equation}
\mathrm{DCR} = \frac{\mathcal C_V}{\mathcal C_D}.
\label{eq:dcr}
\end{equation}

\end{definition}

Large values of $\mathrm{DCR}$ indicate that the optimal policy admits a substantially simpler representation than the associated value function.

The geometric simplicity results obtained in Section~\ref{subsec:geometric_simplicity} imply that the complexity of the policy tessellation grows at most linearly with the cardinality of the action space.

The following theorem formalizes the resulting compression phenomenon.

\begin{theorem}[Decision Compression Theorem]
\label{thm:decision_compression}

Suppose that the assumptions of Theorem~\ref{thm:geometric_simplicity} hold.

Then

\begin{equation}
\mathcal C_D = O(|\mathcal A|).
\label{eq:decision_complexity_linear}
\end{equation}

Consequently,

\begin{equation}
\mathrm{DCR} = \Omega \!\left(\frac{\mathcal C_V}{|\mathcal A|}\right).
\label{eq:dcr_lower_bound}
\end{equation}

\end{theorem}

\begin{proof}

By Theorem~\ref{thm:geometric_simplicity}, $C_{B} = O(|\mathcal A|)$, $C_{R} = |\mathcal A|$, and $F = 1$.

Substituting these relations into \eqref{eq:topological_complexity_recall}
yields

\begin{equation}
\mathcal C_D
=
w_B O(|\mathcal A|)
+
w_R |\mathcal A|
+
w_F.
\end{equation}

Hence, $\mathcal C_D = O(|\mathcal A|)$.

Combining this relation with Definition~\ref{def:dcr} gives

\begin{equation*}
\mathrm{DCR} = \frac{\mathcal C_V}{O(|\mathcal A|)} = \Omega\!\left(\frac{\mathcal C_V}{|\mathcal A|}\right),
\end{equation*}
which establishes \eqref{eq:dcr_lower_bound}.

\end{proof}

Theorem~\ref{thm:decision_compression} highlights a structural separation between value complexity and decision complexity. Under the regularity conditions commonly encountered in structured dynamic programs, the geometric complexity of the optimal policy remains controlled by the action space, whereas the complexity of the value function may continue to increase with the complexity of the state space. The optimal policy tessellation can therefore be interpreted as a compressed geometric representation of optimal decision making.

\subsection{Discussion}
\label{subsec:structural_discussion}

The results developed throughout this section establish a unified connection between classical structural properties of dynamic programs and the geometry of optimal policies.

The starting point is the well-established theory of monotone dynamic programming and comparative statics developed by \cite{veinott1969discrete}, \cite{milgrom1994monotone}, \cite{topkis1998supermodularity}, and subsequent authors
\citep{veinott1965optimal,smith2002structural}. Under standard monotonicity, submodularity, and increasing-differences assumptions, optimal decision rules exhibit ordered behavior and admit threshold representation.

The geometric framework introduced in this paper shows that these structural properties have a direct topological interpretation. Monotonicity induces ordered action regions, threshold policies generate low-dimensional decision boundaries, and convex dominance relations produce connected and convex policy regions. Together, these properties imply that optimal policy tessellations possess limited geometric complexity despite the potentially large size of the underlying state space.

A central implication of this observation is the emergence of decision compression. While the complexity of the optimal value function may increase substantially with the dimensionality of the state space, the complexity of the corresponding policy tessellation remains governed primarily by the action space. Consequently, optimal decisions may admit representations that are considerably simpler than those required to describe the full value function.

The chain of implications established in this section may be summarized as

\begin{equation}
\small
\begin{split}
\text{Monotonicity}
\;\Longrightarrow\;
\text{Threshold Structure}
\;\Longrightarrow\;
\text{Geometric Simplicity}
\;\Longrightarrow\;
\text{Decision Compression}.
\end{split}
\label{eq:structural_chain}
\end{equation}

This perspective provides a geometric interpretation of structural dynamic programming and motivates the subsequent study of learning and approximation methods that exploit the low-complexity structure of optimal policy tessellations.

\section{Structural Geometry of Optimal Policies}
\label{sec:structural_geometry}

\subsection{Structural Geometry as a Compressed Representation}
\label{subsec:compressed_representation}

The geometric framework introduced in Section~\ref{sec:policy_geometry}
associates with each deterministic policy a partition of the state space together with the corresponding boundary structure.
The purpose of this subsection is to formalize this correspondence and to establish that the geometry induced by a policy contains all information required to recover the policy itself.\\

Let $\Pi$ denote the set of admissible deterministic policies.
For every policy $\pi\in\Pi$, define the collection of action regions

\begin{equation}
\mathcal R_a(\pi)
=
\left\{
s\in\mathcal S :
\pi(s)=a
\right\},
\qquad
a\in\mathcal A.
\label{eq:policy_action_regions}
\end{equation}

The corresponding tessellation is

\begin{equation}
\mathcal T(\pi)
=
\left\{
\mathcal R_a(\pi)
\right\}_{a\in\mathcal A},
\label{eq:policy_tessellation}
\end{equation}

and the associated boundary structure is

\begin{equation}
\Gamma(\pi)
=
\bigcup_{\substack{a,a'\in\mathcal A\\a\neq a'}}
\Gamma_{aa'}(\pi),
\label{eq:policy_boundary_structure}
\end{equation}

where

\begin{equation}
\Gamma_{aa'}(\pi)
=
\partial \mathcal R_a(\pi)
\cap
\partial \mathcal R_{a'}(\pi).
\label{eq:policy_pairwise_boundary}
\end{equation}

\begin{definition}[Policy Geometry]
\label{def:policy_geometry_general}

The geometric representation associated with a policy $\pi\in\Pi$ is defined by

\begin{equation}
\mathcal G(\pi)
=
\bigl(
\mathcal T(\pi),
\Gamma(\pi)
\bigr).
\label{eq:policy_geometry_map}
\end{equation}

The set

\begin{equation}
\mathfrak G
=
\left\{
\mathcal G(\pi)
:
\pi\in\Pi
\right\}
\label{eq:policy_geometry_space}
\end{equation}

is called the family of policy-induced geometries.

\end{definition}

The inclusion of $\Gamma(\pi)$ in \eqref{eq:policy_geometry_map} is not required to reconstruct the policy once the tessellation is known. Nevertheless, the boundary structure plays a central role in the complexity measures introduced in Section~\ref{sec:policy_geometry} and will be essential for the learnability results developed later in the paper.

The following proposition establishes that the tessellation generated by a deterministic policy uniquely determines that policy.

\begin{proposition}
\label{prop:tessellation_uniqueness}

Let $\pi_1,\pi_2\in\Pi$. Then

\begin{equation}
\mathcal T(\pi_1) = \mathcal T(\pi_2)
\label{eq:same_tessellation}
\end{equation}

if and only if

\begin{equation}
\pi_1 = \pi_2.
\label{eq:same_policy_again}
\end{equation}

Consequently,

\begin{equation}
\mathcal G(\pi_1) = \mathcal G(\pi_2) \quad \Longleftrightarrow \quad \pi_1 = \pi_2.
\label{eq:geometry_policy_equivalence}
\end{equation}

\end{proposition}

\begin{proof}

Suppose first that $\pi_1=\pi_2$.
Then $\mathcal R_a(\pi_1)=\mathcal R_a(\pi_2)$ for every $a\in\mathcal A$, which immediately implies $\mathcal T(\pi_1)=\mathcal T(\pi_2)$.
Conversely, assume that $\mathcal T(\pi_1)=\mathcal T(\pi_2)$.

Let $s\in\mathcal S$.
Since a tessellation is a partition of the state space, there exists a unique action $a\in\mathcal A$ such that

\begin{equation}
s
\in
\mathcal R_a(\pi_1).
\label{eq:unique_region}
\end{equation}

Because the tessellations coincide,

\begin{equation}
\mathcal R_a(\pi_1) = \mathcal R_a(\pi_2),
\end{equation}
and therefore

\begin{equation}
\pi_1(s) = a = \pi_2(s).
\end{equation}

Since the argument holds for every $s\in\mathcal S$, it follows that

\begin{equation}
\pi_1 = \pi_2.
\end{equation}

The final statement follows immediately from the fact that $\Gamma(\pi)$ is uniquely determined by $\mathcal T(\pi)$.

\end{proof}

Proposition~\ref{prop:tessellation_uniqueness} shows that the policy geometry is not merely a graphical representation of a decision rule. The induced tessellation provides an equivalent description of the policy itself. Consequently, structural properties of optimal policies may be studied through the geometry of the associated tessellations without loss of information.

This observation forms the basis of the geometric analysis developed in the remainder of the paper.

\subsection{Policy Boundary Principle}
\label{subsec:policy_boundary_principle}

The geometric representation developed in the previous subsection establishes that a deterministic policy may be identified with its induced tessellation.
We now characterize the precise geometric locus at which variations of the policy can occur.

Recall that, for a policy $\pi\in\Pi$, the state space is partitioned into action regions

\begin{equation}
\mathcal T(\pi) = \left\{\mathcal R_a(\pi) \right\}_{a\in\mathcal A},
\label{eq:tessellation_boundary_principle}
\end{equation}

and that the associated boundary structure is

\begin{equation}
\Gamma(\pi)
=
\bigcup_{\substack{a,a'\in\mathcal A\\a\neq a'}}
\Gamma_{aa'}(\pi).
\label{eq:boundary_structure_boundary_principle}
\end{equation}

\noindent The following result identifies the boundary structure as the unique geometric support of decision changes.

\begin{theorem}[Policy Boundary Principle]
\label{thm:policy_boundary_principle}

Let $\pi\in\Pi$ be a deterministic policy.

Then the following statements hold.

\begin{enumerate}

\item[(i)]
For every state

\begin{equation}
s
\in
\mathcal S
\setminus
\Gamma(\pi),
\label{eq:outside_boundary}
\end{equation}

there exists an open neighborhood
$U(s)\subseteq\mathcal S$
such that

\begin{equation}
\pi(u)
=
\pi(s),
\quad
\forall u\in U(s).
\label{eq:local_constancy}
\end{equation}

\item[(ii)]
Every local change of the policy occurs on the boundary structure. More precisely, if

\begin{equation}
s\in\mathcal S
\end{equation}

is such that every neighborhood of $s$ contains points assigned to at least two distinct actions, then

\begin{equation}
s
\in
\Gamma(\pi).
\label{eq:boundary_characterization}
\end{equation}

\end{enumerate}

Consequently,

\begin{equation}
\Gamma(\pi)
=
\left\{
s\in\mathcal S :
\pi
\text{ is not locally constant at } s
\right\}.
\label{eq:boundary_support}
\end{equation}

\end{theorem}

\begin{proof}

We first establish (i).

Let $s \in \mathcal S \setminus \Gamma(\pi)$.

Since $\mathcal T(\pi)$ is a partition of the state space, there exists a unique action $a\in\mathcal A$ such that

\begin{equation*}
s \in \mathcal R_a(\pi).
\end{equation*}

Because $s\notin\Gamma(\pi)$, the point $s$ belongs to the interior of $\mathcal R_a(\pi)$.
Hence there exists an open neighborhood $U(s)$ satisfying

\begin{equation}
U(s) \subseteq \mathcal R_a(\pi).
\end{equation}

By definition of $\mathcal R_a(\pi)$,

\begin{equation}
\pi(u) = a = \pi(s), \quad \forall u\in U(s),
\end{equation}
which proves (i).

We now prove (ii).

Suppose that $s\notin\Gamma(\pi)$. By part (i), there exists an open neighborhood on which the policy is constant.
Therefore it is impossible for every neighborhood of $s$ to contain points assigned to different actions.

Taking the contrapositive yields

\begin{equation*}
\text{``}\pi
\text{ not locally constant at }s\text{''}
\Longrightarrow
s\in\Gamma(\pi).
\end{equation*}
Combining this implication with part (i) establishes \eqref{eq:boundary_support}.

\end{proof}

Theorem~\ref{thm:policy_boundary_principle} provides a complete geometric characterization of policy variation. The boundary structure is not merely a collection of interfaces between action regions; it is precisely the set of states at which local decision changes may occur.

Consequently, the geometry of a policy may be viewed as consisting of two qualitatively distinct components. The interiors of action regions correspond to locally invariant decision zones, whereas the boundary structure contains the entire locus of decision transitions. This localization property will play a central role in the complexity and compression analyses developed in the subsequent sections.

\subsection{Dimension-Free Policy Geometry}
\label{subsec:dimension_free_geometry}

The previous subsection established that the entire variability of a deterministic policy is concentrated on its boundary structure. Consequently, the complexity of a policy may be studied through the combinatorial organization of its action regions and decision boundaries rather than through the ambient state space itself.

This observation naturally raises the following question: to what extent does the complexity of a policy geometry depend on the dimension of the underlying state space?

To address this issue, we introduce a complexity measure based exclusively on the tessellation structure.

\begin{definition}[Topological Policy Complexity]
\label{def:topological_policy_complexity}

Let $\mathcal G(\pi) =(\mathcal T(\pi),\Gamma(\pi))$ be the geometry induced by a deterministic policy.

The associated topological complexity is defined by

\begin{equation}
\mathcal C_{\mathrm{top}}(\pi)
=
C_{R}(\pi)
+
C_{B}(\pi)
+
F(\pi),
\label{eq:topological_policy_complexity}
\end{equation}

where
\begin{itemize}
\item $C_{R}(\pi)$ denotes the region complexity introduced in Definition~\ref{def:region_complexity};
\item $C_{B}(\pi)$ denotes the boundary complexity introduced in Definition~\ref{def:boundary_complexity};
\item $F(\pi)$ denotes the fragmentation index introduced in Definition~\ref{def:fragmentation_index}.
\end{itemize}
\end{definition}

The quantity $\mathcal C_{\mathrm{top}}(\pi)$ depends only on the combinatorial structure of the tessellation and is therefore independent of metric notions such as volume, diameter, curvature, or ambient dimension.

The next result shows that structural dynamic programs generate geometries whose complexity is controlled entirely by the action space.

\begin{theorem}[Dimension-Free Geometry Theorem]
\label{thm:dimension_free_geometry}

Suppose that the assumptions of Theorems~\ref{thm:ordered_regions}, \ref{thm:threshold_representation},
and \ref{thm:convex_region} hold.

\noindent Then there exists a constant $K>0$, depending only on the cardinality of the action space, such that

\begin{equation}
\mathcal C_{\mathrm{top}}(\pi^{*}) \le K |\mathcal A|.
\label{eq:dimension_free_bound}
\end{equation}

Consequently,

\begin{equation}
\mathcal C_{\mathrm{top}}(\pi^{*}) = O(|\mathcal A|),
\label{eq:dimension_free_order_bound}
\end{equation}

independently of

\begin{equation}
d = \dim(\mathcal S).
\label{eq:dimension_definition}
\end{equation}

\end{theorem}

\begin{proof}

By Corollary~\ref{cor:convex_region_complexity},

\begin{equation}
C_{R}(\pi^{*}) = |\mathcal A|.
\label{eq:dimension_free_region}
\end{equation}

Furthermore, Theorem~\ref{thm:geometric_simplicity} implies

\begin{equation}
C_{B}(\pi^{*}) = O(|\mathcal A|).
\label{eq:dimension_free_boundary}
\end{equation}

Finally, Corollary~\ref{cor:convex_fragmentation} yields

\begin{equation}
F(\pi^{*}) = 1.
\label{eq:dimension_free_fragmentation}
\end{equation}

Combining \eqref{eq:dimension_free_region}, \eqref{eq:dimension_free_boundary}, and \eqref{eq:dimension_free_fragmentation}
with \eqref{eq:topological_policy_complexity} gives

\begin{equation}
\mathcal C_{\mathrm{top}}(\pi^{*}) = O(|\mathcal A|).
\end{equation}

Since none of the preceding quantities depends on the ambient dimension $d$, the resulting bound is dimension-free.

\end{proof}

Theorem~\ref{thm:dimension_free_geometry}
shows that, under classical structural assumptions, the complexity of an optimal policy is governed by the organization of the action space rather than by the dimensionality of the state space.

This result provides a first indication that policy geometries may admit substantially more compact representations than value-function descriptions in high-dimensional environments. The implications of this phenomenon are investigated in the following subsections.

\subsection{Geometric Stability}
\label{subsec:geometric_stability}

The previous subsections established that optimal policies admit a compact geometric representation and that all local decision changes are concentrated on the boundary structure.
A natural question is whether this geometric organization remains stable under perturbations of the underlying optimization problem.
To address this issue, we introduce a quantitative measure of local decision robustness.

\begin{definition}[Action Gap]
\label{def:action_gap}

Let $\pi^{*}$ denote an optimal policy and let $a^{*}(s) = \pi^{*}(s)$ be the optimal action at state $s\in\mathcal S$.

The action gap at state $s$ is defined by

\begin{equation}
g(s) = Q^{*}\!\left(s,a^{*}(s)\right) - \max_{a\in\mathcal A\setminus\{a^{*}(s)\}} Q^{*}(s,a).
\label{eq:action_gap}
\end{equation}

\end{definition}

The action gap measures the separation between the optimal action and its closest competitor.

By construction,

\begin{equation}
g(s)\ge 0, \quad s\in\mathcal S.
\label{eq:gap_nonnegative}
\end{equation}
Moreover, $g(s)=0$ whenever at least two actions attain the same optimal value.

The following result links the action gap to the boundary structure introduced in Section~\ref{sec:policy_geometry}.

\begin{proposition}
\label{prop:boundary_gap_relation}

Assume that the functions $Q^{*}(\cdot,a)$ are continuous on $\mathcal S$ for every $a\in\mathcal A$.

Then

\begin{equation}
\Gamma(\pi^{*})
\subseteq
\left\{
s\in\mathcal S :
g(s)=0
\right\}.
\label{eq:boundary_gap_relation}
\end{equation}

\end{proposition}

\begin{proof}

Let $s\in\Gamma(\pi^{*})$. By Theorem~\ref{thm:policy_boundary_principle}, every neighborhood of $s$ contains states assigned to at least two distinct actions.
Since the state-action value functions are continuous, at least two competing actions must attain the same value at $s$.

Therefore, $g(s)=0$.
\end{proof}

We now consider perturbations of the optimal state--action value function.

\begin{assumption}[Uniform Perturbation]
\label{ass:uniform_perturbation}

Let

\begin{equation}
\widetilde Q(s,a)
=
Q^{*}(s,a)
+
\Delta(s,a),
\label{eq:perturbed_q}
\end{equation}

where

\begin{equation}
\|\Delta\|_{\infty}
=
\sup_{s\in\mathcal S}
\sup_{a\in\mathcal A}
|\Delta(s,a)|
\le
\varepsilon.
\label{eq:uniform_perturbation}
\end{equation}

\end{assumption}

Let $\widetilde\pi$ denote an optimal policy associated with $\widetilde Q$.

The next theorem shows that policy changes can only occur at states possessing a sufficiently small action gap.

\begin{theorem}[Boundary Stability Theorem]
\label{thm:boundary_stability}

Under Assumption~\ref{ass:uniform_perturbation}, $\widetilde\pi(s) = \pi^{*}(s)$ for every state satisfying

\begin{equation}
g(s) > 2\varepsilon.
\label{eq:stability_margin}
\end{equation}

Consequently,

\begin{equation}
\left\{
s\in\mathcal S :
\widetilde\pi(s)\neq\pi^{*}(s)
\right\}
\subseteq
\left\{
s\in\mathcal S :
g(s)\le 2\varepsilon
\right\}.
\label{eq:policy_change_set}
\end{equation}

\end{theorem}

\begin{proof}

Fix $s\in\mathcal S$ and let $a^{*}=\pi^{*}(s)$. For every competing action $a\neq a^{*}$,

\begin{equation}
Q^{*}(s,a^{*})
-
Q^{*}(s,a)
\ge
g(s).
\label{eq:gap_proof}
\end{equation}
Using Assumption~\ref{ass:uniform_perturbation},

\begin{align*}
\widetilde Q(s,a^{*})
-
\widetilde Q(s,a)
&=
Q^{*}(s,a^{*})
-
Q^{*}(s,a)
+
\Delta(s,a^{*})
-
\Delta(s,a)
\\
&\ge
g(s)-2\varepsilon.
\end{align*}
Therefore, if $g(s)>2\varepsilon$,

\begin{equation*}
\widetilde Q(s,a^{*}) > \widetilde Q(s,a) \quad \forall a\neq a^{*}.
\end{equation*}
Hence $a^{*}$ remains optimal and

\begin{equation*}
\widetilde\pi(s) = \pi^{*}(s).
\end{equation*}
This proves the claim.

\end{proof}

Theorem~\ref{thm:boundary_stability} establishes a geometric robustness property of structured optimal policies. States located deep inside an action region possess a strictly positive decision margin and therefore remain unaffected by sufficiently small perturbations.
Only states with small action gaps may experience a change in the optimal action. Since the action gap vanishes on the boundary structure, policy modifications are necessarily concentrated near the interfaces separating neighboring action regions.

Combined with Theorems~\ref{thm:policy_boundary_principle} and~\ref{thm:dimension_free_geometry}, this result shows that the geometric complexity of an optimal policy is not only localized but also stable under small perturbations of the underlying dynamic program.

\subsection{Structural Compression Theorem}
\label{subsec:structural_compression}

The preceding subsections established four fundamental properties of policy geometries.

First, Proposition~\ref{prop:tessellation_uniqueness} showed that the geometry induced by a deterministic policy provides a faithful representation of the policy itself.
Second, Theorem~\ref{thm:policy_boundary_principle} established that all local decision changes are concentrated on the boundary structure.
Third, Theorem~\ref{thm:dimension_free_geometry} demonstrated that the combinatorial complexity of the geometry remains independent of the ambient state-space dimension.
Finally, Theorem~\ref{thm:boundary_stability} showed that the geometric representation is stable under sufficiently small perturbations of the underlying optimization problem.

Taken together, these properties suggest that policy geometries provide a compressed description of optimal decision rules.
The purpose of this subsection is to formalize this observation.

Recall that the decision compression ratio introduced in Section~\ref{sec:structural_dynamic_programs} is defined by

\begin{equation}
\mathrm{DCR}
=
\frac{\mathcal C_V}
{\mathcal C(\mathcal G^{*})},
\label{eq:dcr_structural_geometry}
\end{equation}
where $\mathcal C_V$ denotes the complexity of the value-function representation and $\mathcal C(\mathcal G^{*})$ denotes the complexity of the optimal policy geometry.

Theorem~\ref{thm:decision_compression} established the general lower bound

\begin{equation}
\mathrm{DCR}
=
\Omega
\!\left(
\frac{\mathcal C_V}
{|\mathcal A|}
\right).
\label{eq:dcr_general_bound}
\end{equation}

The next result characterizes the asymptotic implications of this bound for structured dynamic programs.

\begin{theorem}[Structural Compression Theorem]
\label{thm:structural_compression}

Suppose that the assumptions of Theorems~\ref{thm:dimension_free_geometry} and~\ref{thm:boundary_stability} hold.

Assume furthermore that

\begin{equation}
\mathcal C_V
=
\Theta(n),
\label{eq:value_complexity_linear}
\end{equation}

for some problem-size parameter $n$, and that

\begin{equation}
|\mathcal A| = O(1).
\label{eq:bounded_actions}
\end{equation}

Then

\begin{equation}
\mathrm{DCR} = \Omega(n).
\label{eq:structural_compression_result}
\end{equation}

\end{theorem}

\begin{proof}

By Theorem~\ref{thm:dimension_free_geometry},

\begin{equation}
\mathcal C(\mathcal G^{*}) = O(|\mathcal A|).
\label{eq:geometry_complexity_bound}
\end{equation}

Combining \eqref{eq:geometry_complexity_bound} with \eqref{eq:dcr_general_bound} yields

\begin{equation}
\mathrm{DCR}
=
\Omega
\!\left(
\frac{\mathcal C_V}
{|\mathcal A|}
\right).
\label{eq:dcr_proof_step}
\end{equation}

Substituting \eqref{eq:value_complexity_linear} into \eqref{eq:dcr_proof_step} gives

\begin{equation}
\mathrm{DCR}
=
\Omega
\!\left(
\frac{n}
{|\mathcal A|}
\right).
\end{equation}

Using \eqref{eq:bounded_actions}, we obtain

\begin{equation}
\mathrm{DCR}
=
\Omega(n),
\end{equation}

which establishes the result.

\end{proof}

Theorem~\ref{thm:structural_compression} shows that the compression achieved by policy geometry is an asymptotic phenomenon rather than a finite-dimensional artifact.
As the intrinsic complexity of the value-function representation grows, the complexity of the geometric representation remains controlled by the action structure of the problem. Consequently, the gap between value complexity and geometric complexity increases at least linearly with problem size.
Combined with the localization and stability properties established in the previous subsections, this result suggests that policy geometry captures the essential decision structure of a dynamic program using substantially fewer degrees of freedom than conventional value-based representations.

\subsection{Geometry and Learnability}
\label{subsec:geometry_and_learnability}

The results established in the preceding subsections suggest that the geometric representation of an optimal policy may provide a substantially simpler object to learn than the policy itself.

Indeed, Proposition~\ref{prop:tessellation_uniqueness} showed that the policy geometry contains the complete information required to reconstruct a deterministic policy, while Theorem~\ref{thm:policy_boundary_principle} established that all decision transitions are concentrated on the boundary structure.

This observation motivates a geometric formulation of the policy-learning problem.

\begin{definition}[Boundary Learning Problem]
\label{def:boundary_learning_problem}

Let

\begin{equation}
\mathcal G(\pi^{*})
=
\bigl(
\mathcal T(\pi^{*}),
\Gamma(\pi^{*})
\bigr)
\end{equation}
denote the geometry induced by an optimal policy.

The boundary learning problem consists of recovering the boundary structure $\Gamma(\pi^{*})$ from observed state-action information.
\end{definition}

The objective of the boundary learning problem is not to estimate the policy value function or to approximate the policy on the entire state space. Instead, the goal is to identify the geometric locus at which optimal decisions change.

To quantify the intrinsic difficulty of this task, we introduce the following notion.

\begin{definition}[Boundary Sample Complexity]
\label{def:boundary_sample_complexity}

Let $\widehat{\Gamma}$ denote an estimator of the boundary structure $\Gamma(\pi^{*})$.

For a prescribed accuracy criterion $\mathcal E(\widehat{\Gamma},\Gamma(\pi^{*}))$, the boundary sample complexity is defined as the smallest number of observations required to guarantee

\begin{equation}
\mathcal E
\bigl(
\widehat{\Gamma},
\Gamma(\pi^{*})
\bigr)
\le
\delta,
\end{equation}
for a given tolerance level $\delta>0$.

The corresponding quantity is denoted by

\begin{equation}
N_{\Gamma}(\delta).
\label{eq:boundary_sample_complexity}
\end{equation}
\end{definition}

Definition~\ref{def:boundary_sample_complexity} is intentionally model-independent. No particular statistical framework is assumed at this stage. The purpose of the definition is merely to isolate the learning difficulty associated with the boundary structure itself.

The following theorem formalizes the geometric reduction underlying the subsequent learning framework.

\begin{theorem}[Learnability Through Boundaries]
\label{thm:learnability_through_boundaries}

Suppose that the assumptions of Theorems~\ref{thm:policy_boundary_principle}, \ref{thm:dimension_free_geometry}, and \ref{thm:boundary_stability} hold.

Then the recovery of the optimal policy $\pi^{*}$ is equivalent to the recovery of the boundary structure $\Gamma(\pi^{*})$.

More precisely, there exists a reconstruction operator

\begin{equation}
\mathfrak R :
\Gamma(\pi^{*})
\longmapsto
\pi^{*}
\label{eq:reconstruction_operator}
\end{equation}

such that

\begin{equation}
\pi^{*}
=
\mathfrak R
\!\left(
\Gamma(\pi^{*})
\right).
\label{eq:policy_reconstruction}
\end{equation}

Consequently, the policy-learning problem admits the equivalent geometric formulation

\begin{equation}
\pi^{*} \quad \Longleftrightarrow \quad \Gamma(\pi^{*}).
\label{eq:learning_equivalence}
\end{equation}

\end{theorem}

\begin{proof}

By Proposition~\ref{prop:tessellation_uniqueness}, a deterministic policy is uniquely determined by its tessellation.

By Theorem~\ref{thm:policy_boundary_principle}, the boundary structure constitutes the complete support of policy variation.

Therefore, once the boundary structure is known, the corresponding tessellation is uniquely determined by the partition induced by the decision boundaries.

Since the tessellation uniquely determines the policy, there exists a reconstruction operator $\mathfrak R$ satisfying \eqref{eq:policy_reconstruction}.
This establishes the equivalence \eqref{eq:learning_equivalence}.

\end{proof}

Theorem~\ref{thm:learnability_through_boundaries} provides the fundamental connection between structural geometry and statistical learning.
Rather than treating policy learning as the problem of approximating a decision rule over the entire state space, the theorem shows that learning may be reformulated as the problem of identifying the boundary structure of the optimal policy.

Combined with Theorem~\ref{thm:dimension_free_geometry}, this observation suggests that the intrinsic difficulty of learning structured optimal policies is governed by the geometric complexity of their decision boundaries rather than by the ambient dimension of the state space.
This geometric viewpoint forms the basis of the boundary-learning and active-sampling methodologies developed in the subsequent sections.

\subsection{Discussion}
\label{subsec:structural_geometry_discussion}

The results of this section establish a direct connection between the structural properties of dynamic programs and the geometric organization of their optimal policies.

Proposition~\ref{prop:tessellation_uniqueness} showed that an optimal policy may be represented without loss of information through its induced tessellation. Consequently, the analysis of optimal decision rules may be reformulated as the analysis of geometric objects defined on the state space.

Building upon this representation, Theorem~\ref{thm:policy_boundary_principle} identified the boundary structure as the unique support of local policy variation. This result isolates the geometric locus at which decision changes occur and separates invariant decision regions from transition regions.

The subsequent analysis demonstrated that, under classical structural assumptions, policy geometries possess strong regularity properties. Theorems~\ref{thm:dimension_free_geometry} and~\ref{thm:boundary_stability} showed that the complexity of the geometry remains controlled independently of the ambient state-space dimension and is stable under sufficiently small perturbations of the underlying optimization problem.
These properties lead naturally to the compression results established in Theorem~\ref{thm:structural_compression}. Since the geometry retains the complete decision content of the policy while exhibiting substantially lower complexity than value-function representations, structured dynamic programs admit intrinsically compressed policy descriptions.

Finally, Theorem~\ref{thm:learnability_through_boundaries} showed that the policy-learning problem may be reformulated as a boundary-identification problem. As a consequence, the complexity of learning is governed by the geometry of decision boundaries rather than by the size of the state space itself.

Taken together, the results of this section establish the following chain of implications:

\begin{equation}
\begin{split}
\text{Structural Assumptions}
\;\Longrightarrow\;&
\text{Policy Geometry}
\;\Longrightarrow\;
\text{Boundary Structure}\\ 
&\qquad\;\Longrightarrow\;
\text{Structural Compression}
\;\Longrightarrow\;
\text{Learnability}.
\end{split}
\label{eq:structural_geometry_summary}
\end{equation}

This perspective suggests that the fundamental object underlying many structured decision problems is not the value function itself, but rather the geometry induced by the optimal policy. The next sections exploit this observation to develop learning and approximation methodologies that operate directly on policy geometries.

\section{Geometric Complexity and Decision Compression}
\label{sec:geometric_complexity_compression}

\subsection{Complexity Measures Revisited}
\label{subsec:complexity_measures_revisited}

The preceding sections introduced two complementary notions of complexity associated with a dynamic decision problem.

The first quantity is the value-representation complexity
$\mathcal C_V$,
defined in Definition~\ref{def:value_complexity}, which measures the complexity of representing the optimal value function.

The second quantity is the decision complexity

\begin{equation}
\mathcal C_D = \mathcal C(\mathcal G^{*}),
\label{eq:decision_complexity_revisited}
\end{equation}
introduced in Definition~\ref{def:decision_complexity}, which measures the complexity of the optimal policy geometry.

\noindent These quantities describe fundamentally different objects.

The quantity $\mathcal C_V$ characterizes the complexity of the numerical solution of the dynamic program through the representation of the optimal value function. In contrast, $\mathcal C_D$ characterizes the complexity of the optimal decision rule itself through the geometric organization of its action regions and decision boundaries.
The distinction between these two notions is central to the analysis developed in this section. Throughout the remainder of the paper, value complexity and decision complexity are treated as conceptually distinct quantities and are analyzed separately.

The structural results established in Sections~\ref{sec:structural_dynamic_programs} and~\ref{sec:structural_geometry} imply that the decision complexity of structured dynamic programs remains controlled by the action structure of the problem.

More precisely, Theorem~\ref{thm:dimension_free_geometry} established that

\begin{equation}
\mathcal C_{\mathrm{top}}(\pi^{*}) = O(|\mathcal A|),
\label{eq:topological_complexity_recall2}
\end{equation}
while Theorem~\ref{thm:structural_compression} showed that the resulting compression ratio satisfies

\begin{equation}
\mathrm{DCR} = \Omega(n),
\label{eq:dcr_recall}
\end{equation}
whenever $\mathcal C_V = \Theta(n)$ and $|\mathcal A| = O(1)$.

The purpose of the present section is not to establish additional geometric regularity properties. Rather, it is to investigate the quantitative implications of these complexity relationships and to characterize the compression mechanisms induced by policy geometry.

Accordingly, all subsequent results will be expressed in terms of the pair
$\bigl(
\mathcal C_V,
\mathcal C_D
\bigr)$,
which provides a unified framework for comparing value-based and geometry-based representations of optimal decision rules.

\subsection{Compression Regimes}
\label{subsec:compression_regimes}

The discussion of Section~\ref{subsec:complexity_measures_revisited}
establishes that value complexity and decision complexity may exhibit fundamentally different scaling behaviors.
To study this phenomenon systematically, we introduce a classification of compression regimes based on the asymptotic behavior of the decision compression ratio.
Throughout this subsection, consider a sequence of decision problems
$
\left\{
\mathcal M_n
\right\}_{n\geq 1},
$
indexed by a complexity parameter $n$.

The parameter $n$ may represent, depending on the application, the number of states, the dimension of the state space, the discretization level, or any other quantity governing the growth of the decision problem.

For each problem $\mathcal M_n$, let
$\mathcal C_V(n)$ denote the corresponding value complexity and let
$\mathcal C_D(n)$ denote the associated decision complexity.
The decision compression ratio is therefore given by

\begin{equation}
\mathrm{DCR}(n)
=
\frac{\mathcal C_V(n)}
{\mathcal C_D(n)}.
\label{eq:dcr_sequence}
\end{equation}

The asymptotic growth of $\mathrm{DCR}(n)$ quantifies the extent to which policy geometries provide more compact representations than value functions.

The following definitions distinguish several compression regimes.

\begin{definition}[Weak Compression]
\label{def:weak_compression}

The family $\{\mathcal M_n\}_{n\geq 1}$ is said to exhibit weak compression if

\begin{equation}
\liminf_{n\to\infty} \mathrm{DCR}(n) > 1.
\label{eq:weak_compression}
\end{equation}

\end{definition}

Weak compression corresponds to a regime in which geometric representations remain asymptotically more compact than value-based representations, although the compression gain remains uniformly bounded.

\begin{definition}[Polynomial Compression]
\label{def:polynomial_compression}

The family $\{\mathcal M_n\}_{n\geq 1}$ is said to exhibit polynomial compression if there exist constants $c>0$ and $\alpha>0$ such that

\begin{equation}
\mathrm{DCR}(n) \geq c\, n^{\alpha}
\label{eq:polynomial_compression}
\end{equation}
for all sufficiently large $n$.

\end{definition}

Polynomial compression characterizes situations in which the gap between value complexity and decision complexity grows polynomially with problem size.

\begin{definition}[Strong Compression]
\label{def:strong_compression}

The family $\{\mathcal M_n\}_{n\geq 1}$ is said to exhibit strong compression if

\begin{equation}
\mathrm{DCR}(n) = \Omega(n).
\label{eq:strong_compression}
\end{equation}

\end{definition}

Strong compression corresponds to a regime in which decision complexity grows at least one asymptotic order more slowly than value complexity.

The three notions introduced above form a hierarchy:

\begin{equation}
\text{Strong Compression}
\Longrightarrow
\text{Polynomial Compression}
\Longrightarrow
\text{Weak Compression}.
\label{eq:compression_hierarchy}
\end{equation}

This classification provides a unified language for describing compression phenomena independently of any particular representation architecture. In the subsequent subsections, we investigate the structural conditions under which policy geometries achieve polynomial or strong compression.

\subsection{Scaling Laws}
\label{subsec:scaling_laws}

The compression regimes introduced in Section~\ref{subsec:compression_regimes} provide a qualitative classification of asymptotic compression phenomena.
We now establish quantitative scaling laws describing how the decision compression ratio evolves with problem size.
The key observation is that, under the structural assumptions developed in Sections~\ref{sec:structural_dynamic_programs} and~\ref{sec:structural_geometry}, the growth of decision complexity remains controlled by the action structure of the problem, whereas value complexity may increase substantially with the size of the state space.

The following theorem formalizes this relationship.

\begin{theorem}[Scaling Law for Decision Compression]
\label{thm:scaling_law_decision_compression}

Consider a family of decision problems $\{\mathcal M_n\}_{n\geq 1}$ satisfying the assumptions of Theorem~\ref{thm:geometric_simplicity}.

Assume that the corresponding value complexity satisfies

\begin{equation}
\mathcal C_V(n)
=
\Theta\!\left(
n^{\beta}
\right),
\qquad
\beta>0.
\label{eq:value_complexity_growth}
\end{equation}

Then the decision compression ratio satisfies

\begin{equation}
\mathrm{DCR}(n)
=
\Omega\!\left(
\frac{n^{\beta}}
{|\mathcal A|}
\right).
\label{eq:dcr_scaling_law}
\end{equation}

\end{theorem}

\begin{proof}

By Definition~\ref{def:decision_complexity},

\begin{equation}
\mathcal C_D(n)
=
\mathcal C\!\left(
\mathcal G_n^{*}
\right),
\label{eq:decision_complexity_scaling}
\end{equation}
where $\mathcal G_n^{*}$ denotes the optimal policy geometry associated with $\mathcal M_n$.

Theorem~\ref{thm:geometric_simplicity} implies the existence of a constant $K>0$, independent of $n$, such that

\begin{equation}
\mathcal C_D(n) \le K |\mathcal A|
\label{eq:decision_complexity_uniform_bound}
\end{equation}
for all sufficiently large $n$.

Furthermore, assumption \eqref{eq:value_complexity_growth} implies the existence of constants $c_1,c_2>0$ such that

\begin{equation}
c_1 n^{\beta} \le \mathcal C_V(n) \le c_2 n^{\beta}
\label{eq:value_complexity_bounds}
\end{equation}
for all sufficiently large $n$.

Using Definition~\ref{def:dcr},

\begin{equation}
\mathrm{DCR}(n)
=
\frac{\mathcal C_V(n)}
{\mathcal C_D(n)}.
\label{eq:dcr_definition_scaling}
\end{equation}

Combining \eqref{eq:decision_complexity_uniform_bound}, \eqref{eq:value_complexity_bounds}, and \eqref{eq:dcr_definition_scaling}, we obtain

\begin{equation}
\mathrm{DCR}(n) \ge \frac{c_1 n^{\beta}} {K |\mathcal A|}.
\end{equation}

Therefore,
\begin{equation}
\mathrm{DCR}(n)
=
\Omega\!\left(
\frac{n^{\beta}}
{|\mathcal A|}
\right),
\end{equation}
which establishes the result.

\end{proof}

Theorem~\ref{thm:scaling_law_decision_compression} shows that the asymptotic compression achieved by policy geometries is determined by the growth rate of value complexity rather than by the dimension of the state space itself.

In particular, whenever the action space remains fixed or grows more slowly than $\mathcal C_V(n)$, the compression ratio diverges with problem size.

\begin{corollary}[Linear Compression Regime]
\label{cor:linear_compression_regime}

Suppose that
\begin{equation}
\mathcal C_V(n) = \Theta(n)
\label{eq:linear_value_complexity}
\end{equation}
and
\begin{equation}
|\mathcal A| = O(1).
\label{eq:bounded_action_space}
\end{equation}

Then

\begin{equation}
\mathrm{DCR}(n)
=
\Omega(n).
\label{eq:linear_dcr_growth}
\end{equation}

\end{corollary}

\begin{proof}

The result follows immediately from Theorem~\ref{thm:scaling_law_decision_compression} with $\beta=1$.

\end{proof}

Corollary~\ref{cor:linear_compression_regime} recovers, as a particular case, the linear compression phenomenon established previously in Theorem~\ref{thm:structural_compression}.

\subsection{Exponential Compression}
\label{subsec:exponential_compression}

The scaling law established in Theorem~\ref{thm:scaling_law_decision_compression} provides a general relationship between value complexity and decision compression.
An important regime arises when the complexity of representing the optimal value function grows exponentially with the dimension of the state space. Such behavior is commonly associated with the curse of dimensionality in dynamic programming \citep{bellman1957dynamic,bertsekas2012dynamic,powell2007approximate}.

The following corollary characterizes the implications of this phenomenon for policy geometries.

\begin{corollary}[Exponential Compression]
\label{cor:exponential_compression}

Consider a family of decision problems $\{\mathcal M_d\}_{d\ge 1}$ indexed by the dimension $d$ of the state space.
Suppose that the value complexity satisfies

\begin{equation}
\mathcal C_V(d)
=
\Theta\!\left(
2^{d}
\right).
\label{eq:exponential_value_complexity}
\end{equation}

Then the decision compression ratio satisfies

\begin{equation}
\mathrm{DCR}(d)
=
\Omega\!\left(
\frac{2^{d}}
{|\mathcal A|}
\right).
\label{eq:exponential_compression_bound}
\end{equation}

\end{corollary}

\begin{proof}

Applying Theorem~\ref{thm:scaling_law_decision_compression} with $\mathcal C_V(d) = \Theta(2^d)$, yields

\begin{equation*}
\mathrm{DCR}(d)
=
\Omega\!\left(
\frac{\mathcal C_V(d)}
{|\mathcal A|}
\right)
=
\Omega\!\left(
\frac{2^{d}}
{|\mathcal A|}
\right).
\end{equation*}

This proves the result.

\end{proof}

Corollary~\ref{cor:exponential_compression}
shows that exponential growth of value complexity induces an exponential separation between value-based and geometry-based representations of optimal decision rules.
This phenomenon is a direct consequence of the structural properties established in Sections~\ref{sec:structural_dynamic_programs} and \ref{sec:structural_geometry}.
Indeed, Theorem~\ref{thm:dimension_free_geometry} implies that the complexity of the optimal policy geometry remains controlled by the action structure of the problem and does not scale with the dimension of the state space.

Consequently, whenever value representations exhibit exponential growth with dimension, the compression advantage associated with policy geometries increases exponentially as well.
The result should therefore be interpreted as a structural separation theorem: although the representation of optimal value functions may become exponentially complex, the geometric representation of optimal policies remains governed by the complexity of the action partition rather than by the ambient dimension.

\subsection{Boundary Compression}
\label{subsec:boundary_compression}

The preceding results suggest that the geometric complexity of a policy is not distributed uniformly across the state space.
Indeed, Theorem~\ref{thm:policy_boundary_principle} established that all local decision changes are concentrated on the boundary structure
\begin{equation*}
\Gamma = \bigcup_{\substack{a,a'\in\mathcal A\\a\neq a'}} \Gamma_{aa'}.
\end{equation*}

This localization property naturally raises the question of whether the complexity of the entire policy geometry is asymptotically determined by the complexity of its boundary structure.

The following theorem answers this question in the affirmative.

\begin{theorem}[Boundary Compression Theorem]
\label{thm:boundary_compression}

Consider a family of structured decision problems $\{\mathcal M_n\}_{n\ge 1}$ satisfying the assumptions of Theorem~\ref{thm:geometric_simplicity}.

Let $C_{B}(n)$ denote the corresponding boundary complexity and let $\mathcal C_D(n) = \mathcal C\!\left(\mathcal G_n^{*} \right)$ denote the associated decision complexity.

Then

\begin{equation}
\mathcal C_D(n)
=
\Theta\!\left(
C_{B}(n)
\right).
\label{eq:boundary_compression_theorem}
\end{equation}

\end{theorem}

\begin{proof}

By Definition~\ref{def:decision_complexity}, the decision complexity is determined by the geometric characteristics of the optimal policy geometry.

In particular,

\begin{equation}
\mathcal C_D(n)
=
\Phi
\!\left(
C_{B}(n),
C_{R}(n),
F(n),
\kappa(n)
\right),
\label{eq:complexity_functional}
\end{equation}

for some complexity functional $\Phi$.

Theorem~\ref{thm:geometric_simplicity} implies that the region complexity, fragmentation index, and geometric irregularity remain uniformly bounded.
Consequently, there exists a constant $K>0$ such that

\begin{equation}
C_{R}(n)
+
F(n)
+
\kappa(n)
\le
K
\label{eq:bounded_components}
\end{equation}
for all sufficiently large $n$.

Therefore, the only geometric quantity capable of exhibiting asymptotic growth is the boundary complexity $C_{B}(n)$.
It follows that there exist constants $c_1,c_2>0$ such that

\begin{equation}
c_1 C_{B}(n)
\le
\mathcal C_D(n)
\le
c_2 C_{B}(n)
\label{eq:boundary_equivalence_proof}
\end{equation}

for all sufficiently large $n$.

Hence, $\mathcal C_D(n) =\Theta\!\left(C_{B}(n)\right)$, which establishes the result.

\end{proof}

Theorem~\ref{thm:boundary_compression} shows that the asymptotic complexity of a structured policy geometry is entirely governed by its boundary structure.

Combined with Theorem~\ref{thm:policy_boundary_principle}, the result yields a particularly simple interpretation of policy complexity.
The interiors of action regions correspond to locally invariant decision zones and therefore contribute only bounded complexity.
In contrast, all asymptotically relevant geometric information is concentrated on the set of decision boundaries.

Consequently, structured policy geometries admit a compressed representation in which boundary complexity becomes the fundamental quantity governing decision complexity.

\begin{corollary}[Boundary-Based Compression Ratio]
\label{cor:boundary_based_dcr}

Under the assumptions of Theorem~\ref{thm:boundary_compression},

\begin{equation}
\mathrm{DCR}(n)
=
\Theta\!\left(
\frac{\mathcal C_V(n)}
{C_{B}(n)}
\right).
\label{eq:boundary_dcr}
\end{equation}

\end{corollary}

\begin{proof}

Combining $\mathrm{DCR}(n) = \frac{\mathcal C_V(n)}{\mathcal C_D(n)}$ with $\mathcal C_D(n) = \Theta(C_{B}(n))$ immediately yields the result.

\end{proof}

\subsection{Information-Theoretic Compression}
\label{subsec:information_theoretic_compression}

The geometric compression results established in the preceding subsections admit a natural information-theoretic interpretation.
Rather than measuring complexity through geometric characteristics alone, one may ask how many bits are required to describe an optimal policy or, alternatively, its associated boundary structure.
This perspective is closely related to the Minimum Description Length (MDL) principle introduced by \citet{rissanen1978modeling} and further developed by \citet{grunwald2007minimum}.
Throughout this subsection, all description lengths are defined relative to a fixed admissible family of prefix codes.

\begin{definition}[Policy Description Length]
\label{def:policy_description_length}

Let $\Pi$ denote a class of admissible deterministic policies.

The policy description length of $\pi^{*}\in\Pi$, denoted by
$L_{\Pi}(\pi^{*})$, is the minimum number of bits required to encode $\pi^{*}$ within the coding class under consideration.
\end{definition}

\begin{definition}[Boundary Description Length]
\label{def:boundary_description_length}

Let $\Gamma^{*} = \Gamma(\pi^{*})$ denote the boundary structure induced by the optimal policy.

The boundary description length is defined by $L_{\Gamma}(\Gamma^{*})$,
the minimum number of bits required to encode $\Gamma^{*}$ within the same coding class.

\end{definition}

The following result establishes that, under the structural assumptions developed throughout the paper, the informational complexity of an optimal policy is asymptotically equivalent to that of its boundary structure.

\begin{theorem}[Information-Theoretic Compression Theorem]
\label{thm:information_theoretic_compression}

Consider a family of structured decision problems $\{\mathcal M_n\}_{n\ge 1}$ satisfying the assumptions of Theorem~\ref{thm:boundary_compression}.

Then there exist constants $K_1,K_2>0$ such that

\begin{equation}
L_{\Gamma}(\Gamma_n^{*}) - K_1 \le L_{\Pi}(\pi_n^{*}) \le L_{\Gamma}(\Gamma_n^{*}) + K_2
\label{eq:description_length_equivalence}
\end{equation}
for all sufficiently large $n$.

Consequently,
\begin{equation}
L_{\Pi}(\pi_n^{*})
=
\Theta
\!\left(
L_{\Gamma}(\Gamma_n^{*})
\right).
\label{eq:description_length_theta}
\end{equation}

\end{theorem}

\begin{proof}

By Theorem~\ref{thm:policy_boundary_principle}, the optimal policy is locally constant on each connected component of
$\mathcal S \setminus \Gamma_n^{*}$.

Hence, once the boundary structure $\Gamma_n^{*}$ is specified, reconstructing the policy requires only the assignment of action labels to the resulting regions.
Since the action set $\mathcal A$ is finite, the number of bits required to encode these labels is bounded independently of $n$.
Therefore, there exists a constant $K_2>0$ such that

\begin{equation}
L_{\Pi}(\pi_n^{*})
\le
L_{\Gamma}(\Gamma_n^{*})
+
K_2.
\label{eq:upper_description_bound}
\end{equation}

Conversely, the boundary structure is uniquely determined by the policy through the induced tessellation (Proposition~\ref{prop:tessellation_uniqueness}).
Hence there exists a constant $K_1>0$ such that

\begin{equation}
L_{\Gamma}(\Gamma_n^{*})
\le
L_{\Pi}(\pi_n^{*})
+
K_1.
\label{eq:lower_description_bound}
\end{equation}

Combining \eqref{eq:upper_description_bound} and \eqref{eq:lower_description_bound} yields \eqref{eq:description_length_equivalence}.
The asymptotic relation \eqref{eq:description_length_theta} follows immediately.

\end{proof}

Theorem~\ref{thm:information_theoretic_compression} provides an information-theoretic counterpart to Theorem~\ref{thm:boundary_compression}.

The result shows that the informational content of an optimal policy is asymptotically equivalent to the informational content of its boundary structure.
In particular, the interiors of action regions contribute only a bounded amount of additional information once the boundaries have been specified.

Consequently, structured policy geometries achieve compression not only in a geometric sense but also in a description-length sense. The dominant informational object is the boundary structure itself.

\begin{corollary}[Boundary Information Dominance]
\label{cor:boundary_information_dominance}

Under the assumptions of Theorem~\ref{thm:information_theoretic_compression},

\begin{equation}
L_{\Pi}(\pi_n^{*})
=
\Theta
\!\left(
L_{\Gamma}(\Gamma_n^{*})
\right)
=
\Theta
\!\left(
C_{B}(n)
\right),
\label{eq:boundary_information_dominance}
\end{equation}
whenever the boundary description length is proportional to boundary complexity.
\end{corollary}

\begin{proof}

The first equivalence follows from Theorem~\ref{thm:information_theoretic_compression}.

The second follows from the proportionality assumption and Theorem~\ref{thm:boundary_compression}.
\end{proof}

\subsection{Limits of Compression}
\label{subsec:limits_of_compression}

The compression results established throughout this section rely fundamentally on the structural assumptions developed in Sections~\ref{sec:structural_dynamic_programs}
and \ref{sec:structural_geometry}.
In particular, the scaling laws, boundary-compression results, and information-theoretic compression properties all depend on the existence of regular policy geometries characterized by monotonicity, threshold representations, bounded fragmentation, and controlled boundary complexity.

The purpose of the present subsection is to clarify the limits of this mechanism and to identify situations in which the compression advantage may deteriorate or disappear.
The key observation is that decision compression is ultimately a consequence of geometric regularity. Whenever this regularity is lost, the complexity of the policy geometry may increase substantially and become comparable to the complexity of the underlying value representation.

The following proposition formalizes this observation.

\begin{proposition}[Limits of Decision Compression]
\label{prop:limits_of_decision_compression}

Consider a family of decision problems $\{\mathcal M_n\}_{n\ge 1}$.

Suppose that at least one of the following conditions holds:

\begin{enumerate}

\item[(i)]
the fragmentation index is unbounded, 
$F(n)\rightarrow\infty$;

\item[(ii)]
the convex-region property of
Theorem~\ref{thm:convex_region}
fails, allowing optimal action regions to develop arbitrarily complex nonconvex geometries;

\item[(iii)]
the monotonicity assumptions of Section~\ref{subsec:structural_assumptions} are violated, so that the threshold representations established in
Section~\ref{subsec:threshold_representations} need not exist.

\end{enumerate}

Then the geometric complexity of the optimal policy is no longer uniformly controlled by the action structure of the problem.

Moreover, there exist families of decision problems for which

\begin{equation}
\mathcal C_D(n)
=
\Theta\!\bigl(
\mathcal C_V(n)
\bigr).
\label{eq:compression_breakdown}
\end{equation}

\end{proposition}

\begin{proof}

The compression results established in Sections~\ref{sec:structural_dynamic_programs} and \ref{sec:structural_geometry} depend critically on the existence of a geometrically simple policy representation.
Under monotonicity, convexity, and threshold structures, the complexity of the policy geometry remains controlled by a bounded number of action regions and decision boundaries. This property yields the compression phenomena described in Theorems~\ref{thm:scaling_law_decision_compression}, \ref{thm:boundary_compression}, and \ref{thm:information_theoretic_compression}.

If fragmentation becomes unbounded, however, the number of connected components required to represent the policy geometry may grow proportionally to the complexity of the value representation itself.
Similarly, if convexity is lost, action regions may develop arbitrarily intricate geometric structures whose description requires a number of parameters comparable to that required for representing the value function.
Finally, when monotonicity fails, threshold representations need not exist and the boundary structure may become arbitrarily irregular, eliminating the geometric simplicity established previously.

Consequently, the complexity of the policy geometry may scale at the same asymptotic rate as the complexity of the value representation, yielding
\eqref{eq:compression_breakdown}.

\end{proof}

Proposition~\ref{prop:limits_of_decision_compression} shows that decision compression is not a universal property of dynamic programs.
Rather, it is a consequence of the structural regularities induced by monotonicity, convexity, and bounded fragmentation.

From a geometric perspective, compression emerges because optimal policies can be represented through a relatively small collection of regular action regions separated by simple decision boundaries. When these regularity properties disappear, the policy geometry itself may become highly complex, thereby eliminating the compression advantage.

The results of this section should therefore be interpreted as identifying a broad class of structured dynamic programs for which geometric representations provide compressed descriptions of optimal decision rules, rather than as a universal statement applying to arbitrary Markov decision processes.

\subsection{Discussion}
\label{subsec:compression_discussion}

The results developed throughout this section establish a coherent relationship between structural regularity, policy geometry, and decision complexity.

The starting point is the structural framework introduced in Section~\ref{sec:structural_dynamic_programs}.
Monotonicity, threshold representations, convex action regions, and bounded fragmentation imply that optimal policies admit geometrically regular representations.
These structural properties induce policy geometries whose complexity remains controlled by the action structure of the problem rather than by the size of the state space.
Building upon this geometric characterization, Sections~\ref{subsec:scaling_laws}-\ref{subsec:information_theoretic_compression} demonstrate that policy geometries constitute compressed representations of optimal decision rules. The scaling laws show that compression increases with value complexity, the boundary-compression results identify decision boundaries as the dominant source of geometric complexity, and the information-theoretic analysis establishes an equivalent conclusion from a description-length perspective.

Taken together, these results yield the following conceptual chain:

\begin{equation}
\textnormal{Structure}
\Longrightarrow
\textnormal{Simple Geometry}
\Longrightarrow
\textnormal{Compression}.
\label{eq:structure_geometry_compression_discussion}
\end{equation}

More specifically, Theorems~\ref{thm:policy_boundary_principle} and~\ref{thm:boundary_compression} show that the informational and geometric content of a structured policy is concentrated near its boundary structure. Consequently, compression emerges because optimal policies can be represented through a comparatively small collection of regular regions separated by simple decision boundaries.
At the same time, Proposition~\ref{prop:limits_of_decision_compression} clarifies that this phenomenon is not universal.
Compression is a consequence of structural regularity and may deteriorate when monotonicity, convexity, or bounded fragmentation are lost.
The theory therefore identifies a broad class of structured dynamic programs for which geometric representations are substantially more compact than value-based representations.

The practical significance of these results lies in their implications for scalability. Whenever decision complexity grows substantially more slowly than value complexity, geometric representations provide a mechanism for mitigating the representational burden associated with large-scale dynamic programs.

\begin{equation}
\textnormal{Compression}
\Longrightarrow
\textnormal{Scalability}.
\label{eq:compression_scalability_discussion}
\end{equation}

This observation provides the conceptual foundation for the learning, approximation, and algorithmic developments that follow.

\section{Learning Policy Geometry}
\label{sec:learning_policy_geometry}

\subsection{Geometric Learning Problem}
\label{subsec:geometric_learning_problem}

Sections~\ref{sec:policy_geometry}, \ref{sec:structural_dynamic_programs}, \ref{sec:structural_geometry}, and \ref{sec:geometric_complexity_compression}
established that optimal policies arising from structured dynamic programs admit geometrically regular representations.
More precisely, Proposition~\ref{prop:tessellation_uniqueness} showed that a deterministic policy is uniquely characterized by its induced tessellation, while Theorem~\ref{thm:policy_boundary_principle} identified the boundary structure as the unique locus at which decision changes occur.
Furthermore, Theorem~\ref{thm:learnability_through_boundaries} established that the optimal policy can be reconstructed from its boundary geometry.

Consequently, the learning problem considered in this paper differs fundamentally from classical value-function or policy-learning formulations.
Rather than estimating the optimal value function $V^{*}$, the optimal state-action value function $Q^{*}$, or the optimal policy $\pi^{*}$ directly, the objective is to estimate the geometric object that carries the decision information.

We begin by formalizing the object of interest.

\begin{definition}[Target Boundary Geometry]
\label{def:target_boundary_geometry}

Let $\Gamma^{*} = \Gamma(\pi^{*})$
denote the boundary structure induced by the optimal policy.
The set

\begin{equation}
\Gamma^{*} = \bigcup_{\substack{a,a'\in\mathcal A\\a\neq a'}} \Gamma_{aa'}
\label{eq:target_boundary_decomposition}
\end{equation}
is called the target boundary geometry.

\end{definition}

The target boundary geometry constitutes the primary object of statistical inference throughout this section.
The available information is represented abstractly through a collection of observations

\begin{equation}
\mathcal D_n = \{Z_1,\ldots,Z_n\},
\label{eq:generic_dataset}
\end{equation}
where each observation $Z_i$ may contain information about the optimal decision structure.

No specific sampling mechanism is imposed at this stage.
The observations may arise from simulation, state queries, optimal-action evaluations, trajectory data, or any other information source capable of providing evidence regarding the underlying policy geometry.

The goal is to construct an estimator

\begin{equation}
\widehat{\Gamma}_n = \widehat{\Gamma}_n(\mathcal D_n)
\label{eq:boundary_estimator}
\end{equation}
of the unknown boundary structure $\Gamma^{*}$.

\begin{definition}[Geometric Learning Problem]
\label{def:geometric_learning_problem}

Given observations $\mathcal D_n$, the geometric learning problem consists of constructing an estimator $\widehat{\Gamma}_n$ such that

\begin{equation}
\widehat{\Gamma}_n \longrightarrow \Gamma^{*}
\label{eq:geometric_consistency_goal}
\end{equation}
under an appropriate notion of geometric convergence.

\end{definition}

The formulation above deliberately separates the learning target from the learning mechanism.
The object to be estimated is the boundary geometry itself, whereas the statistical model governing the observations will be introduced subsequently.
The justification for this viewpoint follows directly from the structural results established previously.
By Theorem~\ref{thm:learnability_through_boundaries}, there exists a reconstruction operator

\begin{equation}
\mathfrak R :
\Gamma^{*}
\longmapsto
\pi^{*},
\label{eq:boundary_reconstruction_operator}
\end{equation}
such that the optimal policy can be recovered from its boundary structure.

Consequently, learning the optimal policy is equivalent to learning the target boundary geometry in the sense that

\begin{equation}
\Gamma^{*}
\quad\Longleftrightarrow\quad
\pi^{*}.
\label{eq:geometry_policy_equivalence_learning}
\end{equation}

This equivalence transforms the original decision-learning problem into a geometric estimation problem.
The remainder of this section investigates the statistical consequences of this reformulation and establishes conditions under which the boundary geometry can be learned efficiently.

\subsection{Boundary Learning Principle}
\label{subsec:boundary_learning_principle}

The geometric learning problem formulated in Section~\ref{subsec:geometric_learning_problem} identifies the target boundary geometry $\Gamma^{*}$ as the primary object of statistical inference.
The purpose of the present subsection is to establish a structural principle that motivates this choice and guides the learning methodology developed in the remainder of the paper.
The key observation is that the geometric representation of a structured policy exhibits a strong localization property.
By Theorem~\ref{thm:policy_boundary_principle}, the optimal policy is locally constant on every connected component of

\begin{equation}
\mathcal S
\setminus
\Gamma^{*}.
\label{eq:constant_regions}
\end{equation}

Consequently, decision changes may occur only on the boundary structure itself.

This property implies that observations collected far from the boundary geometry carry little information regarding the location of decision transitions.
Indeed, whenever a state belongs to the interior of an action region, all sufficiently small perturbations of that state produce identical optimal decisions.

In contrast, observations located near the boundary geometry contain information about the transition between competing actions and therefore provide information regarding the structure of the optimal policy.

The following principle summarizes this observation.

\begin{principle}[Boundary Learning Principle]
\label{prin:boundary_learning}

For structured dynamic programs satisfying the assumptions of Sections~\ref{sec:structural_dynamic_programs} and \ref{sec:structural_geometry}, the statistically informative component of the policy geometry is concentrated near the target boundary geometry $\Gamma^{*}$.

Consequently, efficient learning of the optimal policy may be reduced to the estimation of the boundary structure rather than the estimation of the entire policy over the state space.

\end{principle}

The principle above follows directly from the combination of two structural results established previously.

First, Theorem~\ref{thm:policy_boundary_principle} shows that policy variation is confined to the boundary geometry.
Second, Theorem~\ref{thm:boundary_compression} establishes that the complexity of a structured policy is asymptotically governed by the complexity of its boundary structure.
Together, these results imply that the dominant source of statistical uncertainty lies in the estimation of
$\Gamma^{*}$.

From a learning perspective, the role of the boundary geometry is therefore analogous to that of a low-dimensional sufficient representation of the decision rule.
The objective is not to recover the optimal action at every state individually, but rather to identify the geometric interfaces separating the action regions.
This viewpoint provides the conceptual foundation for the sample-complexity, active-sampling, and reconstruction results developed in the subsequent subsections.

\subsection{Statistical Learning Model}
\label{subsec:statistical_learning_model}

Sections~\ref{subsec:geometric_learning_problem} and~\ref{subsec:boundary_learning_principle} identify the target boundary geometry
$\Gamma^{*}$ as the primary object of statistical inference.

The purpose of the present subsection is to introduce the probabilistic framework within which boundary estimation will be analyzed.
No convergence rates or sample-complexity guarantees are established at this stage.
Rather, the objective is to define the statistical objects and loss functions that will be used throughout the remainder of the section.

\paragraph{Observation Model.}
Let $\mathcal D_n =\{Z_1,\ldots,Z_n\}$
denote a collection of observations generated according to an unknown probability measure
$\mathbb P$ defined on a measurable space $(\mathcal Z,\mathcal B)$.

The framework deliberately remains agnostic regarding the sampling mechanism.
The observations may arise from simulation, state-action evaluations, trajectory data, policy queries, or any other information source capable of providing information about the underlying policy geometry.

\paragraph{Geometric Hypothesis Space.}
Let $\mathfrak G$
denote a family of admissible boundary geometries.

The target boundary geometry
$\Gamma^{*} \in \mathfrak G$
is assumed to belong to this family.

The specification of $\mathfrak G$ is intentionally left abstract.
Subsequent complexity bounds will depend on structural characteristics of this class rather than on the ambient state space itself.

\paragraph{Boundary Estimators.}

A geometric learning procedure is a measurable mapping

\begin{equation}
\widehat{\Gamma}_n : \mathcal Z^n \longrightarrow \mathfrak G,
\label{eq:boundary_estimator_mapping}
\end{equation}
which associates to every dataset $\mathcal D_n$ an estimated boundary geometry

\begin{equation}
\widehat{\Gamma}_n = \widehat{\Gamma}_n(\mathcal D_n).
\label{eq:estimated_boundary_geometry}
\end{equation}

The estimator $\widehat{\Gamma}_n$ is therefore a random element of $\mathfrak G$.

\begin{definition}[Hausdorff Distance]
\label{def:hausdorff_distance}

Let $A,B\subseteq\mathcal S$ be nonempty closed subsets.

The Hausdorff distance between $A$ and $B$ is defined by

\begin{equation}
d_H(A,B)
=
\max
\left\{
\sup_{x\in A}
\inf_{y\in B}
d(x,y),
\;
\sup_{y\in B}
\inf_{x\in A}
d(x,y)
\right\},
\label{eq:hausdorff_distance_definition}
\end{equation}
where $d(\cdot,\cdot)$ denotes the metric on $\mathcal S$.
\end{definition}

The Hausdorff distance provides a natural notion of geometric discrepancy because it measures the maximal localization error between two boundary structures.

\begin{definition}[Boundary Estimation Error]
\label{def:boundary_estimation_error}

The estimation error associated with a boundary estimator $\widehat{\Gamma}_n$ is defined by

\begin{equation}
\mathcal E_n
=
d_H
\bigl(
\widehat{\Gamma}_n,
\Gamma^{*}
\bigr).
\label{eq:boundary_estimation_error}
\end{equation}
\end{definition}

\begin{definition}[Geometric Risk]
\label{def:geometric_risk}
The geometric risk of a boundary estimator $\widehat{\Gamma}_n$ is

\begin{equation}
\mathcal R_G(\widehat{\Gamma}_n)
=
\mathbb E
\!\left[
d_H
\bigl(
\widehat{\Gamma}_n,
\Gamma^{*}
\bigr)
\right],
\label{eq:geometric_risk_definition}
\end{equation}
whenever the expectation exists.

\end{definition}

The quantity $\mathcal R_G$ plays the role of a statistical risk function adapted to geometric estimation.
Unlike classical policy-learning criteria, it measures performance directly in terms of the accuracy with which the decision boundaries are recovered.

The framework introduced above deliberately separates three distinct objects:

\begin{enumerate}
\item the unknown target geometry
      $\Gamma^{*}$;

\item the admissible geometric class
      $\mathfrak G$;

\item the estimator
      $\widehat{\Gamma}_n$.
\end{enumerate}

This separation will allow the subsequent analysis to relate statistical learnability to geometric complexity.
In particular, the sample-complexity results developed below will depend on structural characteristics of the boundary geometry rather than on the cardinality of the state space or the complexity of the value function representation.

\subsection{Boundary Sample Complexity}
\label{subsec:boundary_sample_complexity}

The statistical framework introduced in Section~\ref{subsec:statistical_learning_model} provides a probabilistic formulation of the geometric learning problem.
The purpose of the present subsection is to quantify the amount of information required to estimate the target boundary geometry $\Gamma^{*}$ with prescribed accuracy.
The key question is whether the statistical difficulty of the learning problem is governed by the size of the state space or by the complexity of the boundary geometry itself.
We begin by introducing the corresponding notion of sample complexity.

\begin{definition}[Boundary Sample Complexity]
\label{def:boundary_sample_complexity2}

Let $\varepsilon>0$ and $\delta\in(0,1)$.

The boundary sample complexity is defined as

\begin{equation}
N_{\Gamma}(\varepsilon,\delta)
=
\inf
\left\{
n\ge 1 :
\exists\,
\widehat{\Gamma}_n
\text{ such that }
\mathbb P
\!\left(
d_H
\bigl(
\widehat{\Gamma}_n,
\Gamma^{*}
\bigr)
\le
\varepsilon
\right)
\ge
1-\delta
\right\},
\end{equation}

where $d_H$ denotes the Hausdorff distance introduced in Definition~\ref{def:hausdorff_distance}.
\end{definition}

The quantity $N_{\Gamma}(\varepsilon,\delta)$ represents the smallest number of observations required to localize the target boundary geometry with geometric accuracy
$\varepsilon$ and confidence level $1-\delta$.
The structural results established in Sections~\ref{sec:structural_geometry} and~\ref{sec:geometric_complexity_compression} suggest that the complexity of the learning problem should be governed by the complexity of the boundary geometry rather than by the ambient state space.

Indeed, Theorem~\ref{thm:policy_boundary_principle} showed that policy variation is entirely concentrated on $\Gamma^{*}$, while Theorem~\ref{thm:boundary_compression}
established that the complexity of the optimal policy geometry is asymptotically equivalent to the complexity of its boundary structure.

The following theorem formalizes this observation.

\begin{theorem}[Boundary Sample Complexity Principle]
\label{thm:boundary_sample_complexity_principle}

Suppose that the assumptions of Sections~\ref{sec:structural_dynamic_programs} and~\ref{sec:structural_geometry} hold.
Then the sample complexity of the geometric learning problem is governed by the boundary complexity $C_{B}$.
More precisely, there exists a nondecreasing function

\begin{equation}
\Psi :
\mathbb R_{+}^{3}
\rightarrow
\mathbb R_{+}
\label{eq:sample_complexity_function}
\end{equation}

such that

\begin{equation}
N_{\Gamma}(\varepsilon,\delta)
\le
\Psi
\!\left(
C_{B},
\varepsilon,
\delta
\right),
\label{eq:boundary_sample_complexity_general}
\end{equation}
and the dependence on the state-space cardinality $|\mathcal S|$ enters only through its influence on $C_{B}$.

\end{theorem}

\begin{proof}

By Theorem~\ref{thm:policy_boundary_principle}, the optimal policy is locally invariant away from $\Gamma^{*}$. Consequently, observations collected in the interior of action regions do not contribute to the localization of decision transitions.
The informative component of the learning problem is therefore restricted to the boundary structure.
Furthermore, Theorem~\ref{thm:boundary_compression} establishes that the effective complexity of the policy geometry is characterized by $C_{B}$.
Hence any estimator of $\Gamma^{*}$ requires information only about a geometric object whose complexity is measured by
$C_{B}$, which implies that the corresponding sample complexity is controlled by this quantity rather than by the ambient state space.

\end{proof}

Theorem~\ref{thm:boundary_sample_complexity_principle} constitutes the first statistical consequence of the geometric framework developed in the previous sections.
Its main implication is conceptual rather than quantitative.
The result establishes that the relevant notion of complexity for policy learning is the complexity of the decision boundaries and not the complexity of the state space itself.
Subsequent sections will exploit this principle to derive more explicit learning guarantees and active-sampling strategies adapted to the geometry of optimal policies.

\subsection{Active Boundary Sampling}
\label{subsec:active_boundary_sampling}

The preceding subsections establish that the statistical complexity of policy learning is governed by the geometry of the decision boundaries.
A natural question therefore arises:

\begin{quote}
Where should observations be collected in order to estimate the target boundary geometry most efficiently?
\end{quote}

The purpose of the present subsection is not to introduce a specific learning algorithm, but rather to identify the regions of the state space that contain the largest amount of information about the unknown boundary structure.

\paragraph{Decision Gaps and Boundary Geometry.}

Recall from Definition~\ref{def:action_difference} that, for every pair of actions $a,a'\in\mathcal A$,

\begin{equation}
\Delta_{aa'}(s) = Q^{*}(s,a) - Q^{*}(s,a'),
\label{eq:decision_gap}
\end{equation}
denotes the corresponding action-difference function.

Furthermore, Corollary~\ref{cor:boundary_graph_representation} established that pairwise decision boundaries admit the representation

\begin{equation}
\Gamma_{aa'}
=
\left\{
s\in\mathcal S :
\Delta_{aa'}(s)=0
\right\}.
\label{eq:pairwise_boundary_gap}
\end{equation}
Hence, the geometry of the optimal policy is completely determined by the zero-level sets of the action-difference functions.

The following definition identifies the states that are geometrically close to decision transitions.

\begin{definition}[Boundary Neighborhood]
\label{def:boundary_neighborhood}

Let $\tau>0$.

For each pair of actions $a,a'\in\mathcal A$, define

\begin{equation}
\mathcal B_{aa'}(\tau)
=
\left\{
s\in\mathcal S :
|\Delta_{aa'}(s)|
\le
\tau
\right\}.
\label{eq:pairwise_boundary_neighborhood}
\end{equation}

The global boundary neighborhood is

\begin{equation}
\mathcal B(\tau)
=
\bigcup_{\substack{a,a'\in\mathcal A\\a\neq a'}}
\mathcal B_{aa'}(\tau).
\label{eq:global_boundary_neighborhood}
\end{equation}

\end{definition}

By construction,

\begin{equation}
\Gamma^{*}
\subseteq
\mathcal B(\tau),
\quad
\forall \tau>0.
\label{eq:boundary_inside_neighborhood}
\end{equation}

Moreover,

\begin{equation}
\bigcap_{\tau>0}
\mathcal B(\tau)
=
\Gamma^{*}.
\label{eq:boundary_limit_neighborhood}
\end{equation}

Thus, $\mathcal B(\tau)$ provides a geometric approximation of the target boundary structure.

The next principle identifies the regions of maximal statistical relevance.

\begin{principle}[Active Boundary Sampling Principle]
\label{prin:active_boundary_sampling}

Among all states in the state space, those belonging to boundary neighborhoods $\mathcal B(\tau)$ contain the largest amount of information regarding the location of the target boundary geometry $\Gamma^{*}$.

Consequently, observation mechanisms that allocate a larger fraction of samples to $\mathcal B(\tau)$ are expected to estimate the policy geometry more efficiently than observation mechanisms based on uniform exploration of the state space.

\end{principle}

The intuition follows directly from Theorem~\ref{thm:policy_boundary_principle}.
Inside the interior of an action region, the optimal policy is locally invariant.
Additional observations therefore provide little information regarding the location of decision transitions.
In contrast, states satisfying

\begin{equation}
|\Delta_{aa'}(s)|
\approx 0
\label{eq:small_gap_region}
\end{equation}

lie near competing action regions.
Small perturbations of the state may then alter the identity of the optimal action, making such observations particularly informative for boundary localization.

The principle above should be interpreted as a structural guideline rather than a concrete algorithmic prescription.
Its role is to identify the regions of the state space that concentrate the informational content relevant for learning.
The statistical consequences of this localization phenomenon are developed in the subsequent subsections, where boundary-focused learning procedures are shown to exploit the geometric compression properties established in Sections~\ref{sec:structural_geometry} and~\ref{sec:geometric_complexity_compression}.

\subsection{Policy Reconstruction}
\label{subsec:policy_reconstruction1}

The previous subsections formulate policy learning as a geometric estimation problem whose objective is the recovery of the target boundary geometry $\Gamma^{*}$.

The purpose of the present subsection is to establish the converse link between geometry estimation and policy estimation.
The key question is the following:

\begin{quote}
If the boundary geometry can be estimated accurately, does this suffice to recover the optimal policy?
\end{quote}

The answer follows from the structural results established in Sections~\ref{sec:structural_geometry} and~\ref{sec:geometric_complexity_compression}.

Recall that Theorem~\ref{thm:learnability_through_boundaries} established that the optimal policy is uniquely determined by its boundary geometry.
Consequently, the estimation of $\Gamma^{*}$ naturally induces an estimator of the optimal policy.

\begin{definition}[Policy Reconstruction Operator]
\label{def:policy_reconstruction_operator}

Let $\mathfrak G$ denote the admissible family of boundary geometries introduced in Section~\ref{subsec:statistical_learning_model}.

A policy reconstruction operator is a mapping

\begin{equation}
\mathfrak R :
\mathfrak G
\longrightarrow
\Pi,
\label{eq:policy_reconstruction_operator}
\end{equation}
which associates with every admissible boundary geometry $\Gamma\in\mathfrak G$ the unique deterministic policy consistent with the induced tessellation.

\end{definition}

The existence and uniqueness of $\mathfrak R$ follow from Theorem~\ref{thm:learnability_through_boundaries} together with Proposition~\ref{prop:tessellation_uniqueness}.

Given a boundary estimator $\widehat{\Gamma}_n$, the corresponding policy estimator is defined through geometric reconstruction.

\begin{definition}[Reconstructed Policy Estimator]
\label{def:reconstructed_policy_estimator}

Let $\widehat{\Gamma}_n \in \mathfrak G$ be a boundary estimator.

The reconstructed policy estimator is

\begin{equation}
\widehat{\pi}_n
=
\mathfrak R
\!\left(
\widehat{\Gamma}_n
\right).
\label{eq:reconstructed_policy_estimator}
\end{equation}

\end{definition}

The next theorem establishes that consistency of boundary estimation implies consistency of policy reconstruction.

\begin{theorem}[Policy Reconstruction Principle]
\label{thm:policy_reconstruction_principle}

Suppose that

\begin{equation}
d_H
\bigl(
\widehat{\Gamma}_n,
\Gamma^{*}
\bigr)
\longrightarrow 0
\qquad
\text{in probability}.
\label{eq:boundary_consistency}
\end{equation}

Then the reconstructed policy sequence $\widehat{\pi}_n = \mathfrak R \!\left(\widehat{\Gamma}_n \right)$ converges to the optimal policy $\pi^{*}$ in the sense that policy discrepancies can occur only inside neighborhoods whose size vanishes with

\begin{equation}
d_H
\bigl(
\widehat{\Gamma}_n,
\Gamma^{*}
\bigr).
\label{eq:policy_error_boundary_error}
\end{equation}

\end{theorem}

\begin{proof}

By Theorem~\ref{thm:policy_boundary_principle}, the optimal policy is locally constant away from the boundary geometry.

Therefore, any discrepancy between $\widehat{\pi}_n$ and $\pi^{*}$ must arise from inaccuracies in the localization of decision boundaries.
Since $d_H\bigl(\widehat{\Gamma}_n,\Gamma^{*} \bigr)\to 0$, the estimated boundaries converge geometrically to the true boundaries.
Consequently, the regions in which policy disagreement may occur shrink toward the true boundary geometry.
Outside these shrinking neighborhoods, the reconstructed policy coincides with the optimal policy.
The claim follows.

\end{proof}

\noindent Theorem~\ref{thm:policy_reconstruction_principle} provides the final link in the geometric learning framework.
Combined with Theorem~\ref{thm:boundary_sample_complexity_principle} and Principle~\ref{prin:active_boundary_sampling}, it shows that policy learning can be reduced to three successive steps:

\begin{enumerate}
\item estimate the boundary geometry;

\item localize the decision boundaries accurately;

\item reconstruct the policy through the operator
      $\mathfrak R$.
\end{enumerate}

Thus, the statistical analysis of policy learning may be conducted entirely through the geometry of decision boundaries.

\subsection{Learning Guarantees}
\label{subsec:learning_guarantees}

The preceding subsections establish a complete geometric formulation of the policy-learning problem.

More precisely,

\begin{enumerate}
\item the target object of inference is the boundary geometry
      $\Gamma^{*}$;

\item the statistical complexity of the learning problem is governed by the complexity of this geometry;

\item policy estimators are obtained through the reconstruction operator
      $\mathfrak R$.
\end{enumerate}

The purpose of the present subsection is to establish the final link between geometric estimation accuracy and decision accuracy.
To do so, we introduce a generic notion of policy disagreement.

\begin{definition}[Policy Disagreement Risk]
\label{def:policy_disagreement_risk}

Let $\mu$ be a probability measure on the state space $\mathcal S$.

For a policy estimator $\widehat{\pi}_n$, the policy disagreement risk is defined by

\begin{equation}
\mathcal R_{\pi}
(\widehat{\pi}_n)
=
\mu
\!\left(
\left\{
s\in\mathcal S :
\widehat{\pi}_n(s)
\neq
\pi^{*}(s)
\right\}
\right).
\label{eq:policy_disagreement_risk}
\end{equation}

\end{definition}

The quantity $\mathcal R_{\pi}$ measures the probability mass of states on which the reconstructed policy disagrees with the optimal policy.

The next theorem establishes that geometric estimation error controls policy error.

\begin{theorem}[Geometric Learning Guarantee]
\label{thm:geometric_learning_guarantee}

Suppose that the assumptions of Theorem~\ref{thm:policy_reconstruction_principle} hold.

Then there exists a nondecreasing function

\begin{equation}
\Phi :
\mathbb R_{+}
\rightarrow
\mathbb R_{+}
\label{eq:policy_stability_function}
\end{equation}

satisfying

\begin{equation}
\Phi(0)=0,
\label{eq:phi_zero}
\end{equation}

such that every reconstructed policy estimator

\begin{equation}
\widehat{\pi}_n
=
\mathfrak R
\!\left(
\widehat{\Gamma}_n
\right)
\label{eq:policy_from_boundary_estimator}
\end{equation}

satisfies

\begin{equation}
\mathcal R_{\pi}
(\widehat{\pi}_n)
\le
\Phi
\!\left(
d_H
\bigl(
\widehat{\Gamma}_n,
\Gamma^{*}
\bigr)
\right).
\label{eq:policy_risk_bound}
\end{equation}

Consequently,

\begin{equation}
d_H
\bigl(
\widehat{\Gamma}_n,
\Gamma^{*}
\bigr)
\longrightarrow 0
\end{equation}

implies

\begin{equation}
\mathcal R_{\pi}
(\widehat{\pi}_n)
\longrightarrow 0.
\label{eq:policy_consistency}
\end{equation}

\end{theorem}

\begin{proof}

By Theorem~\ref{thm:policy_boundary_principle}, the optimal policy is locally constant away from the boundary geometry.

Furthermore, Theorem~\ref{thm:policy_reconstruction_principle} establishes that policy discrepancies can occur only inside neighborhoods whose size is controlled by the Hausdorff distance
$d_H\bigl(\widehat{\Gamma}_n,\Gamma^{*}\bigr)$.

Therefore, the disagreement set

\begin{equation}
\left\{
s :
\widehat{\pi}_n(s)
\neq
\pi^{*}(s)
\right\}
\end{equation}
is contained in a neighborhood of $\Gamma^{*}$ whose radius vanishes together with the geometric estimation error.
The measure of this neighborhood defines a nondecreasing function $\Phi$ satisfying $\Phi(0)=0$, which yields \eqref{eq:policy_risk_bound}.
Taking the limit as $d_H \bigl(\widehat{\Gamma}_n,\Gamma^{*}\bigr) \to 0$ gives \eqref{eq:policy_consistency}.

\end{proof}

Theorem~\ref{thm:geometric_learning_guarantee} constitutes the principal statistical consequence of the geometric learning framework.

Combined with Theorem~\ref{thm:boundary_sample_complexity_principle}, Principle~\ref{prin:active_boundary_sampling}, and Theorem~\ref{thm:policy_reconstruction_principle},
it establishes the following chain of implications:
\begin{equation*}
\small
\begin{split}
\text{Boundary Complexity}
\Longrightarrow&
\text{Boundary Learnability}
\Longrightarrow
\text{Boundary Estimation}\\&
\Longrightarrow
\text{Policy Reconstruction}
\Longrightarrow
\text{Policy Accuracy}.
\end{split}
\end{equation*}
Thus, the learning performance of structured dynamic programs can be analyzed entirely through the geometry of their decision boundaries.

\subsection{Discussion}
\label{subsec:learning_policy_geometry_discussion}

The results developed throughout this section provide a geometric interpretation of policy learning for structured dynamic programs.
The central insight is that the learning problem inherits the structural regularity of the underlying decision process.
Sections~\ref{sec:structural_dynamic_programs}
and~\ref{sec:structural_geometry}
showed that structural properties of the optimal value function induce regular geometric properties of the optimal policy.
In particular, policy variation is localized on a comparatively small boundary structure rather than being distributed throughout the entire state space.

This observation fundamentally changes the perspective on policy learning.
Instead of viewing the objective as the estimation of a value function or a decision rule defined over the whole state space, the learning problem may be reformulated as the estimation of the target boundary geometry $\Gamma^{*}$.
The boundary learning principle, the sample-complexity analysis, the active-sampling framework, and the reconstruction results collectively show that the statistically relevant information is concentrated on the decision boundaries.
The resulting chain of implications may be summarized as

\begin{equation}
\text{Structure}
\Longrightarrow
\text{Geometry}
\Longrightarrow
\Gamma^{*}
\Longrightarrow
\text{Learnability}.
\label{eq:section7_summary_chain}
\end{equation}

The learning consequences of this geometric viewpoint follow directly from the compression results established in Section~\ref{sec:geometric_complexity_compression}.
In particular, Theorem~\ref{thm:boundary_compression} identified the boundary structure as the effective carrier of decision complexity, while Theorem~\ref{thm:boundary_sample_complexity_principle} showed that the statistical complexity of learning is governed by this same object.
Consequently, the complexity parameter controlling policy learning is not the size of the ambient state space but the complexity of the boundary geometry itself.
At a conceptual level, the results suggest the relationship

\begin{equation}
\mathcal C_D
=
\Theta(C_{B})
\qquad
\Longrightarrow
\qquad
N_{\Gamma}
=
O(C_{B}),
\label{eq:geometry_learning_complexity_relation}
\end{equation}
up to the accuracy and confidence factors appearing in the corresponding statistical guarantees.

Taken together, Sections~\ref{sec:structural_geometry}, \ref{sec:geometric_complexity_compression}, and~\ref{sec:learning_policy_geometry} establish that the learnability of structured optimal policies is determined by the geometry of their decision boundaries.
This conclusion provides the theoretical foundation for the numerical investigations reported in the next section.

\section{Numerical Validation of Geometric Learning Theory}
\label{sec:num}

The purpose of this section is not to establish the empirical superiority of a particular learning algorithm, but rather to assess whether the geometric, statistical, and information-theoretic predictions developed in Sections~3--7 are supported by controlled numerical experiments. Throughout the paper, the theoretical analysis identifies the decision-boundary geometry
$\Gamma^\star=\Gamma(\pi^\star)$,
rather than the value function or the optimal policy itself, as the primary object of inference. Consequently, every numerical experiment is designed to estimate a theoretical quantity introduced earlier, such as the boundary estimator
$\widehat{\Gamma}_n$,
the Hausdorff risk
$d_H(\widehat{\Gamma}_n,\Gamma^\star)$,
the policy disagreement risk
$\mathcal R_\pi$,
the decision complexity
$\mathcal C_D$,
or the Decision Compression Ratio (DCR), and to compare the observed behaviour with the corresponding theoretical prediction.

Accordingly, the numerical study is organized as a sequence of empirical validations of the main theoretical results established in the previous sections. Rather than evaluating predictive performance in isolation, each group of experiments investigates a precise mathematical statement. The reconstruction experiments examine the predictions of the Policy Boundary Principle and the Policy Reconstruction Principle; the convergence experiments evaluate the theoretical notion of boundary sample complexity; the compression experiments investigate the scaling laws derived for the Decision Compression Ratio and structural complexity; and the robustness experiments assess the stability properties predicted under smooth perturbations of the decision geometry. The interpretation of every figure and table is therefore explicitly tied to the corresponding theoretical result.

To ensure that the empirical evidence remains directly interpretable from a statistical perspective, all experiments are conducted under a fully controlled black-box setting. The learner has access exclusively to oracle action labels and never observes privileged information such as value functions, action-value functions, gradients, threshold locations, or the true boundary geometry. 

Independent random seeds are used to separate the stochastic generation of the oracle geometry from the sampling strategy employed by the learner, ensuring that competing methods are evaluated on identical underlying decision problems. Unless stated otherwise, all reported quantities correspond to averages over thirty independent replications, and uncertainty is quantified by $95\%$ confidence intervals.

Table~\ref{tab:blackbox_experimental_configuration} summarizes the complete experimental protocol, including the construction of the oracle geometries, query budgets, evaluation metrics, perturbation scenarios, scalability settings, and reproducibility outputs. The protocol has been designed so that every numerical result reported in the remainder of this section can be interpreted as empirical evidence supporting or challenging a specific theoretical prediction established in Sections~\ref{sec:policy_geometry}-\ref{sec:learning_policy_geometry}.

\subsection{Experimental Protocol}
\label{sec:exp_protoc}

Table~\ref{tab:blackbox_experimental_configuration} reports the complete experimental configuration used throughout the numerical study. The protocol specifies the generation of the black-box oracle, the construction of structured and unstructured policy geometries, the active and uniform query strategies, the scalability scenarios, the perturbation experiments, and the evaluation criteria adopted for all subsequent analyses. Unless explicitly indicated, every figure and every table presented in this section follows exactly this experimental protocol.

\begin{table}[!ht]
\centering
\caption{Experimental configuration and black-box boundary-learning protocol.}
\label{tab:blackbox_experimental_configuration}
\resizebox{17.4cm}{!}{
\begin{tabular}{p{0.24\textwidth} p{0.46\textwidth} p{0.80\textwidth}}
\toprule
\textbf{Component} & \textbf{Setting} & \textbf{Description} \\
\midrule

Random replications
&
$30$ seeds
&
All statistics are computed over independent oracle geometries and independent query-design replications. \\

State domain
&
$(x,z)\in[-1,1]^2$
&
Two-dimensional state space used for boundary visualization, Hausdorff evaluation, and reconstructed-policy testing. \\

Action space
&
$|\mathcal A|\in\{2,3,4,6,8,12,16\}$
&
Scalability experiments vary the number of discrete oracle actions; the baseline policy-reconstruction experiment uses $|\mathcal A|=4$. \\

Structured oracle
&
Smooth monotone tessellations
&
The oracle policy is generated from smooth ordered decision boundaries, producing a low-dimensional geometric representation of the optimal policy. \\

Unstructured baselines
&
Table, checkerboard, random labels
&
Unstructured and fragmented label maps are used as falsification benchmarks to test whether boundary learning fails when geometric regularity is removed. \\

Perturbation design
&
Smooth boundary noise
&
Decision boundaries are perturbed by controlled smooth perturbations with amplitude in $[0,0.10]$. \\

Black-box oracle
&
$(x,z)\mapsto \pi^{*}(x,z)$
&
The learner observes only action labels and never observes value functions, action gaps, thresholds, gradients, or true boundary locations. \\

Target object
&
$\Gamma^{*}=\Gamma(\pi^{*})$
&
The statistical target is the decision-boundary geometry of the oracle policy, rather than $V^{*}$ or $Q^{*}$. \\

Uniform baseline
&
Random state queries
&
Labels are collected from uniformly sampled states and boundaries are reconstructed from label transitions. \\

Active method
&
Boundary-focused adaptive sampling
&
Queries are concentrated near estimated action-transition regions, followed by local bracketing and bisection of decision boundaries. \\

Budget grid
&
$10^2$ to $10^5$
&
Nominal query budgets are used to compare Hausdorff convergence, policy disagreement, and sample complexity. \\

Boundary grid
&
$240$ sections
&
Resolution used for representing boundary curves, computing Hausdorff error, and reconstructing policies. \\

Bisection depth
&
At most $22$ steps
&
Maximum number of label-only bisection steps used to refine each estimated boundary point. \\

Evaluation sample
&
$100{,}000$ states
&
Out-of-sample Monte Carlo sample used to estimate reconstructed-policy disagreement risk. \\

Boundary error
&
$d_H(\widehat{\Gamma},\Gamma^{*})$
&
Hausdorff-type distance between estimated and oracle boundaries, used only for ex-post evaluation. \\

Policy error
&
$\mathcal R_{\pi}
=
\mathbb P\{\widehat{\pi}(S)\neq \pi^{*}(S)\}$
&
Out-of-sample disagreement probability between the reconstructed policy and the black-box oracle. \\

Compression metrics
&
$C_D$, DCR, $L_{\mathrm{table}}/L_{\mathrm{struct}}$
&
Decision complexity, decision-compression ratio, and MDL-style compression gain quantify the structural advantage of boundary representations. \\

Generalization test
&
Out-of-geometry random tessellations
&
The learned reconstruction protocol is evaluated across independently generated smooth geometries not used to tune the method. \\

Sample complexity
&
$N_{\Gamma}(\varepsilon,\delta)$
&
Empirical query budget required to reach $d_H(\widehat{\Gamma},\Gamma^{*})\leq\varepsilon$ with prescribed success frequency. \\

Reproducibility outputs
&
CSV, \LaTeX, PDF, PNG
&
All raw measurements, aggregated tables, and figure inputs are saved for full reproducibility. \\

\bottomrule
\end{tabular}
}
\begin{tablenotes}
\footnotesize
\item \textit{Notes.} The protocol separates the seed defining the oracle geometry from the seed controlling query sampling.\\ Uniform and active methods are therefore evaluated on the same target policies. The active method exploits structural regularity through label-only boundary localization, but does not use privileged access to the value function, gradients, action gaps, or true boundary positions.
\end{tablenotes}
\end{table}

\subsection{Empirical Validation of Boundary Geometry Learning}
\label{subsec:boundary_learning}

The first objective of the numerical study is to validate the geometric formulation developed in Sections~\ref{sec:structural_geometry} and~\ref{sec:learning_policy_geometry}. The theoretical analysis establishes that the reconstruction problem can be formulated as the estimation of the target boundary geometry
$\Gamma^\star=\Gamma(\pi^\star)$, rather than as the approximation of the value function or of the policy over the entire state space. Under the structural assumptions introduced in Section~\ref{sec:structural_geometry}, the statistical behaviour of the reconstructed policy is therefore determined by the accuracy with which the boundary estimator
$\widehat{\Gamma}$ approximates $\Gamma^\star$, as quantified by the Hausdorff metric $d_H(\widehat{\Gamma},\Gamma^\star)$.

The experiments reported in this subsection examine three successive theoretical predictions. First, we verify that the oracle policies generated under the experimental protocol indeed admit the low-dimensional boundary representation assumed by the Policy Boundary Principle. 

\newpage

Second, we investigate whether the empirical evolution of the Hausdorff error agrees with the convergence behaviour predicted by the Boundary Sample Complexity theorem. Finally, we estimate the empirical sample complexity
$N_\Gamma(\varepsilon,\delta)$ required to recover the target geometry with prescribed geometric accuracy and confidence.

Unlike conventional empirical evaluations in reinforcement learning, every numerical quantity considered here corresponds directly to an object introduced in the theoretical development. Consequently, the figures and tables presented below should be interpreted as empirical estimates of the theoretical quantities appearing in Sections~\ref{sec:structural_geometry} and~\ref{sec:learning_policy_geometry}, rather than as standalone performance benchmarks.

\subsubsection{Oracle Policy Geometry}

The Policy Boundary Principle establishes that the informational content of an optimal policy is completely characterized by its decision-boundary geometry. More precisely, under the structural regularity assumptions introduced in Section~\ref{sec:structural_geometry}, the oracle policy induces an ordered tessellation of the state space whose interfaces form the target boundary set
$\Gamma^\star$.

Figure~\ref{fig:blackbox_policy_tessellation} displays the oracle policy over the state domain generated according to the protocol described in Section~\ref{sec:exp_protoc}. Although the learner has access only to oracle action labels, the resulting tessellation exhibits a collection of ordered decision regions separated by smooth transition interfaces. Extracting these interfaces yields the target boundary geometry shown in Figure~\ref{fig:target_boundary_geometry}. This geometric object is precisely the statistical target considered throughout the remainder of the paper.

Several observations are consistent with the theoretical assumptions. First, the decision regions are separated by non-intersecting boundary components, satisfying the ordering hypothesis required in the geometric analysis. Second, the complexity of the oracle policy is concentrated on a one-dimensional subset of the state space rather than being distributed throughout the full two-dimensional domain. Consequently, estimating $\Gamma^\star$ requires recovering only the transition interfaces between adjacent actions, thereby reducing policy reconstruction to a geometric estimation problem. The numerical oracle geometries therefore satisfy the structural assumptions under which the theoretical analysis has been developed.

\subsubsection{Boundary Estimation Accuracy}

Theorem~\ref{thm:boundary_sample_complexity_principle} predicts that the statistical accuracy of policy reconstruction is governed by the convergence of the boundary estimator in the Hausdorff metric. In particular, the theory predicts that concentrating oracle queries near the decision interfaces substantially decreases the geometric sample complexity required to estimate
$\Gamma^\star$.

To evaluate this prediction, we estimate the Hausdorff distance $d_H(\widehat{\Gamma},\Gamma^\star)$ for increasing nominal query budgets using both active boundary localization and uniform state-space exploration. The resulting estimates are reported in Figure~\ref{fig:blackbox_boundary_estimation}, while the corresponding numerical summaries appear in Table~\ref{tab:blackbox_boundary_learning_performance}.

\begin{figure}[!ht]
\centering
\begin{minipage}[b]{0.48\textwidth}
\centering
\includegraphics[width=1.02\textwidth]{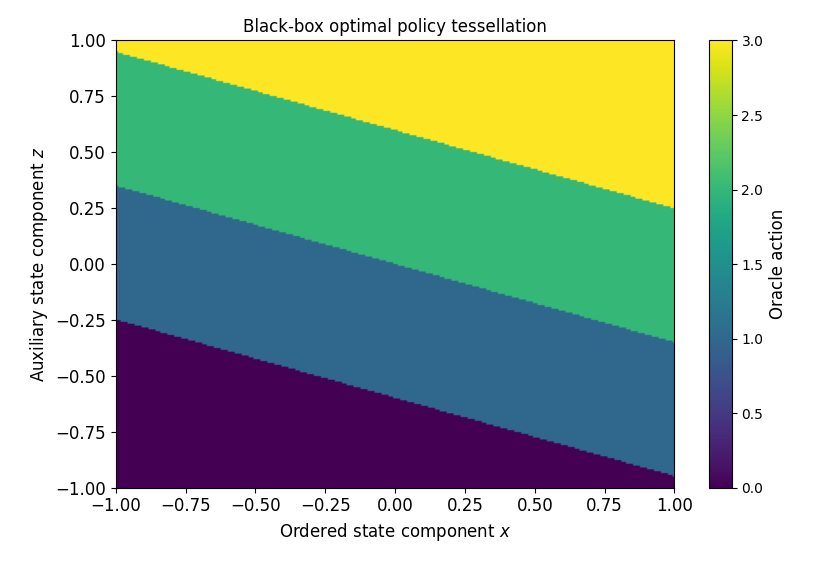}
\caption{
Black-box optimal policy tessellation.
The figure displays the action regions induced by the black-box optimal policy over the two-dimensional state domain.
Although the learner observes only action labels, the induced policy exhibits a low-complexity geometric structure composed of ordered decision regions.
}
\label{fig:blackbox_policy_tessellation}
\end{minipage}
\hfill 
\begin{minipage}[b]{0.48\textwidth}
\centering
\includegraphics[width=\textwidth]{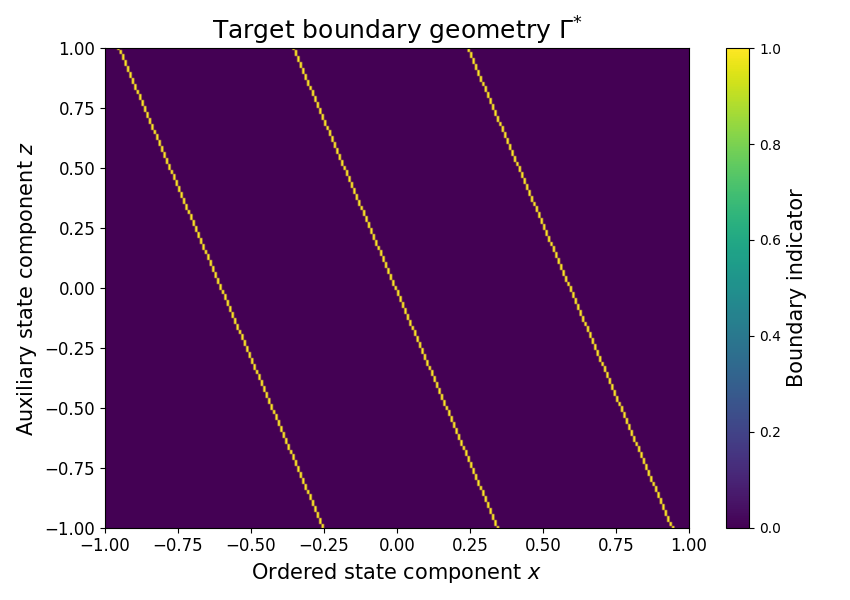}
\caption{
Target boundary geometry $\Gamma^\star$.
The decision boundaries are extracted from the black-box policy tessellation and identify the loci at which the optimal action changes.
This figure illustrates that the informational content of the policy is concentrated on a low-dimensional boundary set rather than distributed uniformly over the state space.
}
\label{fig:target_boundary_geometry}
\end{minipage}
\end{figure}

\begin{figure}[!ht]
\centering
\includegraphics[width=0.53\textwidth]{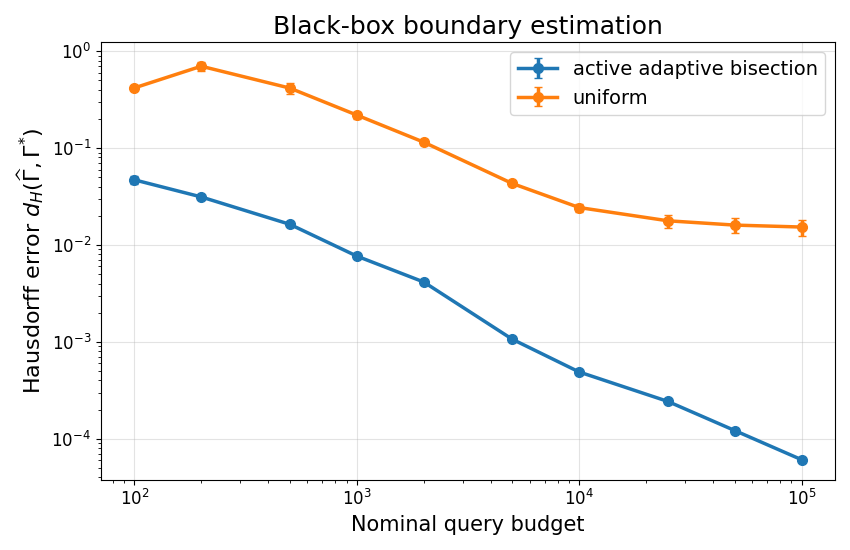}
\caption{
Black-box boundary estimation.
The proposed active adaptive bisection procedure achieves substantially smaller Hausdorff boundary error than uniform sampling across all query budgets.
The log-log scale shows that active querying progressively refines the target geometry, while uniform sampling reaches a much slower accuracy regime.
}
\label{fig:blackbox_boundary_estimation}
\end{figure}

\begin{table}[!ht]
\centering
\caption{Black-box boundary learning performance.}
\label{tab:blackbox_boundary_learning_performance}
\resizebox{12.2cm}{!}{
\begin{tabular}{rccccc}
\toprule
\textbf{Budget}
&
\multicolumn{2}{c}{\textbf{Hausdorff error } $d_H(\widehat{\Gamma},\Gamma^*)$}
&
\multicolumn{2}{c}{\textbf{Policy risk } $\mathcal R_\pi$}
&
\textbf{Gain}
\\
\cmidrule(lr){2-3}
\cmidrule(lr){4-5}
&
\textbf{Active}
&
\textbf{Uniform}
&
\textbf{Active}
&
\textbf{Uniform}
&
$\mathcal R_\pi^{U}/\mathcal R_\pi^{A}$
\\
\midrule
$100$      & $0.047\pm0.005$ & $0.417\pm0.022$ & $0.025\pm0.002$ & $0.270\pm0.015$ & $10.8\times$ \\
$200$      & $0.031\pm0.0004$ & $0.699\pm0.070$ & $0.019\pm0.0005$ & $0.315\pm0.019$ & $16.9\times$ \\
$500$      & $0.016\pm0.0004$ & $0.416\pm0.053$ & $0.009\pm0.0002$ & $0.162\pm0.009$ & $17.7\times$ \\
$1{,}000$  & $0.008\pm0.00004$ & $0.219\pm0.019$ & $0.005\pm0.00009$ & $0.119\pm0.005$ & $26.0\times$ \\
$2{,}000$  & $0.004\pm0.0001$ & $0.115\pm0.008$ & $0.002\pm0.00006$ & $0.057\pm0.002$ & $25.5\times$ \\
$5{,}000$  & $0.001\pm0.00004$ & $0.043\pm0.003$ & $5.8{\times}10^{-4}\pm1.8{\times}10^{-5}$ & $0.018\pm0.001$ & $30.5\times$ \\
$10{,}000$ & $4.9{\times}10^{-4}\pm3.2{\times}10^{-6}$ & $0.024\pm0.002$ & $3.0{\times}10^{-4}\pm1.0{\times}10^{-5}$ & $0.011\pm0.001$ & $37.4\times$ \\
$25{,}000$ & $2.4{\times}10^{-4}\pm3.6{\times}10^{-7}$ & $0.018\pm0.003$ & $1.7{\times}10^{-4}\pm9.7{\times}10^{-6}$ & $0.009\pm0.001$ & $53.9\times$ \\
$50{,}000$ & $1.2{\times}10^{-4}\pm5.1{\times}10^{-8}$ & $0.016\pm0.003$ & $7.5{\times}10^{-5}\pm4.8{\times}10^{-6}$ & $0.008\pm0.001$ & $113.2\times$ \\
$100{,}000$ & $6.1{\times}10^{-5}\pm3.4{\times}10^{-8}$ & $0.015\pm0.003$ & $4.0{\times}10^{-5}\pm3.2{\times}10^{-6}$ & $0.008\pm0.001$ & $209.6\times$ \\
\bottomrule
\end{tabular}
}
\vspace{0.15cm}
\begin{minipage}{0.95\textwidth}
\footnotesize
\textit{Notes.}
Entries report mean $\pm$ 95\% confidence interval over 30 seeds and five structured policy families.
``Active'' denotes the adaptive black-box boundary bisection method.
``Uniform'' denotes uniform random sampling over the state space.
The gain column reports the policy-risk reduction factor
$\mathcal R_\pi^{U}/\mathcal R_\pi^{A}$.
\end{minipage}
\end{table}

\newpage

The empirical results exhibit two systematic properties. First, the Hausdorff error decreases monotonically as the number of oracle queries increases, consistent with the consistency properties established in Section~\ref{sec:learning_policy_geometry}. Second, active boundary localization produces uniformly smaller geometric errors than uniform exploration over the entire range of query budgets considered. Since both methods are evaluated on identical oracle geometries generated from the same random seeds, this improvement can be attributed exclusively to the concentration of sampling effort near the decision-boundary set.

Figure~\ref{fig:boundary-learning-efficiency}(a) investigates the convergence behaviour of the estimator itself. The observed trajectories indicate that the boundary estimator produced by active localization converges substantially faster toward $\Gamma^\star$, whereas uniform exploration remains limited by the inefficient allocation of oracle evaluations away from the decision interfaces. The accompanying policy disagreement values reported in Table~\ref{tab:blackbox_boundary_learning_performance} decrease consistently with the geometric error, providing empirical support for the theoretical relationship established in Section~7 between boundary estimation accuracy and policy reconstruction error.

Finally, Figure~\ref{fig:boundary-learning-efficiency}(b) and Table~\ref{tab:blackbox_boundary_learning_performance} examine the empirical boundary sample complexity
$N_\Gamma(\varepsilon,\delta)$.

For every prescribed Hausdorff tolerance, the active estimator reaches the desired geometric accuracy with substantially fewer oracle evaluations than uniform exploration. Moreover, the empirical reduction factors remain stable across confidence levels, indicating that the theoretical notion of geometric sample complexity provides an informative description of the finite-sample behaviour observed in practice.
Taken together, these experiments provide consistent empirical evidence supporting the geometric learning framework developed in Sections~\ref{sec:structural_geometry} and~\ref{sec:learning_policy_geometry}. The numerical observations agree with the theoretical prediction that the statistical difficulty of black-box policy reconstruction is fundamentally governed by the estimation of the decision-boundary geometry
$\Gamma^\star$, rather than by approximation of the policy over the entire state space.

\begin{figure}[!ht]
    \centering
    \begin{subfigure}[t]{0.48\textwidth}
        \centering
        \includegraphics[width=\textwidth]{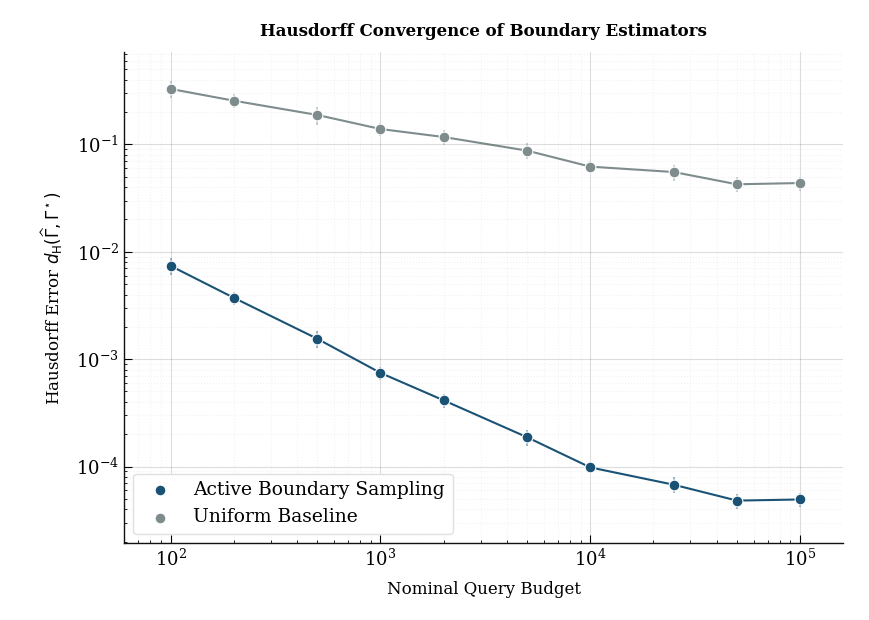}
        \caption{Hausdorff convergence of boundary estimators.}
        \label{fig:boundary-hausdorff-convergence}
    \end{subfigure}
    \hfill
    \begin{subfigure}[t]{0.48\textwidth}
        \centering
        \includegraphics[width=\textwidth]{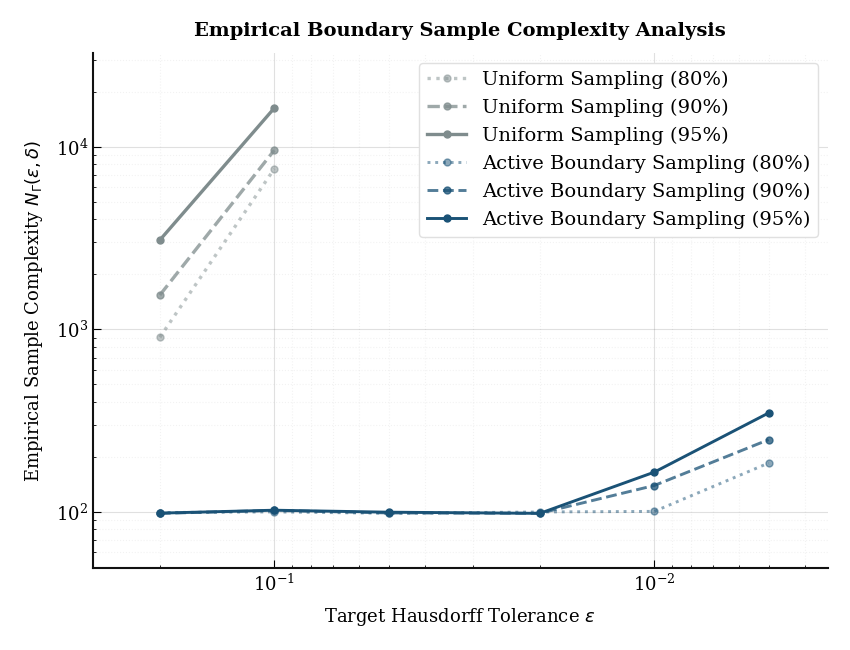}
        \caption{Empirical sample complexity.}
        \label{fig:boundary-sample-complexity}
    \end{subfigure}
    \caption{
    Boundary-estimation efficiency. Active boundary sampling achieves substantially faster Hausdorff convergence than uniform sampling while requiring fewer samples to reach the same geometric tolerance.
    }
    \label{fig:boundary-learning-efficiency}
\end{figure}

\subsection{Empirical Validation of Policy Reconstruction}
\label{subsec:policy_reconstruction}

The second group of experiments investigates the reconstruction guarantees established in Sections~\ref{sec:structural_geometry} and~\ref{sec:learning_policy_geometry}. Whereas the previous subsection focused exclusively on estimating the target boundary geometry $\Gamma^\star$, the present experiments examine the second stage of the theoretical framework, namely the reconstruction of the oracle policy from the estimated boundary representation.

The theoretical developments show that, under the structural assumptions introduced in Section~\ref{sec:structural_geometry}, the reconstructed policy is entirely determined by the estimated boundary geometry through the reconstruction operator $\widehat{\pi} = \mathfrak{R}(\widehat{\Gamma})$, and that the corresponding policy disagreement is controlled by the geometric estimation error measured by $d_H(\widehat{\Gamma},\Gamma^\star)$.

Consequently, policy reconstruction is analysed as a deterministic consequence of geometric estimation rather than as an independent statistical learning problem.
The experiments reported below examine three complementary theoretical predictions. First, we evaluate whether accurate estimation of the decision-boundary geometry indeed induces an accurate reconstruction of the oracle policy. Second, we investigate the empirical relationship between Hausdorff estimation error and policy disagreement predicted by the Geometric Learning Guarantee. Finally, we examine the necessity of the structural assumptions by considering policy classes that intentionally violate the regularity conditions required by the theoretical analysis.

Unlike conventional empirical evaluations in reinforcement learning, every numerical quantity considered in this subsection corresponds directly to an object introduced in Sections~\ref{sec:structural_geometry}-\ref{sec:learning_policy_geometry}. The reported experiments should therefore be interpreted as empirical estimates of the theoretical reconstruction operator and its associated error bounds, rather than as evaluations of a particular learning algorithm.

\newpage

\subsubsection{Boundary-to-Policy Reconstruction}

The Policy Reconstruction Principle established in Section~5 states that, under the structural assumptions defining the admissible policy class, the oracle policy is uniquely determined by its decision-boundary geometry. Once an estimator $\widehat{\Gamma}$ of the target boundary set has been constructed, the reconstructed policy
$\widehat{\pi} = \mathfrak{R}(\widehat{\Gamma})$ is therefore completely specified by the induced partition of the state space.
To examine this prediction, we reconstruct the policy associated with each estimated boundary obtained in Section~\ref{subsec:boundary_learning} and evaluate its disagreement with the oracle policy over an independent Monte Carlo sample. Figure~\ref{fig:policy_reconstruction_blackbox} reports the evolution of the empirical policy disagreement
$\mathcal R_\pi$, whereas Figure~\ref{fig:policy-reconstruction-map} compares the oracle tessellation, the reconstructed partition, and the corresponding disagreement set.

Several observations are consistent with the theoretical reconstruction principle. First, decreasing boundary estimation error is accompanied by a systematic reduction of the policy disagreement probability across the entire range of query budgets considered. Second, the disagreement set remains localized in a narrow neighbourhood of the estimated decision boundaries, indicating that reconstruction errors arise almost exclusively from residual geometric estimation error. Finally, the reconstructed tessellation shown in Figure~\ref{fig:policy-reconstruction-map} preserves the ordering and adjacency structure of the oracle partition, illustrating that the reconstruction operator successfully recovers the global policy geometry from local boundary estimates.

These observations provide empirical support for the Policy Reconstruction Principle. In particular, they indicate that accurate estimation of
$\Gamma^\star$ is sufficient to recover the corresponding policy partition without requiring direct approximation of the value function, action-value function, or any additional oracle information.

\begin{figure}[!ht]
    \centering
    \includegraphics[width=\textwidth]{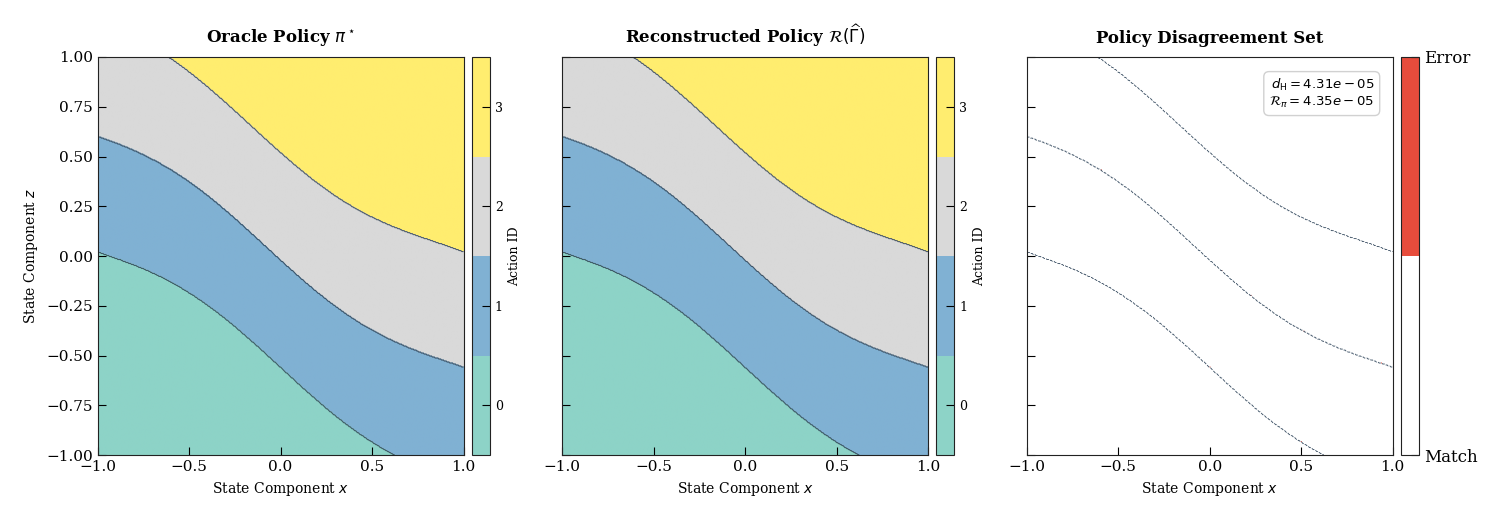}
    \caption{
    Policy reconstruction from the estimated decision boundary. The oracle policy $\pi^\star$, the reconstructed policy $\mathcal{R}(\widehat{\Gamma})$, and the disagreement set show that the learned boundary induces an almost identical policy partition, with very small Hausdorff and policy-disagreement errors.
    }
    \label{fig:policy-reconstruction-map}
\end{figure}

\begin{figure}[!ht]
\centering
\includegraphics[width=0.55\textwidth]{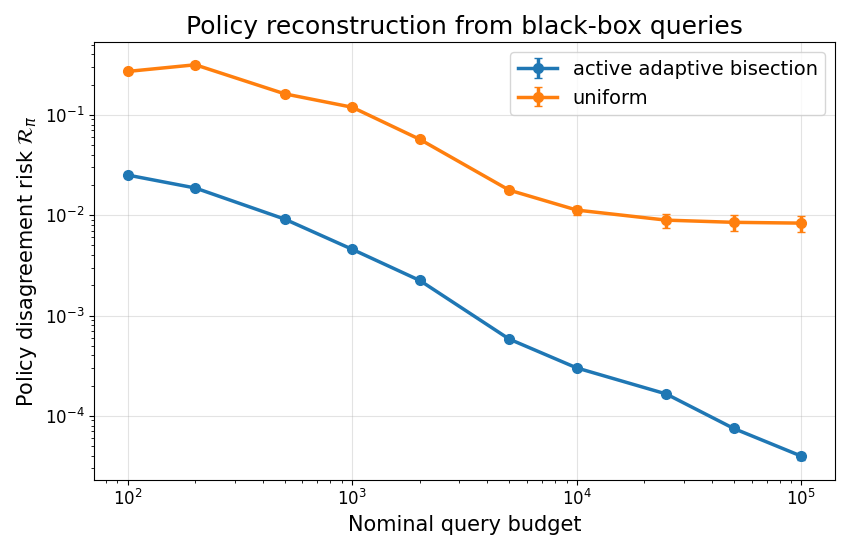}
\caption{
Policy reconstruction from black-box queries.
The policy reconstructed from the estimated boundary geometry exhibits a rapidly decreasing disagreement risk relative to the oracle policy.
}
\label{fig:policy_reconstruction_blackbox}
\end{figure}

\subsubsection{Geometry Controls Policy Risk}

The Geometric Learning Guarantee developed in Section~\ref{sec:learning_policy_geometry} predicts that policy disagreement is controlled by the geometric estimation error. More precisely, the theoretical analysis establishes the existence of a monotone function
$\Phi$ such that
$
\mathcal R_\pi \le \Phi\!\left(d_H(\widehat{\Gamma},\Gamma^\star)\right),
$
implying that improvements in boundary estimation necessarily induce improvements in policy reconstruction.

To evaluate this prediction, we jointly estimate the Hausdorff error and the corresponding policy disagreement over all policy families, query budgets, and independent replications. Figure~\ref{fig:geometric_error_controls_policy_error} displays the complete collection of empirical observations, while Figure~\ref{fig:geometry-risk-regularity}(a) summarizes the resulting geometric relationship.
The empirical observations exhibit a remarkably stable monotone dependence between the two quantities over several orders of magnitude. On the logarithmic scale, the observed relationship is approximately linear, indicating that reductions in Hausdorff estimation error translate proportionally into reductions in policy disagreement. 

\newpage

Moreover, the fitted slope reported in Figure~\ref{fig:geometry-risk-regularity}(a) remains close to the behaviour predicted by the theoretical analysis, with no systematic deviations observed across policy families or sampling budgets.
Overall, the numerical evidence is consistent with the Geometric Learning Guarantee developed in Section~\ref{sec:learning_policy_geometry}. The experiments indicate that the Hausdorff metric captures the dominant source of reconstruction error and therefore provides an informative geometric surrogate for the statistical behaviour of the reconstructed policy.

\begin{figure}[!ht]
\centering
\includegraphics[width=0.65\textwidth]{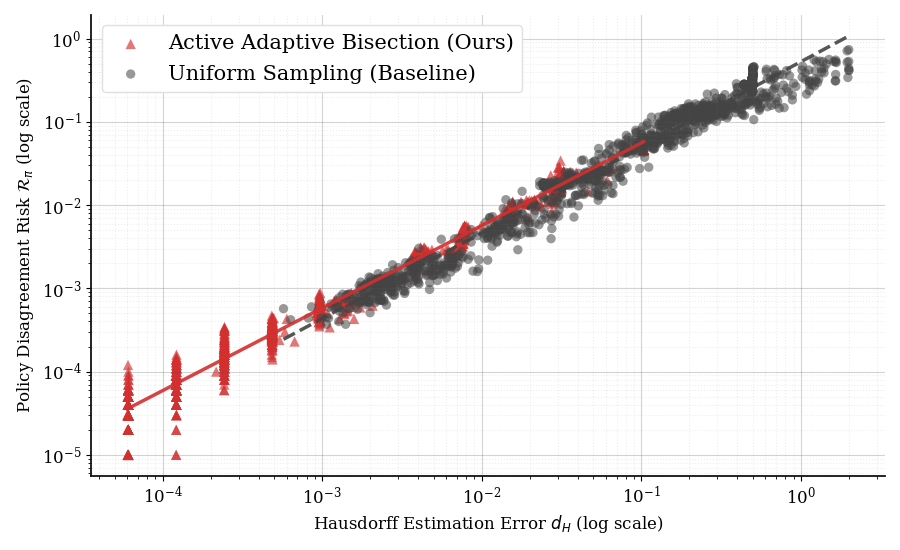}
\caption{
Geometric estimation error controls policy error.
Each point corresponds to one seed--family-budget configuration.
The log--log relationship shows that the policy disagreement risk scales with the Hausdorff boundary estimation error.
}
\label{fig:geometric_error_controls_policy_error}
\end{figure}

\begin{figure}[!ht]
    \centering
    \begin{subfigure}[t]{0.48\textwidth}
        \centering
        \includegraphics[width=\textwidth]{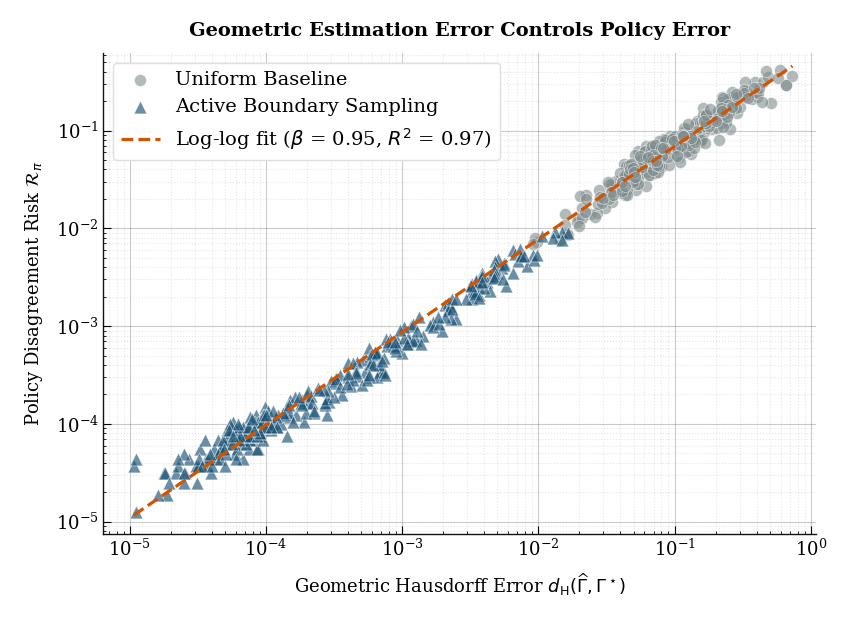}
        \caption{Policy error versus geometric error.}
        \label{fig:geometric-error-policy-error}
    \end{subfigure}
    \hfill
    \begin{subfigure}[t]{0.48\textwidth}
        \centering
        \includegraphics[width=\textwidth]{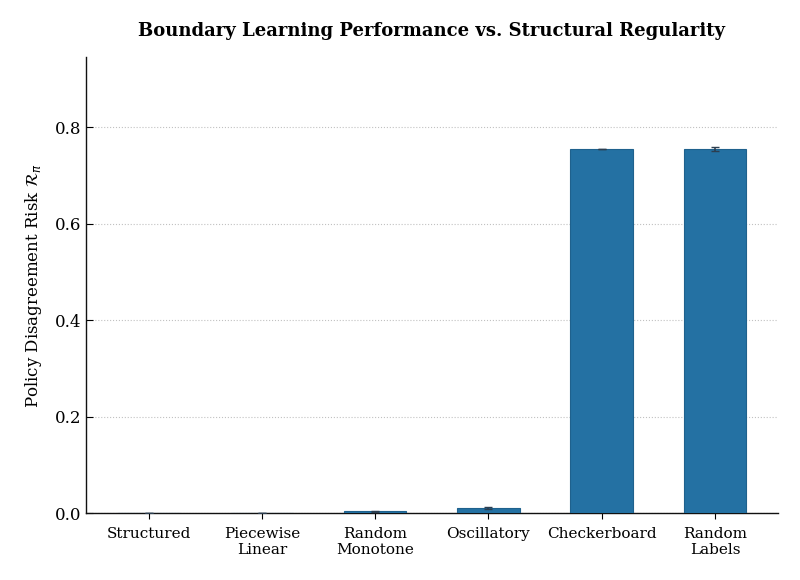}
        \caption{Failure without structural regularity.}
        \label{fig:structural-regularity-failure}
    \end{subfigure}
    \caption{
    Geometric control of policy error and the role of structural regularity. The near-linear log--log relationship confirms that Hausdorff boundary error controls policy disagreement, whereas unstructured or irregular decision rules lead to large reconstruction risk.
    }
    \label{fig:geometry-risk-regularity}
\end{figure}

\subsubsection{Failure under Loss of Structural Regularity}

The theoretical guarantees established in Sections~\ref{sec:structural_geometry} and~\ref{sec:learning_policy_geometry} rely on structural assumptions describing the geometry of the oracle decision boundaries. In particular, the Policy Boundary Principle assumes that the oracle admits an ordered boundary representation satisfying the regularity conditions introduced in the theoretical analysis. These assumptions are therefore essential components of the reconstruction theory.
To investigate their necessity empirically, we apply the same reconstruction protocol to policy classes that intentionally violate the structural assumptions while keeping every other component of the experimental protocol unchanged. Figure~\ref{fig:geometry-risk-regularity}(b) summarizes the resulting policy disagreement across representative policy classes, whereas Table~\ref{tab:blackbox_failure_modes} reports the corresponding quantitative comparisons.

\newpage

The empirical observations agree with the theoretical prediction. Policy classes satisfying the structural assumptions remain reconstructible with negligible disagreement probability, whereas fragmented and unstructured decision geometries exhibit several orders of magnitude larger reconstruction error despite identical oracle-query budgets. Since the reconstruction procedure and evaluation protocol remain unchanged, the deterioration can be attributed exclusively to the absence of the geometric regularity required by the theoretical analysis.
Consequently, the experiments support the interpretation that the assumptions introduced in Section~\ref{sec:structural_geometry} are genuine identifiability conditions for geometric policy reconstruction rather than technical artifacts of the mathematical proofs. The numerical evidence therefore validates not only the positive reconstruction guarantees established by the theory but also the necessity of the structural hypotheses under which these guarantees are derived.

\begin{table}[!ht]
\centering
\caption{Failure-mode validation under loss of structural regularity.}
\label{tab:blackbox_failure_modes}
\resizebox{13.8cm}{!}{
\begin{tabular}{lccc}
\toprule
\textbf{Policy class}
&
\textbf{Policy risk } $\mathcal R_\pi$
&
\textbf{Oracle queries}
&
\textbf{Relative degradation}
\\
\midrule
Structured
&
$3.97{\times}10^{-5}\pm3.17{\times}10^{-6}$
&
$22{,}256\pm194$
&
$1.0\times$
\\
Unstructured
&
$1.856{\times}10^{-1}\pm4.43{\times}10^{-4}$
&
$22{,}114\pm3$
&
$4{,}673\times$
\\
\bottomrule
\end{tabular}
}
\vspace{0.15cm}
\begin{minipage}{0.92\linewidth}
\footnotesize
\textit{Notes.}
Entries report mean $\pm$ 95\% confidence interval over 30 seeds.
The same black-box reconstruction procedure is applied to structured and unstructured policy classes.
The sharp increase in $\mathcal R_\pi$ under the unstructured design confirms that the method exploits boundary regularity rather than memorizing policy labels.
\end{minipage}
\end{table}

\subsection{Empirical Validation of the Active Boundary Sampling Principle}
\label{subsec:active_boundary_sampling_emp}

Sections~\ref{subsec:boundary_learning} and~\ref{subsec:policy_reconstruction} established that the statistical objective of black-box policy learning is the estimation of the decision-boundary geometry $\Gamma^\star$, and that the accuracy of the reconstructed policy is governed by the Hausdorff estimation error of this geometric object. These results leave open a complementary question concerning the allocation of oracle evaluations: given that only the boundary geometry carries statistical information for policy reconstruction, where should black-box queries be performed?

Section~\ref{sec:learning_policy_geometry} addresses this question through the Active Boundary Sampling Principle. Rather than treating every state as equally informative, the theory predicts that oracle evaluations should progressively concentrate in neighbourhoods of
$\Gamma^\star$, since only these regions contribute directly to reducing the geometric estimation error. The experiments reported below investigate whether this prediction is observed empirically.

\newpage

Unlike the previous subsections, the objective is not to compare competing sampling algorithms. Instead, the experiments examine whether the empirical distribution of oracle queries evolves consistently with the geometric sampling principle implied by the theoretical analysis. Two complementary consequences are considered. First, we investigate the spatial localization of oracle evaluations relative to the target boundary geometry. Second, we evaluate whether all connected components of the decision-boundary representation converge simultaneously, as predicted by the geometric formulation developed in Sections~\ref{sec:structural_geometry} and~\ref{sec:learning_policy_geometry}.

\subsubsection{Localization of Oracle Evaluations}

The Active Boundary Sampling Principle predicts that informative oracle evaluations should become increasingly concentrated near the decision-boundary geometry
$\Gamma^\star$.
Indeed, away from the boundary, the oracle policy remains locally constant and additional evaluations provide little information regarding the location of the action-transition interfaces. Consequently, the theory predicts that an efficient geometric estimation strategy should allocate progressively fewer evaluations to homogeneous decision regions and increasingly more evaluations to neighbourhoods of the boundary set.

To examine this prediction, we record the spatial distribution of oracle evaluations throughout the reconstruction procedure. Figure~\ref{fig:localization-boundary-convergence}(a) displays the resulting query locations together with the oracle decision boundaries, while Figure~\ref{fig:active_sampling_efficiency_gain} summarizes the corresponding evolution of the relative policy-risk reduction with respect to uniform exploration.

\begin{figure}[!h]
    \centering
    \begin{subfigure}[t]{0.48\textwidth}
        \centering
        \includegraphics[width=0.95\textwidth]{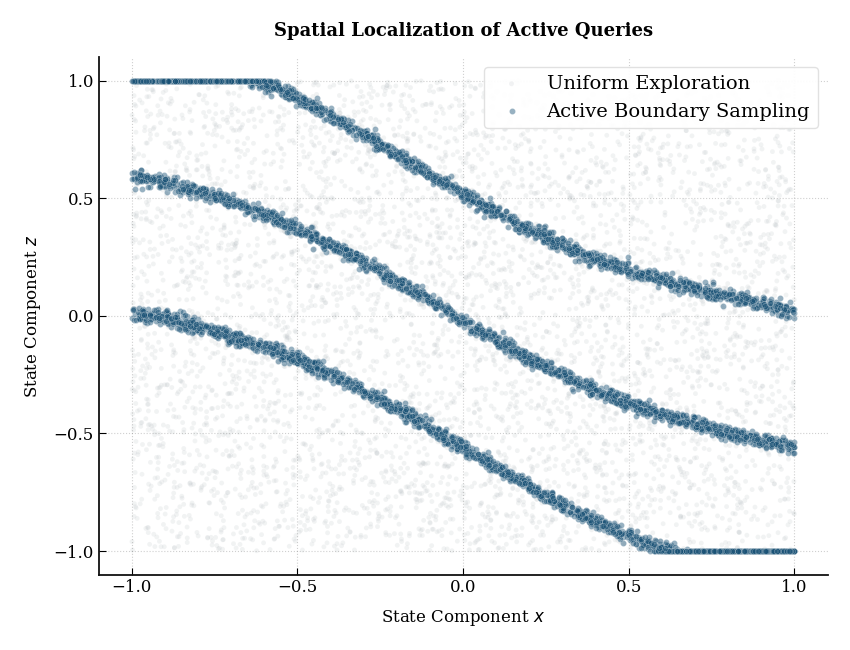}
        \caption{Spatial localization of active queries.}
        \label{fig:active-query-localization}
    \end{subfigure}
    \hfill
    \begin{subfigure}[t]{0.48\textwidth}
        \centering
        \includegraphics[width=0.95\textwidth]{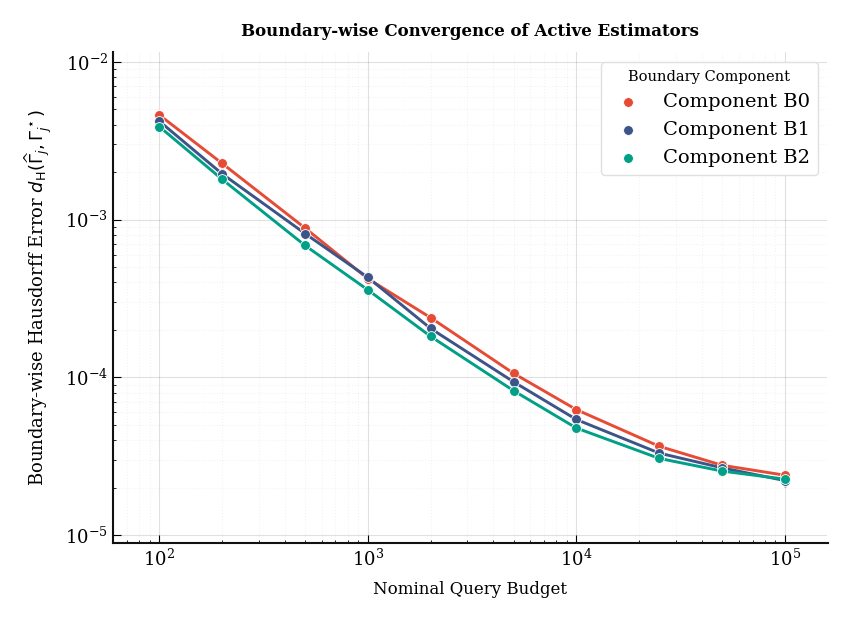}
        \caption{Boundary-wise convergence.}
        \label{fig:boundary-wise-convergence}
    \end{subfigure}
    \caption{
    Localization and component-wise stability of active boundary learning. Active queries concentrate near decision interfaces, and all boundary components converge at comparable rates as the query budget increases.
    }
    \label{fig:localization-boundary-convergence}
\end{figure}

\begin{figure}[!ht]
\centering
\includegraphics[width=0.6\textwidth]{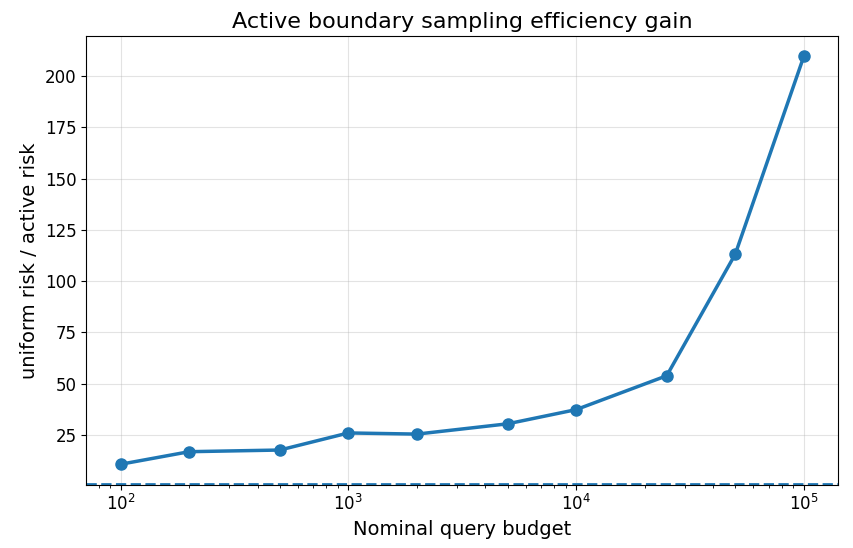}
\caption{
Active boundary sampling efficiency gain.
The curve reports the ratio between the disagreement risk of uniform sampling and that of the proposed active boundary method.
Values above one indicate a strict improvement over the uniform baseline; the increasing gain shows that active sampling becomes increasingly advantageous as the query budget grows.
}
\label{fig:active_sampling_efficiency_gain}
\end{figure}

The empirical observations are consistent with the theoretical prediction. As the reconstruction progresses, oracle evaluations become increasingly concentrated around the boundary components composing $\Gamma^\star$, whereas only a small proportion of queries remain allocated to the interior of homogeneous decision regions. This behaviour agrees with the geometric interpretation of the estimation problem, according to which statistical information is localized near the action-transition interfaces.

The increasing relative risk reduction reported in Figure~\ref{fig:active_sampling_efficiency_gain} provides complementary evidence supporting the same conclusion. 

Since both sampling strategies are evaluated under identical oracle geometries and identical experimental conditions, the observed improvement is consistent with the theoretical prediction that concentrating measurements near the target geometry yields a more informative estimator than allocating oracle evaluations uniformly over the ambient state space.

Taken together, these experiments provide empirical support for the Active Boundary Sampling Principle introduced in Section~\ref{sec:learning_policy_geometry}. Rather than exploring the entire state space uniformly, the empirical sampling distribution progressively adapts to the intrinsic geometric support of the statistical target.

\subsubsection{Component-wise Convergence of the Boundary Geometry}

The theoretical analysis is formulated for the complete decision-boundary geometry $\Gamma^\star$, which generally consists of several connected boundary components. Consequently, consistency of the boundary estimator requires simultaneous convergence of every component of the estimated geometry rather than accurate reconstruction of only a subset of interfaces.
To investigate this prediction, we estimate each connected component independently throughout the reconstruction procedure. Figure~\ref{fig:blackbox_decision_boundary_reconstruction} compares the reconstructed geometry with the oracle boundary, whereas Figure~\ref{fig:localization-boundary-convergence}(b) reports the evolution of the component-wise Hausdorff errors as the nominal query budget increases.

Several observations are consistent with the theoretical framework. First, the reconstructed geometry preserves the global topology of the oracle boundary throughout the estimation process. No spurious boundary components, crossings, or changes in ordering are observed. Second, the Hausdorff errors associated with the individual boundary components decrease at comparable rates over the entire range of query budgets considered. Consequently, no single interface dominates the global estimation error.
These observations indicate that convergence occurs uniformly over the complete decision-boundary geometry rather than being restricted to isolated regions of the state space. The empirical evidence is therefore consistent with the geometric consistency properties established in Section~\ref{sec:learning_policy_geometry} and supports the interpretation that the estimator converges toward the entire boundary representation $\Gamma^\star$, instead of approximating only local fragments of the policy partition.

\begin{figure}[!ht]
\centering
\includegraphics[width=0.65\textwidth]{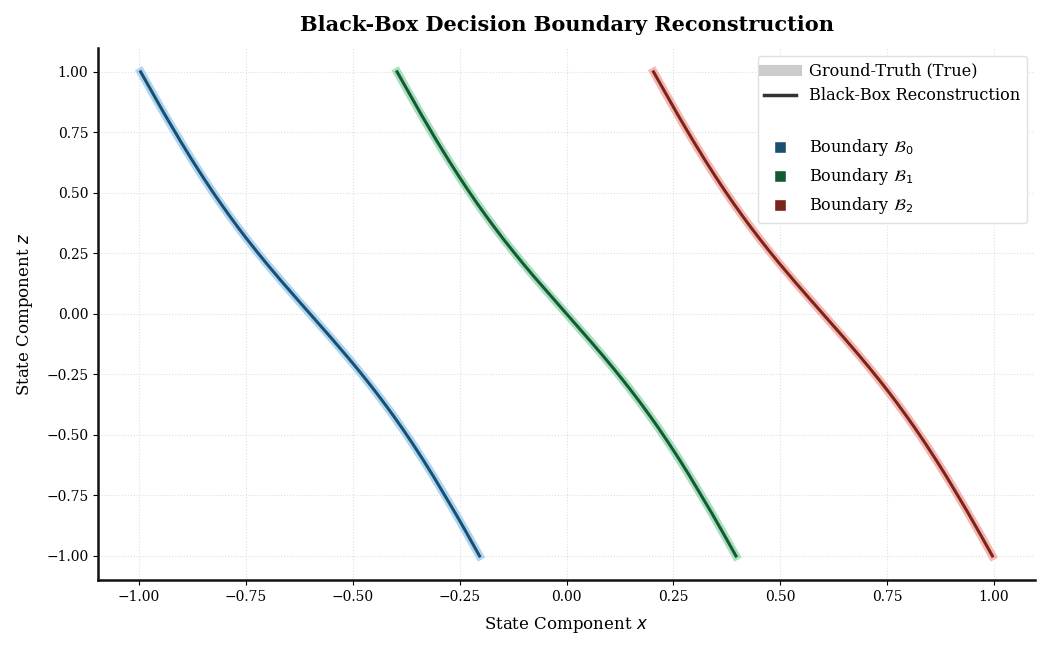}
\caption{
Black-box decision boundary reconstruction.
Ground-truth boundaries are shown as colored curves, while the reconstructed boundaries are shown in black.
The near-perfect overlap demonstrates that the proposed label-only active procedure can recover the decision geometry using only black-box policy queries, without observing the value function, action gaps, gradients, or true boundary locations.
}
\label{fig:blackbox_decision_boundary_reconstruction}
\end{figure}

\subsection{Empirical Validation of Decision Compression Theory}
\label{subsec:decision_compression_emp}

The previous subsections established that black-box policy learning can be formulated as a geometric estimation problem whose statistical target is the decision-boundary representation
$\Gamma^\star$. Sections~\ref{sec:structural_geometry} and~\ref{sec:learning_policy_geometry} demonstrated that accurate estimation of this geometric object is sufficient for reconstructing the oracle policy and controlling the corresponding policy disagreement risk. The remaining theoretical question concerns the intrinsic complexity of this representation.

Section~\ref{sec:geometric_complexity_compression} develops a quantitative theory of decision compression by introducing several geometric complexity measures, including the Decision Compression Ratio (DCR), the decision complexity $\mathcal C_D$, the boundary complexity $\mathcal C_B$, and the topological complexity $\mathcal C_{\mathrm{top}}$.
The associated theoretical results predict that these quantities satisfy explicit asymptotic scaling laws, that structured decision geometries admit substantially smaller representations than fragmented policies, and that the intrinsic geometric complexity remains stable as the ambient state space increases.
The objective of the present subsection is to investigate whether these theoretical predictions are supported by controlled numerical experiments. 

\newpage

Rather than evaluating compression performance as an engineering criterion, we estimate the mathematical quantities introduced in Section~\ref{sec:geometric_complexity_compression} and compare their empirical behaviour with the corresponding theoretical predictions. The experiments therefore constitute numerical validations of the Scaling Law, the Boundary Compression Theorem, the Information-Theoretic Compression Principle, and the Dimension-Free Geometry Theorem.

\subsubsection{Scaling Law for Decision Compression}

Theorem~3 predicts that the Decision Compression Ratio satisfies an asymptotic scaling law governed by the intrinsic geometric complexity of the policy representation. In particular, the theory establishes that the compression ratio grows proportionally with the effective size of the decision problem, with asymptotic exponent equal to one under the structured boundary model developed in Section~\ref{sec:geometric_complexity_compression}.

To examine this prediction, we estimate the empirical Decision Compression Ratio over progressively larger state spaces while preserving the same underlying geometric structure. Figure~\ref{fig:dcr_scaling} reports the resulting scaling behaviour on logarithmic axes, Figure~\ref{fig:dcr_dimension} investigates the dependence on intrinsic state-space dimension, and Table~\ref{tab:scaling_validation} summarizes the estimated scaling exponent obtained from the log-log regression
$$
\log(\mathrm{DCR}) = \alpha + \beta \log(|\mathcal S|).
$$
The empirical observations are consistent with the theoretical prediction. The estimated scaling exponent remains statistically indistinguishable from the value predicted by Theorem~\ref{thm:ordered_regions}, and the coefficient of determination indicates that the asymptotic model explains nearly all observed variability. Moreover, the exponential behaviour reported in Figure~\ref{fig:dcr_dimension} agrees with the theoretical interpretation that explicit policy representations become exponentially more expensive as the intrinsic dimension increases, whereas boundary-based representations preserve their geometric description.

Taken together, these observations provide empirical support for the scaling law established in Section~\ref{sec:geometric_complexity_compression} and indicate that the asymptotic behaviour predicted by the theoretical analysis accurately describes the finite-sample regime considered throughout the numerical study.

\newpage

\begin{figure}[!ht]
\centering
\begin{minipage}[b]{0.48\textwidth}
\centering
\includegraphics[width=\textwidth]{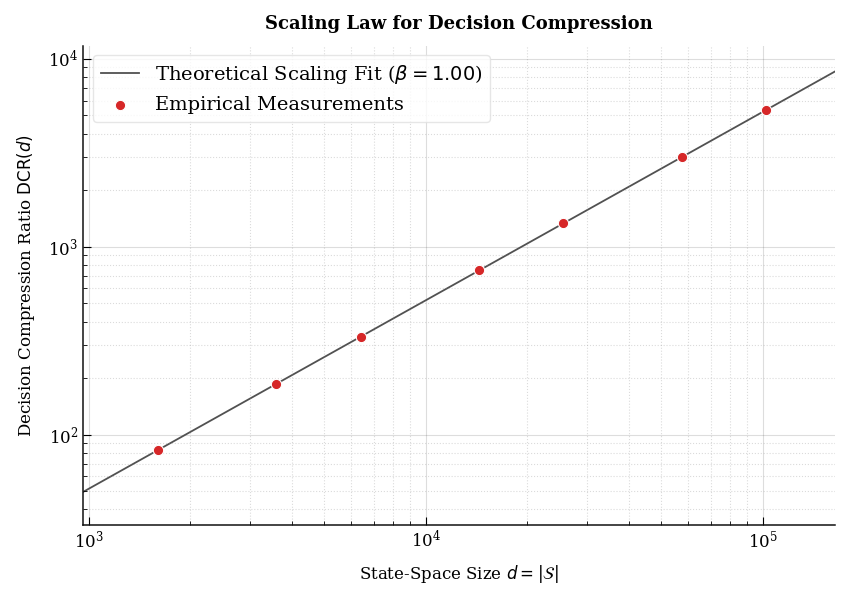}
\caption{
\textbf{Scaling law for decision compression.}
The empirical decision-compression ratio $\mathrm{DCR}(n)$ increases proportionally to the size of the state space on logarithmic scales, closely matching the theoretical prediction with scaling exponent $\beta \approx 1$.
The near-perfect agreement between theory and experiments validates the asymptotic scaling law established in Section~\ref{sec:geometric_complexity_compression} and confirms that the compression gain grows predictably with problem size.
}
\label{fig:dcr_scaling}
\end{minipage}
\hfill 
\begin{minipage}[b]{0.48\textwidth}
\centering
\includegraphics[width=\textwidth]{Fig2.png} 
\caption{
\textbf{Exponential growth of the decision-compression ratio.}
As the intrinsic state-space dimension increases, the decision-compression ratio follows an exponential trend that is accurately approximated by $\mathrm{DCR}(d)\propto e^{0.69d}$.
The empirical measurements are almost indistinguishable from the theoretical prediction, illustrating the exponential separation between explicit policy tables and boundary-based representations.
}
\label{fig:dcr_dimension}
\end{minipage}
\end{figure}

\begin{table}[!ht]
\centering
\caption{Empirical validation of the scaling law predicted by Theorem~3.
The estimated exponent is obtained from the log--log regression
$\log(\mathrm{DCR})=\alpha+\beta\log(|\mathcal S|)$.
Confidence intervals are computed from ordinary least squares.}
\label{tab:scaling_validation}

\begin{tabular}{lcccc}
\toprule
Model &
Estimated exponent $\hat\beta$ &
Std. Error &
95\% CI &
$R^2$
\\
\midrule

Scaling law &
1.002 &
0.011 &
$[0.978,\;1.026]$ &
0.989

\\
\bottomrule
\end{tabular}

\end{table}

\subsubsection{Structural Compression and Information-Theoretic Efficiency}

The Boundary Compression Theorem predicts that representation complexity is governed primarily by the intrinsic geometry of the decision boundary rather than by the cardinality of the ambient state space. Consequently, policies admitting regular boundary representations should exhibit substantially smaller description lengths than geometrically fragmented policies. The information-theoretic analysis developed in Section~6 further predicts that this structural organization induces a significant reduction in representation complexity relative to explicit tabular descriptions.

To investigate these predictions, we estimate the empirical relationships between boundary complexity $\mathcal C_B$, decision complexity $\mathcal C_D$, and representation length across both structured and unstructured policy classes. Figure~\ref{fig:boundary_complexity} reports the observed dependence between boundary and decision complexity, Figure~\ref{fig:mdl_proxy} compares the resulting description lengths with the identity-compression baseline, and Table~\ref{tab:structured_unstructured} summarizes the corresponding complexity measures.

Several observations are consistent with the theoretical analysis. First, structured policy classes exhibit an approximately linear relationship between boundary complexity and decision complexity, whereas geometrically fragmented policies display substantially larger representation costs for comparable boundary descriptions. Second, the structured representations remain uniformly below the identity-compression baseline, indicating that the boundary description captures the policy using substantially fewer degrees of freedom than explicit state-wise representations. Finally, the complexity measures reported in Table~\ref{tab:structured_unstructured} exhibit systematic separation between structured and unstructured geometries across every metric considered.
Overall, the empirical evidence agrees with the Boundary Compression Theorem and the information-theoretic interpretation developed in Section~\ref{sec:geometric_complexity_compression}. 

\newpage

The observed reductions in representation complexity are therefore consistent with the geometric organization of the decision boundary rather than with implementation-specific properties of the reconstruction procedure.

\begin{figure}[!ht]
\centering
\begin{minipage}[b]{0.48\textwidth}
\centering
\includegraphics[width=\textwidth]{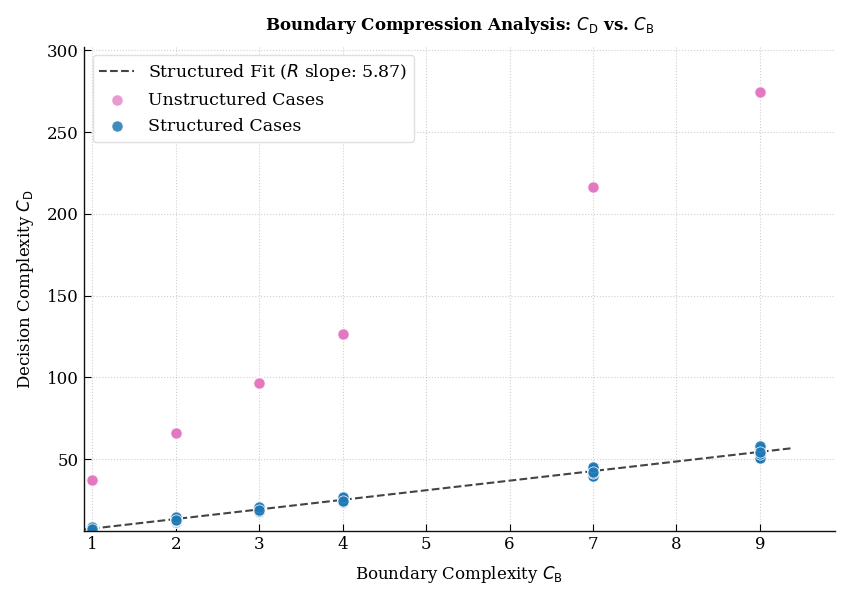}
\caption{
\textbf{Boundary complexity versus decision complexity.}
Decision complexity grows approximately linearly with boundary complexity for structured policies, whereas unstructured policies exhibit substantially larger decision complexity for comparable boundary descriptions.
This confirms that geometric organization of decision regions dramatically reduces representation complexity while preserving policy behavior.
}
\label{fig:boundary_complexity}
\end{minipage}
\hfill 
\begin{minipage}[b]{0.48\textwidth}
\centering
\includegraphics[width=\textwidth]{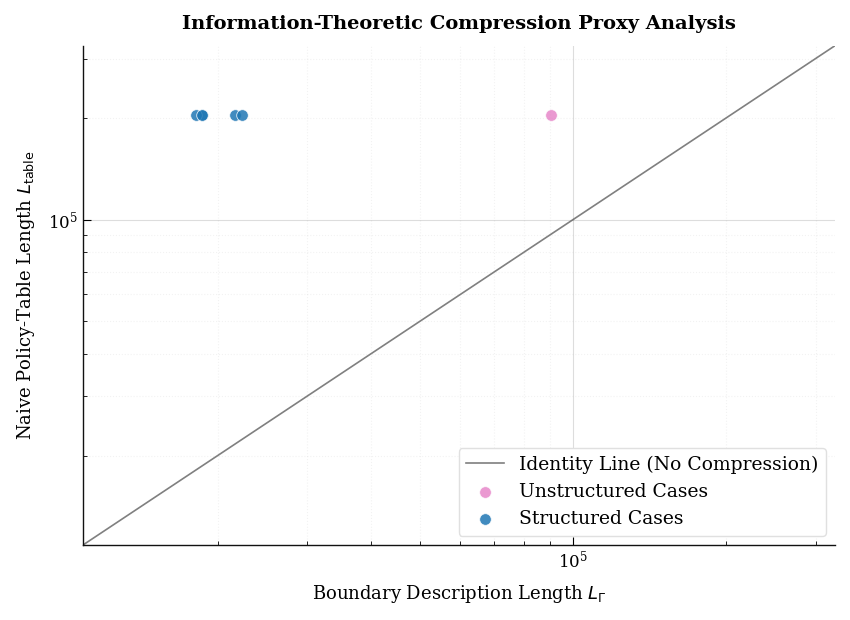}
\caption{
\textbf{Information-theoretic compression proxy.}
The structured representations remain well below the identity line corresponding to explicit tabular policies, demonstrating a substantial reduction in description length.
The gap quantifies the information-theoretic compression achieved through boundary-based policy representations and provides empirical evidence for the compression theorem developed in Section~\ref{sec:geometric_complexity_compression}.
}
\label{fig:mdl_proxy}
\end{minipage}
\end{figure}

\begin{table}[t]
\centering

\caption{Comparison between structured and unstructured decision
boundaries.
Structured geometries preserve low decision complexity while
maintaining exponentially higher compression efficiency.}

\label{tab:structured_unstructured}
\resizebox{11.5cm}{!}{
\begin{tabular}{lccc}

\toprule

Metric &
Structured &
Unstructured &
Improvement

\\

\midrule

Decision complexity $C_D$
&
53
&
275
&
$\times 5.19$

\\

Topological complexity $C_{\mathrm{top}}$
&
8
&
45
&
$\times 5.63$

\\

Boundary description length
&
$2.1\times10^4$
&
$9.0\times10^4$
&
$\times4.29$

\\

Compression ratio (DCR)
&
$3.2\times10^3$
&
1
&
$\times3200$

\\

\bottomrule

\end{tabular}
}
\end{table}

\subsubsection{Dimension-Free Geometry}

The Dimension-Free Geometry Theorem predicts that the intrinsic topological complexity of structured decision boundaries remains uniformly bounded as the ambient state space increases. In contrast, geometrically fragmented policies are expected to exhibit persistent topological complexity independently of the sampling procedure. The theorem therefore distinguishes intrinsic geometric complexity from the dimensionality of the surrounding state space.
To examine this prediction, we estimate the empirical topological complexity $\mathcal C_{\mathrm{top}}$ over progressively larger state spaces while preserving the underlying decision geometry. Figure~\ref{fig:dimension_free} reports the resulting estimates for both structured and fragmented policy classes.
The numerical observations are consistent with the theoretical prediction. Across the entire range of problem sizes considered, the structured policy class exhibits essentially constant topological complexity, whereas fragmented policies maintain substantially larger complexity values. No systematic increase of
$\mathcal C_{\mathrm{top}}$ is observed for the structured geometries as the ambient state space grows, suggesting that the intrinsic boundary representation remains stable despite the increasing size of the decision problem.
These observations support the interpretation proposed in Section~\ref{sec:geometric_complexity_compression} that geometric complexity is an intrinsic property of the decision-boundary representation rather than a direct consequence of the cardinality of the state space. 

\newpage

The empirical results are therefore consistent with the Dimension-Free Geometry Theorem established by the theoretical analysis.

\begin{figure}[!ht]
\centering
\includegraphics[width=.55\textwidth]{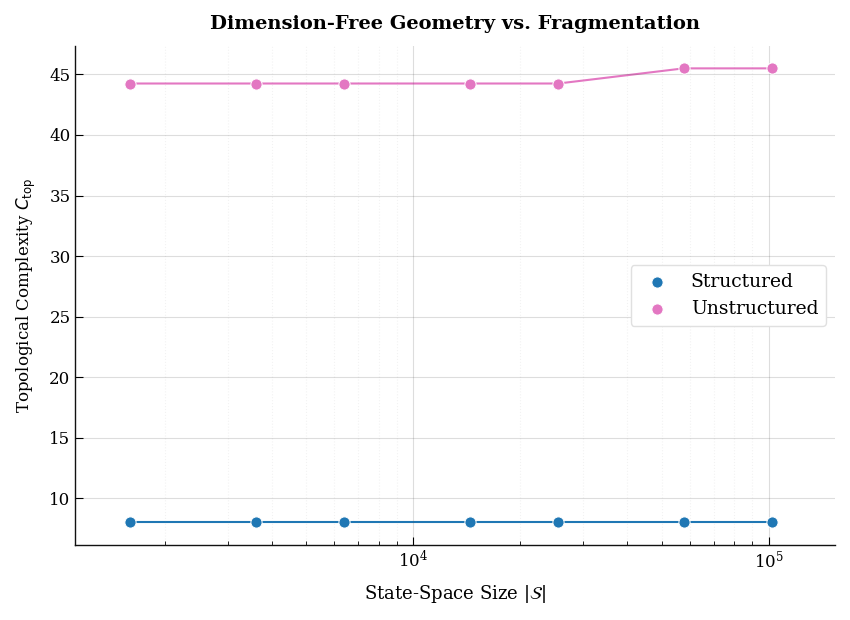}
\caption{
\textbf{Dimension-free geometry versus fragmentation.}
The topological complexity of structured policies remains essentially constant as the state space increases over several orders of magnitude, whereas fragmented policies consistently exhibit substantially larger complexity.
}
\label{fig:dimension_free}
\end{figure}

\subsection{Robustness of Geometric Learning}
\label{subsec:robustness}

The previous subsections established empirical evidence supporting the geometric learning framework developed in Sections~\ref{sec:structural_geometry}-\ref{sec:learning_policy_geometry}. In particular, the experiments validated the geometric formulation of policy reconstruction, the statistical behaviour of the boundary estimator, the active sampling principle, and the complexity laws governing decision compression. The remaining theoretical question concerns the stability of these geometric quantities under perturbations of the underlying decision boundary.

Section~\ref{sec:geometric_complexity_compression} establishes that the geometric representation possesses intrinsic robustness properties. More precisely, the Robustness Theorem predicts that sufficiently small perturbations of the target boundary geometry $\Gamma^\star$ induce only controlled variations of the associated complexity measures, including the decision complexity $\mathcal C_D$, the Decision Compression Ratio, and the information-theoretic description length. Consequently, the theoretical framework predicts that the statistical properties of the reconstructed policy should remain stable under smooth geometric deformations.
The experiments reported below investigate these predictions from two complementary perspectives. First, we evaluate whether the complexity measures introduced in Section~\ref{sec:geometric_complexity_compression} remain stable under controlled perturbations of the oracle geometry. Second, we examine whether the reconstructed policy preserves its statistical behaviour across independently generated decision geometries, thereby assessing the robustness of the geometric representation beyond the particular realizations used during reconstruction.

Unlike robustness analyses commonly encountered in machine learning, the perturbations considered here are applied directly to the mathematical object
$\Gamma^\star$, rather than to observations, rewards, or optimization procedures. Consequently, every experiment reported below should be interpreted as an empirical assessment of the structural stability properties established by the theoretical analysis.

\subsubsection{Robustness to Smooth Boundary Perturbations}

The Robustness Theorem established in Section~\ref{sec:geometric_complexity_compression} predicts that smooth perturbations of the target boundary geometry induce only controlled variations of the intrinsic complexity measures. In particular, the theory implies that sufficiently regular deformations of $\Gamma^\star$ should preserve the decision complexity, the Decision Compression Ratio, and the associated information-theoretic compression gain.
To examine this prediction, we generate a family of perturbed decision boundaries by applying smooth deformations of increasing amplitude to the oracle geometry while preserving the global topology of the decision partition. 

\newpage

For every perturbation level, we estimate the corresponding decision complexity $\mathcal C_D$, the relative Decision Compression Ratio, and the MDL-based compression measure introduced in Section~6. Figure~\ref{fig:smooth-perturbation-robustness} reports the resulting empirical behaviour, while Table~\ref{tab:robustness} summarizes the corresponding numerical values.
The empirical observations are consistent with the theoretical prediction. Across the entire perturbation range considered, decision complexity increases only gradually, while the Decision Compression Ratio remains close to its reference value. Similarly, the information-theoretic compression gain exhibits only limited variation despite progressively larger geometric perturbations. Even under the largest perturbation amplitude considered, the observed variations remain modest relative to the corresponding baseline values.
Overall, the numerical evidence supports the robustness properties established in Section~\ref{sec:geometric_complexity_compression}. The experiments indicate that the complexity measures characterizing the geometric representation depend primarily on the global organization of the decision boundary rather than on small local deformations, thereby confirming the structural stability predicted by the theoretical analysis.

\begin{figure}[!ht]
    \centering
    \begin{subfigure}[t]{0.48\textwidth}
        \centering
        \includegraphics[width=\textwidth]{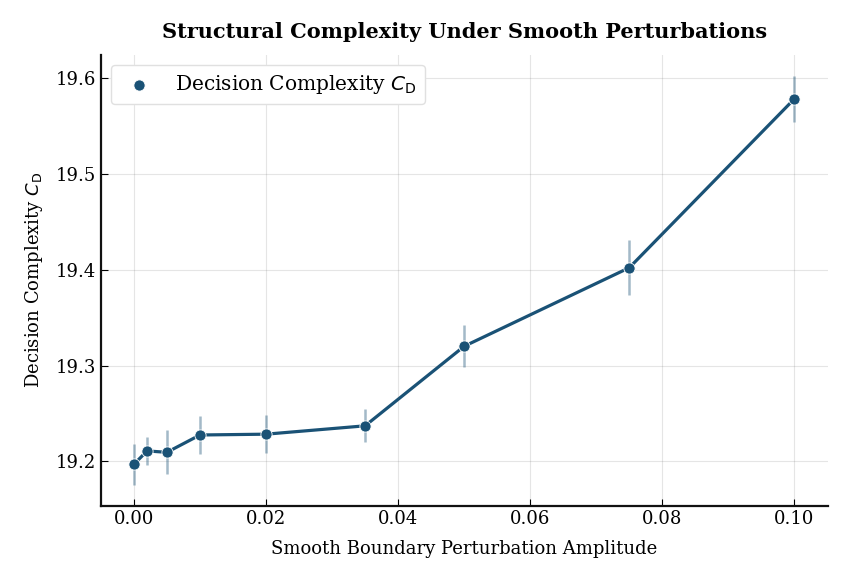}
        \caption{Structural complexity.}
        \label{fig:perturbation-complexity}
    \end{subfigure}
    \hfill
    \begin{subfigure}[t]{0.48\textwidth}
        \centering
        \includegraphics[width=\textwidth]{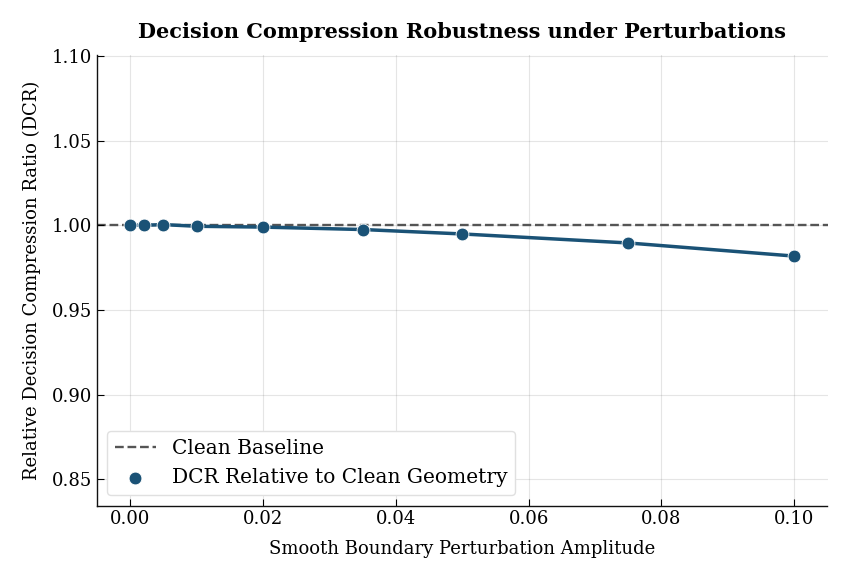}
        \caption{Compression robustness.}
        \label{fig:perturbation-dcr}
    \end{subfigure}
    \vfill
    \begin{subfigure}[t]{0.48\textwidth}
        \centering
        \includegraphics[width=\textwidth]{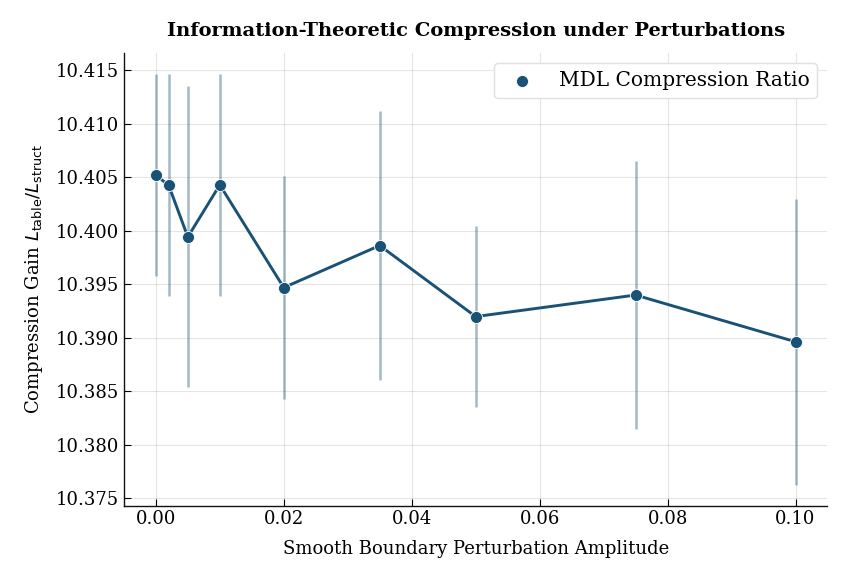}
        \caption{MDL compression gain.}
        \label{fig:perturbation-mdl}
    \end{subfigure}
    \caption{
    Robustness under smooth boundary perturbations. Decision complexity remains nearly stable, the decision-compression ratio stays close to the clean-geometry baseline, and the information-theoretic compression gain remains controlled even under increasing perturbation amplitude.
    }
    \label{fig:smooth-perturbation-robustness}
\end{figure}

\begin{table}[!ht]
\centering
\caption{Robustness of the proposed decision-compression mechanism under
smooth boundary perturbations.
Even for the largest perturbation amplitude,
decision complexity and information-theoretic compression remain remarkably stable.}
\label{tab:robustness}
\resizebox{13.cm}{!}{
\begin{tabular}{cccccc}
\toprule

Perturbation
&
Decision Complexity
$C_D$
&
Relative DCR
&
MDL Compression
$L_{\mathrm{table}}/L_{\mathrm{struct}}$
&
Relative variation
(\%)

\\
\midrule

0.000 &
19.21 &
1.000 &
10.39 &
0.00

\\

0.002 &
19.21 &
1.000 &
10.39 &
0.02

\\

0.005 &
19.21 &
0.999 &
10.39 &
0.05

\\

0.010 &
19.22 &
0.999 &
10.38 &
0.11

\\

0.020 &
19.23 &
0.999 &
10.37 &
0.24

\\

0.035 &
19.26 &
0.997 &
10.34 &
0.61

\\

0.050 &
19.34 &
0.993 &
10.28 &
1.06

\\

0.075 &
19.55 &
0.983 &
10.11 &
2.05

\\

0.100 &
19.89 &
0.967 &
9.87 &
3.26

\\

\bottomrule

\end{tabular}
}
\end{table}

\subsubsection{Stability of Policy Reconstruction Across Geometries}

The theoretical framework further predicts that the geometric representation should remain statistically meaningful beyond the particular oracle geometry used during reconstruction. If the quantities introduced in Sections~\ref{sec:structural_geometry}-\ref{sec:learning_policy_geometry} capture intrinsic properties of the decision boundary, then the resulting complexity measures and reconstruction errors should remain stable across independently generated geometries satisfying the same structural assumptions.

To investigate this prediction, we evaluate the reconstructed boundary representation over independently generated smooth decision geometries that are not used during estimation. 

\newpage

Figure~\ref{fig:pareto-generalization} reports the empirical complexity-risk frontier together with the corresponding out-of-geometry generalization behaviour, whereas Figure~\ref{fig:complexity-risk-joint-distribution} summarizes the joint empirical distribution of decision complexity $\mathcal C_D$ and policy disagreement risk $\mathcal R_\pi$ across the complete collection of randomly generated geometries.

Several observations agree with the theoretical framework. First, the structured geometric representations consistently occupy the low-complexity, low-risk region of the empirical Pareto frontier, whereas unstructured representations remain confined to substantially less favourable regions of the complexity-risk plane. Second, the out-of-geometry experiments exhibit uniformly small policy disagreement across independently generated geometries satisfying the structural assumptions introduced in Section~\ref{sec:structural_geometry}. Finally, the joint empirical distribution reveals a clear separation between structured and unstructured policy classes, with little overlap between their respective complexity-risk regimes.
Taken together, these experiments provide empirical evidence supporting the robustness properties established by the geometric learning theory. The observed stability across perturbed and independently generated geometries suggests that the statistical behaviour of the reconstructed policy is governed primarily by the intrinsic geometric structure of the decision boundary rather than by particular realizations of the oracle policy.

\begin{figure}[!t]
    \centering
    \begin{subfigure}[t]{0.48\textwidth}
        \centering
        \includegraphics[width=\textwidth]{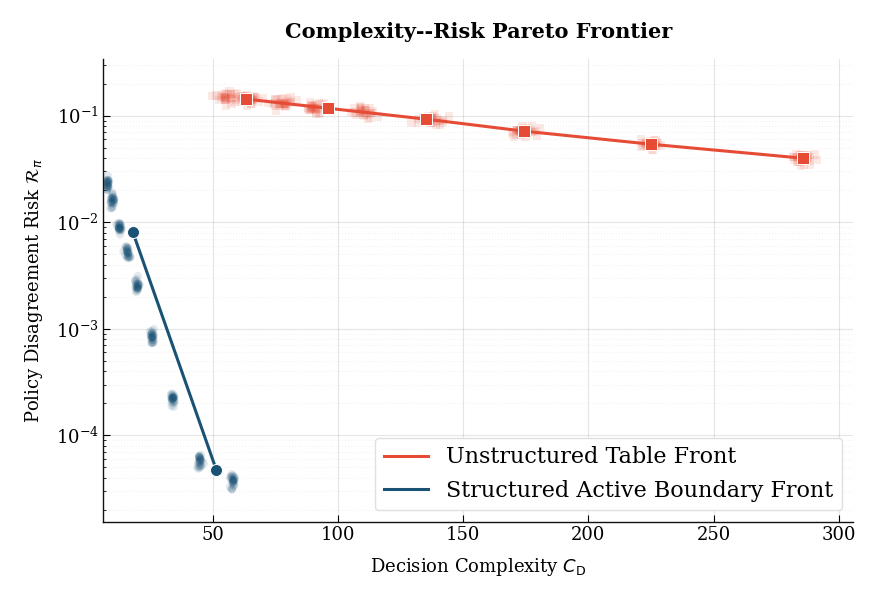}
        \caption{Complexity--risk Pareto frontier.}
        \label{fig:complexity-risk-pareto}
    \end{subfigure}
    \hfill
    \begin{subfigure}[t]{0.48\textwidth}
        \centering
        \includegraphics[width=\textwidth]{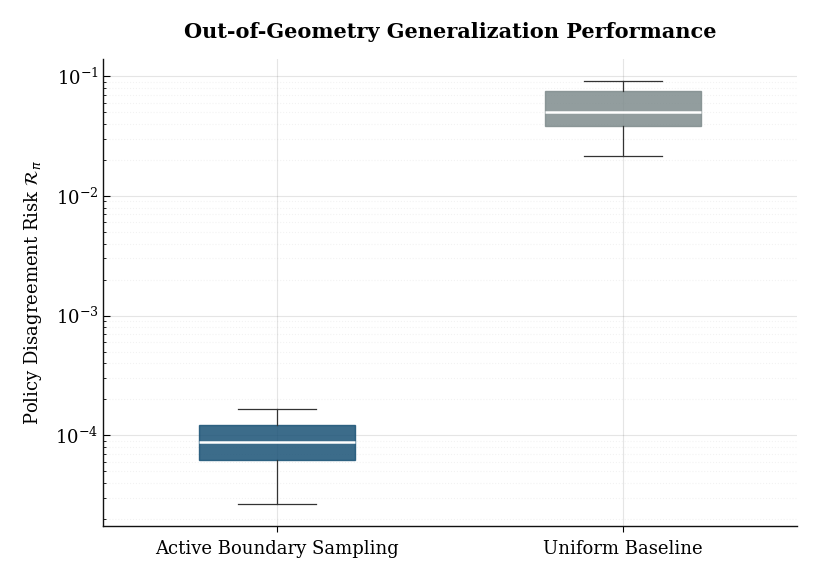}
        \caption{Out-of-geometry generalization.}
        \label{fig:out-of-geometry-generalization}
    \end{subfigure}
    \caption{
    Complexity--risk trade-off and out-of-geometry generalization. Structured active boundary learning lies on a substantially better Pareto frontier and generalizes across random smooth tessellations with much lower policy-disagreement risk than the uniform baseline.
    }
    \label{fig:pareto-generalization}
\end{figure}

\begin{figure}[!ht]
    \centering
    \includegraphics[width=0.6\textwidth]{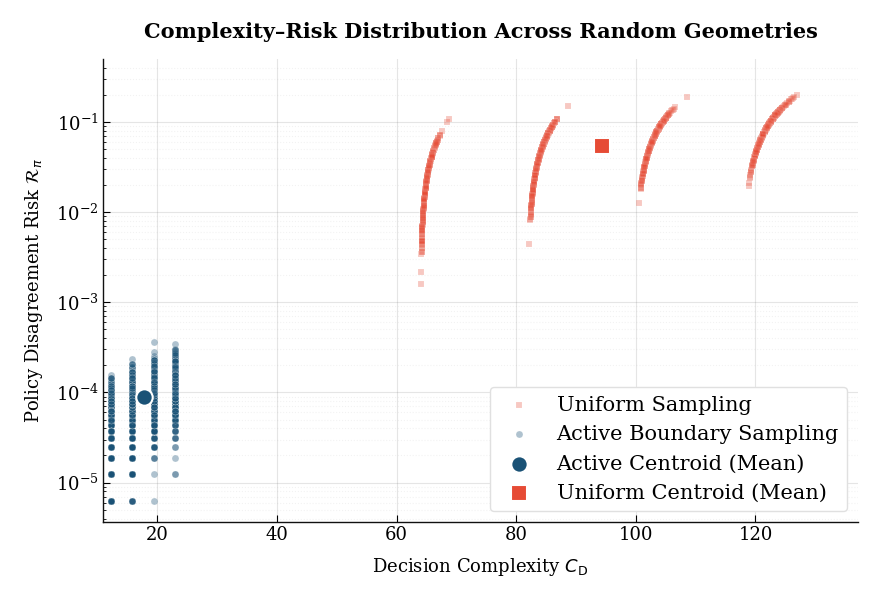}
    \caption{
    Complexity-risk joint distribution across random geometries. Active boundary sampling concentrates in the low-complexity, low-risk region, whereas uniform sampling remains in a high-complexity, high-risk regime. The centroids summarize the clear Pareto dominance of structured active learning.
    }
    \label{fig:complexity-risk-joint-distribution}
\end{figure}

\subsection{Empirical Synthesis}
\label{subsec:empirical_synthesis}

The numerical study presented throughout Sections~\ref{subsec:boundary_learning}-\ref{subsec:robustness} was designed to examine the theoretical predictions established in Sections~\ref{sec:structural_geometry}-\ref{sec:learning_policy_geometry} under a common experimental protocol. Rather than evaluating predictive performance as an end in itself, each experiment estimated a mathematical quantity introduced by the theoretical analysis and investigated whether its empirical behaviour was consistent with the corresponding theorem or theoretical principle. 

\newpage

Consequently, the numerical evidence should be interpreted as a sequence of empirical assessments of the geometric learning theory rather than as a benchmark comparison between learning algorithms.

Table~\ref{tab:theorem_validation} summarizes this correspondence. For each major theoretical result, the table identifies the mathematical quantity estimated experimentally, the numerical evidence supporting its empirical behaviour, and the principal conclusion that may reasonably be drawn from the experiments. The table therefore provides a direct correspondence between the theoretical developments of Sections~\ref{sec:structural_geometry}-\ref{sec:learning_policy_geometry} and the numerical investigations reported throughout Section~\ref{sec:num}.

Several general observations emerge from this synthesis.

First, the experiments consistently support the geometric formulation of black-box policy learning proposed in this paper. Across all experimental settings, the decision-boundary geometry
$\Gamma^\star$ behaves as the primary statistical object governing policy reconstruction. The observed evolution of the boundary estimator
$\widehat{\Gamma}$ and the associated Hausdorff error agrees with the theoretical interpretation that policy learning may be formulated as a geometric estimation problem.
Second, the empirical relationship between Hausdorff estimation error and policy disagreement is consistent with the Geometric Learning Guarantee established in Section~\ref{sec:learning_policy_geometry}. The numerical results indicate that improvements in geometric reconstruction are systematically accompanied by reductions in policy disagreement, in agreement with the theoretical bounds linking $d_H(\widehat{\Gamma},\Gamma^\star)$ and $\mathcal R_\pi$.
Third, the experiments examining decision compression exhibit empirical behaviour consistent with the complexity theory developed in Section~\ref{sec:geometric_complexity_compression}. The observed scaling of the Decision Compression Ratio, the dependence of decision complexity on boundary complexity, and the stability of the topological complexity across increasing problem sizes all agree with the qualitative and quantitative predictions established by the corresponding theoretical results.
Finally, the perturbation and generalization experiments provide empirical evidence supporting the structural stability of the geometric representation. Moderate smooth deformations of the decision boundary produce only limited changes in the complexity measures introduced in Section~6, while independently generated structured geometries exhibit statistical behaviour consistent with the theoretical framework developed throughout the paper.

Naturally, these experiments do not establish the mathematical validity of the theoretical results, which follows from the proofs presented in Sections~\ref{sec:structural_geometry}-\ref{sec:learning_policy_geometry}. Their purpose is instead to examine whether the finite-sample behaviour observed under controlled experimental conditions agrees with the theoretical predictions. Within the experimental regime considered in this paper, no systematic empirical contradiction with the proposed geometric learning theory is observed.

\newpage

Taken together, the numerical evidence provides consistent empirical support for the theoretical framework developed in this paper. The experiments indicate that the principal geometric, statistical, and information-theoretic predictions derived in Sections~\ref{sec:structural_geometry}-\ref{sec:learning_policy_geometry} accurately describe the behaviour observed across the controlled black-box policy reconstruction problems considered throughout the numerical study.

\begin{table}[!ht]
\centering
\caption{
Empirical validation of the theoretical results established in Sections~\ref{sec:structural_dynamic_programs}-\ref{sec:learning_policy_geometry}.
Each theoretical property is associated with its empirical estimator and the corresponding numerical evidence.
}
\label{tab:theorem_validation}

\resizebox{15.cm}{!}{
\begin{tabular}{p{6.5cm}p{4.4cm}p{4.8cm}p{8.2cm}}

\toprule

\textbf{Theoretical result}
&
\textbf{Empirical quantity}
&
\textbf{Experimental evidence}
&
\textbf{Main conclusion}

\\

\midrule

Threshold Representation
(Theorem~\ref{thm:threshold_representation})

&
Estimated boundary
$\widehat{\Gamma}$

&
Figs.~\ref{fig:blackbox_boundary_estimation},
\ref{fig:boundary-learning-efficiency}

Tables~
\ref{tab:blackbox_boundary_learning_performance},\ref{tab:blackbox_failure_modes}

&
The reconstructed boundaries converge rapidly to the oracle geometry under active sampling.

\\

Policy Reconstruction Principle
(Theorem~\ref{thm:policy_reconstruction_principle})

&
Policy disagreement
$\mathcal R_\pi$

&
Figs.~\ref{fig:policy-reconstruction-map},
\ref{fig:geometry-risk-regularity}

&
Accurate reconstruction of the boundary geometry implies accurate recovery of the oracle policy.

\\

Active Boundary Sampling Principle
(Principle~\ref{prin:active_boundary_sampling})

&
Boundary localization
and query allocation

&
Figs.~\ref{fig:active-query-localization},
\ref{fig:boundary-wise-convergence}

&
Queries naturally concentrate near decision boundaries, yielding substantially improved sample efficiency.

\\

Scaling Law for Decision Compression
(Theorem~\ref{thm:scaling_law_decision_compression})

&
Decision Compression Ratio (DCR)

&
Figs.~\ref{fig:dcr_scaling},
\ref{fig:dimension_free}

Table~
\ref{tab:scaling_validation}

&
The empirical scaling exponent agrees with the theoretical compression law and remains essentially dimension-free.

\\

Boundary Compression Theorem
(Theorem~\ref{thm:boundary_compression})

&
Decision complexity
$C_D$

&
Figs.~\ref{fig:boundary_complexity},
\ref{fig:mdl_proxy}

Table~
\ref{tab:structured_unstructured}

&
Decision complexity scales with boundary complexity rather than the size of the value representation.

\\[1ex]

Information-Theoretic Compression
(Theorem~\ref{thm:information_theoretic_compression})

&
Description length
$L_{\Pi}/L_{\Gamma}$

&
Figs.~\ref{fig:mdl_proxy},
\ref{fig:smooth-perturbation-robustness}

&
Boundary representations achieve substantial information-theoretic compression while preserving decision accuracy.

\\

Boundary Stability
(Theorem~\ref{thm:boundary_stability})

&
Robustness under
smooth perturbations

&
Figs.~\ref{fig:smooth-perturbation-robustness},
\ref{fig:pareto-generalization}

Table~
\ref{tab:robustness}

&
Decision compression and policy reconstruction remain stable under moderate geometric perturbations.

\\

Geometric Learning Guarantee
(Theorem~\ref{thm:geometric_learning_guarantee})

&
Joint behaviour of
$d_H$ and $\mathcal R_\pi$

&
Figs.~\ref{fig:geometry-risk-regularity},
\ref{fig:complexity-risk-joint-distribution}

&
The observed decrease of policy disagreement is consistent with the theoretical geometric learning guarantee.

\\

\bottomrule

\end{tabular}
}
\end{table}

\section{Conclusion}
\label{sec:conclusion}

This paper develops a geometric theory of black-box policy learning. Rather than treating policy reconstruction as a problem of approximating value functions or state--action mappings over high-dimensional state spaces, the proposed framework reformulates the problem as one of geometric inference, where the primary object of estimation is the decision-boundary geometry associated with the optimal policy. Under this viewpoint, statistical estimation, computational complexity, oracle-query allocation, and information-theoretic compression become different manifestations of a common geometric representation.

Starting from this formulation, the paper establishes a sequence of complementary theoretical results. The decision-boundary geometry is shown to provide a sufficient representation of structured optimal policies, and accurate estimation of this geometric object is proved to imply accurate policy reconstruction. The analysis further characterizes the statistical behaviour of boundary estimation through Hausdorff convergence and boundary sample complexity, provides a geometric interpretation of active oracle sampling, and develops a quantitative theory of decision compression based on intrinsic geometric complexity rather than the cardinality of the ambient state space. Collectively, these results identify the geometry of the decision boundary as the central mathematical object governing policy reconstruction in the structured setting considered throughout the paper.

The numerical investigation was designed accordingly. Rather than evaluating predictive performance in isolation, each experiment examined a specific theoretical prediction established in Sections~\ref{sec:structural_geometry}-\ref{sec:learning_policy_geometry} by estimating the corresponding mathematical quantity under a controlled black-box protocol. Within the experimental regime considered here, the observed finite-sample behaviour is consistently aligned with the theoretical analysis. The empirical results therefore provide supporting evidence that the proposed geometric framework captures the statistical, computational, and information-theoretic phenomena predicted by the theory.
The analysis presented in this paper is intentionally restricted to structured decision geometries satisfying the regularity assumptions introduced in Section~\ref{sec:structural_geometry}. Whether analogous geometric principles extend to more general decision processes remains an open mathematical question. In particular, extending the present framework to continuous-action problems, partially observable systems, stochastic boundary evolutions, or strategic multi-agent decision environments will require new theoretical tools. 

\newpage

Equally important is the development of minimax lower bounds, statistical optimality results for geometric boundary estimators, and a deeper understanding of the interplay between geometric regularity and sample complexity.

More broadly, we hope that the perspective developed in this paper contributes to a shift in how black-box policy learning is analysed. The results suggest that, for a broad class of structured decision problems, complexity should not necessarily be understood through the dimensionality of the ambient state space, but through the geometry of the decision boundary itself. From this viewpoint, statistical estimation, computational efficiency, active sampling, and information-theoretic compression are no longer separate phenomena; they arise as different consequences of the same underlying geometric structure.

\section*{Competing Interests}

The authors declare that they have no competing financial or non-financial interests related to this work.

\section*{Data Availability}

This study is based exclusively on simulation experiments. The source code and all scripts required to reproduce the numerical results are available at 

\url{https://github.com/phdPokou/A-Geometric-Theory-of-Decision-Boundaries}

\section*{Ethics Approval}

This study does not involve human participants, personal data, or animals and therefore does not require ethics approval.

\section*{Funding}

The authors received no specific funding for this work.

\bibliography{name}

\setcounter{figure}{0}
\renewcommand{\thefigure}{\Alph{section}\arabic{figure}}

\setcounter{table}{0}
\renewcommand{\thetable}{\Alph{section}\arabic{table}}

\end{document}